\documentclass[MAL,biber]{style}

\usepackage[utf8]{inputenc}

\usepackage{amsmath,amsbsy,amsfonts,amssymb,amsthm,dsfont,units}
\usepackage{graphicx,psfrag,epsfig,epsf}
\usepackage{algorithm}
\usepackage{algorithmic,mathtools}
\usepackage{color,cases}
\usepackage{tikz}
\usepackage{enumitem}
\usetikzlibrary{calc,shapes}
\usepackage{subfigure}
\usepackage{sublabel}
\usepackage{appendix}
\usepackage{blkarray}
\usepackage{rotating}

\usepackage{listings}
\DeclareFixedFont{\ttb}{T1}{txtt}{bx}{n}{12} 
\DeclareFixedFont{\ttm}{T1}{txtt}{m}{n}{12}  

\usepackage{color}
\definecolor{deepblue}{rgb}{0,0,0.5}
\definecolor{deepred}{rgb}{0.6,0,0}
\definecolor{deepgreen}{rgb}{0,0.5,0}
\definecolor{gray}{rgb}{0.5,0.5,0.5}

\newcommand\pythonstyle{\lstset{
  language=Python,
  basicstyle=\scriptsize,
  otherkeywords={self},
  morekeywords={as},
  deletendkeywords={round,ord},
  keywordstyle=\bfseries\color{deepblue},
  emph={MyClass,__init__},          
  backgroundcolor=\color{white}, 
  emphstyle=\color{deepred},    
  stringstyle=\color{deepgreen},
  commentstyle=\color{blue},  
  frame=tb, 
  showstringspaces=false            
  breaklines=true,
  rulecolor=\color{black},
  tabsize=4,
  numbers=left,
  numbersep=5pt,
  numberstyle=\tiny\color{gray},
  xleftmargin=2em,
  framexleftmargin=1.5em,
  aboveskip=20pt,
  belowskip=20pt,
}}
\lstnewenvironment{python}[1][]
{
\pythonstyle
\lstset{#1}
}
{}
\newcommand\pythoninline[1]{{\pythonstyle\lstinline!#1!}}

\renewcommand{\algorithmicrequire}{\textbf{Input:}}
\renewcommand{\algorithmicensure}{\textbf{Output:}}

\newcommand\coherence{\operatorname{coherence}}

 \newcommand{\IGNORE}[1]{}

\def\nn{\nonumber}

\newcommand\E{\mathbb{E}}
\newcommand\R{\mathbb{R}}
\newcommand\N{\mathcal{N}}

\renewcommand\t{{\scriptscriptstyle\top}}
\DeclareMathOperator{\diag}{diag}

\newcommand\eps{\varepsilon}
\newcommand\veps{\varepsilon}

\newcommand\tl{\tilde}

\newcommand\mat{\operatorname{mat}}

\newcommand\poly{\operatorname{poly}}

\DeclareMathOperator{\krank}{krank}
\DeclareMathOperator{\var}{var}

\newcommand\lambdamax{\ensuremath{\lambda_{\max}}}
\newcommand\lambdamin{\ensuremath{\lambda_{\min}}}
\renewcommand\th[1]{\ensuremath{\theta_{#1}}}

\def\tl{\tilde}\renewcommand\t{{\scriptscriptstyle\top}}

\newcommand\Dir{\operatorname{Dir}}
\newcommand\inner[1]{\ensuremath{\langle #1 \rangle}}

\newcommand{\mjcomment}[1]{\noindent{\textcolor{blue}{\textbf{\#\#\# MJ:} \textsf{#1} \#\#\#}}}

\DeclareMathOperator{\tr}{Tr}

 \DeclareMathOperator*{\argmin}{arg\,min}
 \DeclareMathOperator*{\argmax}{arg\,max}

\DeclareMathOperator{\rank}{Rank}

\DeclareMathOperator{\Var}{Var}

\DeclareMathOperator{\Diag}{Diag}

\DeclarePairedDelimiter\norm{\lVert}{\rVert}

 \def\0{{\bf 0}}

\def\nn{\nonumber}

\def\qed{\hfill\hbox{${\vcenter{\vbox{
    \hrule height 0.4pt\hbox{\vrule width 0.4pt height 6pt
    \kern5pt\vrule width 0.4pt}\hrule height 0.4pt}}}$}}

\def\tcr{\textcolor{red}}
\def\tcb{\textcolor{blue}}

\definecolor{myred}{rgb}{0.3,0.0,0.7}
\definecolor{dkg}{rgb}{0.1,0.7,0.2}
\definecolor{dkb}{rgb}{0.0,0.2,0.8}
\definecolor{brm}{rgb}{1,0.0,1}

\def\tcdkg{\textcolor{dkg}}

 \def\ha{\widehat{a}}
 \def\hb{\widehat{b}}
 \def\hc{\widehat{c}}

 \def\hf{\widehat{f}}

\def\Sc{{\cal S}}

\def\Ebb{{\mathbb E}}

\def\Rbb{{\mathbb R}}

\newcommand{\bprf}{\begin{myproof}}
\newcommand{\eprf}{\end{myproof}}
\newcommand{\bp}{\begin{psfrags}}
\newcommand{\ep}{\end{psfrags}}
\newcommand{\bl}{\begin{lemma}}
\newcommand{\el}{\end{lemma}}
\newcommand{\bt}{\begin{theorem}}
\newcommand{\et}{\end{theorem}}
\newcommand{\bc}{\begin{center}}
\newcommand{\ec}{\end{center}}
\newcommand{\bi}{\begin{itemize}}
\newcommand{\ei}{\end{itemize}}
\newcommand{\ben}{\begin{enumerate}}
\newcommand{\een}{\end{enumerate}}
\newcommand{\bd}{\begin{definition}}
\newcommand{\ed}{\end{definition}}
\def\beq{\begin{equation}}
\def\eeq{\end{equation}\noindent}
\def\beqn{\begin{eqnarray}}
\def\eeqn{\end{eqnarray} \noindent}
\def\beqnn{  \begin{eqnarray*}}
\def\eeqnn{\end{eqnarray*}  \noindent}
\def\bcase{  \begin{numcases}}
\def\ecase{\end{numcases}   \noindent}
\def\bsbcase{  \begin{subnumcases}}
\def\esbcase{\end{subnumcases}   \noindent}

\newtheorem{lem}{Lemma}

\newtheorem{prop}{Proposition}[section]

\newtheorem{claim}{Claim}

\theoremstyle{remark} 

\newenvironment{myproof}{\noindent{\bf Proof:} \hspace*{1em}}{
    \hspace*{\fill} $\Box$ }
\newenvironment{proof_of}[1]{\noindent {\bf Proof of #1: }}{\hspace*{\fill} $\Box$ }

\newcommand{\matplottc}[1]{               
        \unitlength .45truein
        \begin{center}
        \includegraphics{#1.ps}
        \end{picture}
        \end{center}
}

\def\psfancypar#1#2{\begingroup\def\par{\endgraf\endgroup\lineskiplimit=0pt}
               \setbox2=\hbox{\large\sc #2}
               \newdimen\tmpht \tmpht \ht2 \advance\tmpht by \baselineskip
               \font\hhuge=Times-Bold at \tmpht
               \setbox1=\hbox{{\hhuge #1}}
               \count7=\tmpht \count8=\ht1
               \divide\count8 by 1000 \divide\count7 by \count8
               \tmpht=.001\tmpht\multiply\tmpht by \count7
               \font\hhuge=Times-Bold at \tmpht
               \setbox1=\hbox{{\hhuge #1}}
               \noindent
                \hangindent1.05\wd1
               \hangafter=-2 {\hskip-\hangindent
               \lower1\ht1\hbox{\raise1.0\ht2\copy1}
                \kern-0\wd1}\copy2\lineskiplimit=-1000pt}

\def\Kout{\setbox1=\hbox{\Huge\bf K}\hbox to
1.05\wd1{\hspace{.05\wd1}
}}
\def\Sout{\setbox1=\hbox{\Huge\bf S}\hbox to 1.05\wd1{\hspace{.05\wd1}
}}

\newcommand{\torestate}[3]{
\expandafter \def \csname BBRESTATE #2 \endcsname{#3}
\theoremstyle{plain}
\newtheorem{BBRESTATETHMNUM#2}[theorem]{#1}
\begin{BBRESTATETHMNUM#2}\label{#2}\csname BBRESTATE #2 \endcsname   \end{BBRESTATETHMNUM#2}
\newtheorem*{BBRESTATETHMNONNUM#2}{{#1}~\ref{#2}}
}

\newcommand{\restate}[1]{\begin{BBRESTATETHMNONNUM#1}[Restated] \csname BBRESTATE #1 \endcsname
\end{BBRESTATETHMNONNUM#1}}

\definecolor{blue1}{HTML}{0066FF}
\definecolor{lpurple}{cmyk}{.05,0.18,0,0}

 \DeclareMathOperator{\Norm}{Norm}

\title{Spectral Learning on Matrices and Tensors}

\subtitle{}

\maintitleauthorlist{
Majid Janzamin \\
Twitter \\
majid.janzamin@gmail.com
\and
Rong Ge \\
Duke University \\
rongge@cs.duke.edu
\and
Jean Kossaifi \\
Imperial College London \\
jean.kossaifi@imperial.ac.uk
\and
Anima Anandkumar \\
NVIDIA 
\& California Institute of Technology \\
anima@caltech.edu
}

\issuesetup
{
copyrightowner={M.~Janzamin, R.~Ge, J.~Kossaifi and A.~Anandkumar}, 
volume = 12, 
issue = 5-6, 
pubyear = 2019, 
isbn = 978-1-68083-640-0, 
eisbn = 978-1-68083-641-7, 
doi = 10.1561/2200000057, 
firstpage = 393, 
lastpage = 536 
}

\usepackage{mwe}

\author[1]{Janzamin,Majid}
\author[2]{Ge,Rong}
\author[3]{Kossaifi,Jean}
\author[4]{Anandkumar,Anima}

\affil[1]{Twitter; majid.janzamin@gmail.com}
\affil[2]{Duke University; rongge@cs.duke.edu}
\affil[3]{Imperial College London; jean.kossaifi@imperial.ac.uk}
\affil[4]{NVIDIA \& California Institute of Technology; anima@caltech.edu
}

\articledatabox{\nowfntstandardcitation}

\begin{document}

\makeabstracttitle

\begin{abstract}

Spectral methods have been the mainstay in several domains such as machine learning, applied mathematics and scientific computing. They involve finding a certain kind of spectral decomposition   to obtain basis functions that can  capture important structures or directions for the problem at hand. The most common spectral method  is the principal component analysis (PCA). It utilizes the principal components or the top eigenvectors of the data covariance matrix to carry out dimensionality reduction as one of its applications. This data pre-processing step is often effective in separating signal from noise.

PCA and other spectral techniques applied to {\em matrices} have several limitations. By limiting to only pairwise moments, they are effectively making a Gaussian approximation on the underlying data. Hence, they fail  on data with hidden variables which lead to non-Gaussianity. However, in almost any data set, there are latent effects   that cannot be directly observed, e.g., topics in a document corpus, or underlying causes of a disease. By extending the spectral decomposition methods to higher order moments, we demonstrate the ability to learn a wide range of latent variable models efficiently. Higher-order moments can be represented by {\em tensors}, and intuitively, they can encode   more information than just pairwise moment matrices.  
More crucially, tensor decomposition can pick up latent  effects that are missed by matrix methods. For instance, tensor decomposition can uniquely identify non-orthogonal components. 
Exploiting   these aspects turns out to be  fruitful for provable unsupervised learning of a wide range of latent variable models.

We also outline the computational techniques to design efficient tensor decomposition  methods. They are embarrassingly parallel and  thus scalable to large data sets.  Whilst there exist many optimized linear algebra software packages, efficient tensor algebra packages are also beginning to be developed. We introduce Tensorly, which has a simple python interface for expressing tensor operations. It has a flexible back-end system supporting NumPy, PyTorch, TensorFlow and MXNet amongst others. This allows it to carry out multi-GPU and CPU operations, and can also be seamlessly integrated with deep-learning functionalities.

\end{abstract}

\clearpage{}\chapter{Introduction} \label{sec:intro}

Probabilistic models form an important area of machine learning. They attempt to model the probability distribution of the observed data, such as documents, speech and images. Often, this entails relating observed data to   {\em latent or hidden variables}, e.g., topics for documents, words for speech and objects for images. The goal of learning is to then discover the latent variables and their relationships to the observed data.

Latent variable models have shown to be useful to provide a good explanation of the observed data, where they can capture the effect of hidden causes which are not directly observed.
Learning these hidden factors is central to many applications, e.g., identifying latent diseases through observed symptoms, and identifying latent communities through observed social ties. Furthermore, latent representations are very useful in feature learning. Raw data is in general very complex and redundant and feature learning is about extracting informative features from raw data. Learning efficient and useful features is crucial for the performance of learning tasks, e.g., the classification task that we perform using the learned features.

Learning latent variable models is challenging since the latent variables cannot,  by definition, be directly observed. In extreme cases, when there are more latent variables than observations, learning is theoretically impossible because of the lack of data, unless further constraints are imposed. More generally, learning latent variable models raises several questions. How much data do we need to observe in order to uniquely determine the model's parameters? Are there efficient algorithms to effectively learn these parameters? Can we get provable guarantees on the running time of the algorithm and the number of samples required to estimate the parameters? These are all important questions about learning latent variable models that we will try to address here.

In this monograph, we survey recent progress in using spectral methods including matrix and tensor decomposition techniques to learn many popular latent variable models. With careful implementation, tensor-based methods can run efficiently in practice, and in many cases they are the only algorithms with provable guarantees on running time and sample complexity.

There exist other surveys and overviews on tensor decomposition and its applications in machine learning and beyond. Among them, the work by \citet{kolda_survey} is very well-received in the community where they provide a comprehensive introduction to major tensor decomposition forms and algorithms and discuss some of their applications in science and engineering. More recently, \citet{sidiropoulos2017tensor} provide an overview of different types of tensor decompositions and some of their applications in signal processing and machine learning. \citet{papalexakis2017tensors} discuss several applications of tensor decompositions in data mining. \citet{rabanser2017introduction} review some basic concepts of tensor decompositions and a few applications.
\citet{debals2017concept} review several tensorization techniques which had been proposed in the literature. Here, tensorization is the mapping of a vector or matrix to a tensor to enable us using tensor tools.

In contrast to the above works, our focus in this monograph is on a special type of tensor decomposition called CP decomposition (see~\eqref{eq:tensor} as an example), and we cover a wide range of algorithms to find the components of such tensor decomposition. We also discuss the usefulness of this decomposition by reviewing several probabilistic models that can be learned using such tensor methods.

\section{Method of Moments and Moment Tensors}

How can we learn latent variable models, even though we cannot observe the latent variables? The key lies in understanding the relationship between latent variables and observed variables. A common framework for such relationship is known as the {\bf method of moments} which dates back to \citet{pearson1894contributions}. 

\paragraph{Pearson's 1-d Example:}The main idea of method of moments is to first estimate {\em moments} of the data, and use these estimates to learn the unknown parameters of the probabilistic model. For a one-dimensional random variable $X\in \R$, the $r$-th order moment is denoted by $\E[X^r]$, where $r$ is a positive integer and $\E[\cdot]$ is the expectation operator. Consider a simple example where $X$ is a mixture of two Gaussian variables. More precisely, with probability $p_1$, $X$ is drawn from a Gaussian distribution with mean $\mu_1$ and variance $\sigma_1^2$, and with probability $p_2$, $X$ is drawn from a Gaussian distribution with mean $\mu_2$ and variance $\sigma_2^2$. Here we have $p_1+p_2=1$. Let us consider the problem of estimating these unknown parameters given samples of $X$. The random variable $X$ can be viewed as drawn from a latent variable model because given a sample of $X$, we do not know which Gaussian it came from. Let latent variable $Z \in \{1,2\}$ be a random variable with probability $p_1$ of being 1. Then given $Z$, $X$ is just a Gaussian distribution as
$$
[X|Z = z] \sim \N(\mu_z, \sigma_z^2).
$$

As noted by \citet{pearson1894contributions}, even though we cannot observe $Z$, the moments of $X$ are closely related to the unknown parameters (probabilities $p_1,p_2$, means $\mu_1,\mu_2$, standard deviations $\sigma_1,\sigma_2$) we desire to estimate. More precisely, for the first three moments we have
\begin{align*}
\E[X] & = p_1\mu_1 + p_2\mu_2,\\
\E[X^2] & = p_1(\mu_1^2+\sigma_1^2)+p_2(\mu_2^2+\sigma_2^2),\\
\E[X^3] & = p_1 (\mu_1^3 + 3\mu_1\sigma_1^2)+p_2 (\mu_2^3 + 3\mu_2\sigma_2^2).\\
\end{align*}

The moments $\E[X], \E[X^2], \E[X^3],\ldots$ can be empirically estimated given observed data. Therefore, the equations above can be interpreted as a system of equations on the six unknown parameters stated above. \citet{pearson1894contributions} showed that with the first 6-th moments, we have enough equations to {\em uniquely} determine the values of the parameters.

\paragraph{Moments for Multivariate Random Variables of Higher Dimensions:}For a scalar random variable, its $p$-th moment is just a scalar number. However, for a random vector, higher order moments can reveal much more information. Let us consider a random vector $X \in \R^d$. The first moment of this variable is a vector $\mu \in \R^d$ such that $\mu_i = \E[X_i], \forall i\in [d]$, where $[d] := \{1,2,\dotsc, d \}$. For the second order moment, we are not only interested in the second moments of individual coordinates $\E[X_i^2]$, but also in the {\em correlation} between different coordinates $\E[X_i X_j], i\ne j$. Therefore, it is convenient to represent the second order moment as a $d\times d$ symmetric matrix $M$, where $M_{i,j} = \E[X_i X_j]$. 

This becomes more complicated when we look at higher order moments. For 3rd order moment, we are  interested in the correlation between all {\em triplets} of variables. In order to represent this compactly, we use a 3-dimensional $d\times d\times d$ object $T$, also known as a 3rd order tensor. The tensor is constructed such that $T_{i,j,k} = \E[X_iX_jX_k], \forall i,j,k\in[d]$. This tensor has $d^3$ elements or ${d+2\choose 3}$ distinct entries. In general, $p$-th order moment can be represented as a $p$-th order tensor with $d^p$ entries. These tensors are called moment tensors. Vectors and matrices are special cases of moment tensors of order 1 and 2, respectively.

In applications, it is often crucial to define what the random variable $X$ is, and examine what moments of $X$ we can estimate from the data.
We now provide a simple example to elaborate on how to form a useful moment and defer the proposal of many more examples to Section~\ref{ch:applications}.

\section{Warm-up: Learning a Simple Model with Tensors}
\label{sec:intro-LVMs}
\label{sec:intro-example}

In this section, we will give a simple example to demonstrate what is a tensor decomposition, and how it can be applied to learning latent variable models. Similar ideas can be applied to more complicated models, which we will discuss in Section~\ref{ch:applications}.

\paragraph{Pure Topic Model:}
The model we consider is a very simple topic model~\citep{papadimitriou2000latent, hofmann1999probabilistic}. In this model, there are $k$ unknown topics. Each topic entails a probability distribution over words in the vocabulary. Intuitively, the probabilities represent the likelihood of using a particular word when talking about a specific topic. As an example, the word ``snow'' should have a high probability in the topic ``weather'' but not the topic ``politics''. These probabilities are represented as a matrix $A\in \R^{d\times k}$, where $d$ is the size of the vocabulary and every column represents a topic. So, the columns of matrix $A$ correspond to the probabilities over vocabulary that each topic entails. We will use $\mu_j\in \R^d, j\in[k]$ to denote these probability distribution of words given $j$-th topic ($j$-th column of matrix $A$).

The model assumes each document is generated in the following way: first a topic $h\in[k]$ is chosen with probability $w_h$  where $w\in \R^k$ is a vector of probabilities; next, $l$ words $x_1,x_2,\dotsc,x_l$ are independently sampled from the $h$-th topic-word probability vector $\mu_h$. Therefore, we finally observe words for the documents. See Figure~\ref{fig:Multiview-Intro} for a graphical illustration of this model.
This is clearly a latent variable model, since we don't observe the topics. Our goal is to learn the parameters, which include the topic probability vector $w$ and the topic-word probability vectors $\mu_1,\dotsc,\mu_k$.

\begin{figure}
\begin{center}
\begin{tikzpicture}
  [
    scale=1.1,
    observed/.style={circle,minimum size=0.6cm,inner
sep=0mm,draw=black,fill=black!20},
    hidden/.style={circle,minimum size=0.6cm,inner sep=0mm,draw=black},
  ]
  \node [hidden,name=h] at ($(0,0)$) {$h$};
  \node [observed,name=x1] at ($(-1.5,-1)$) {$x_1$};
  \node [observed,name=x2] at ($(-0.5,-1)$) {$x_2$};
  \node [observed,name=xl] at ($(1.5,-1)$) {$x_l$};
  \node at ($(0.5,-1)$) {$\dotsb$};
  \draw [->] (h) to (x1);
  \draw [->] (h) to (x2);
  \draw [->] (h) to (xl);
\end{tikzpicture}
\end{center}
\caption{Pure Topic Model}
\label{fig:Multiview-Intro}
\end{figure}
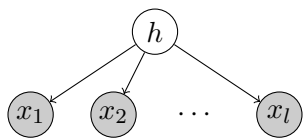

\paragraph{Computing the Moments:}First, we need to identify what the interesting moments are in this case. Since all we can observe are words in documents, and documents are all generated independently at random, it is natural to consider correlations between words as moments.

We say $x\in \R^d$ is an indicator vector of a word $z$ in our size-$d$ vocabulary if the $z$-th coordinate of $x$ is 1 and all other coordinates of $x$ are 0. For each document, let $x_1,x_2,x_3 \in \R^d$ be indicator vectors for the first three words. Given these word representations, the entries of the first three moments of $x_1,x_2,x_3$ can be written as
\begin{align*}
M_1(i) &= \Pr[x_1 = e_i],\\
M_2(i_1,i_2) & = \Pr[x_1 = e_{i_1}, x_2=e_{i_2}],\\
M_3(i_1,i_2,i_3) & = \Pr[x_1 = e_{i_1}, x_2=e_{i_2}, x_3 = e_{i_3}],
\end{align*}
where $e_i \in \R^d$ denotes the $i$-th basis vector in $d$-dimensional space. 
Intuitively, the first moment $M_1$ represents the probabilities for words; the second moment $M_2$ represents the probabilities that two words co-occur; and the third moment $M_3$ represents the probabilities that three words co-occur. 

We can empirically estimate $M_1,M_2,M_3$ from the observed documents. Now in order to apply the method of moments, we need to represent these probabilities based on the unknown parameters of our model. We can show that 
\begin{align}
M_1 &= \sum_{h=1}^k w_h \ \mu_h,\\
M_2 & = \sum_{h=1}^k w_h \ \mu_h \mu_h^\top,\label{eq:matrix}\\
M_3 & = \sum_{h=1}^k w_h \ \mu_h\otimes \mu_h\otimes \mu_h.\label{eq:tensor}
\end{align}
The computation follows from the law of total expectations (explained in more details in Section~\ref{sec:applications}).
Here, the first moment $M_1$ is the weighted average of $\mu_h$; the second moment $M_2$ is the weighted average of outer-products $\mu_h \mu_h^\top$; and the third moment $M_3$ is the weighted average of {\em tensor-products} $\mu_h\otimes \mu_h\otimes \mu_h$. The tensor product $\mu_h\otimes \mu_h\otimes \mu_h$ is a $d\times d\times d$ array whose $(i_1,i_2,i_3)$-th entry is equal to $\mu_h(i_1)\mu_h(i_2)\mu_h(i_3)$. See Section~\ref{ch:tensor-decomp} for more precise definition of the tensor product operator $\otimes$.

Note that the second moment $M_2$ is a matrix of rank at most $k$, and Equation \eqref{eq:matrix} provides a low-rank matrix decomposition of $M_2$. Similarly, finding $w_h$ and $\mu_h$ from $M_3$ using Equation \eqref{eq:tensor} is a problem called {\em tensor decomposition}. Clearly, if we can solve this problem, and it gives a unique solution, then we have learned the parameters of the model and we are done.

\section{What's Next?}

In the rest of this monograph, we will discuss the properties of tensor decomposition problem, review algorithms to efficiently find the components of such decomposition, and explain how they can be applied to learn the parameters of various probabilistic models such as latent variable models.

In Section~\ref{ch:matrix}, we first give a brief review of some basic matrix decomposition problems, including the singular value decomposition (SVD) and canonical correlation analysis (CCA). In particular, we will emphasize why matrix decomposition is often not enough to learn all the parameters of the latent variable models. 

Section~\ref{ch:tensor-decomp} discusses several algorithms for tensor decomposition. We will highlight under what conditions the tensor decomposition is {\em unique}, which is crucial in identifying the parameters of latent variable models.

In Section~\ref{ch:applications}, we give more examples on how to apply tensor decomposition to learn different latent variable models. In different situations, there are many tricks to manipulate the moments in order to get a clean equation that looks similar to \eqref{eq:tensor}.

In Section~\ref{ch:implementation}, we illustrate how to implement tensor operations in practice using the Python programming language. We then show how to efficiently perform tensor learning using TensorLy and scale things up using PyTorch.

Tensor decomposition and its applications in learning latent variable models are still active research directions. In the last two sections of this monograph we discuss some of the more recent results, which deals with the problem of overcomplete tensors and improves the guarantees on running time and sample complexity.

\clearpage{}

\clearpage{}\chapter{Matrix Decomposition}
\label{ch:matrix}
In this chapter, we describe some basic applications of matrix decomposition techniques including singular value decomposition (SVD), Principle Component Analysis (PCA) and canonical correlation analysis (CCA). These techniques are widely used in data analysis, and have been covered in many previous books (see e.g., \cite{golub1996matrix,horn2012matrix,blum2016foundations}).

The goal of this chapter is to give a brief overview of the matrix decomposition techniques. At the same time we try to point out connections and differences with relevant concepts in tensor decomposition. Especially, in many cases these matrix-based methods have the problem of ambiguity, and cannot be directly applied to learning parameters for latent variable models. In the next section, we will describe how these limitations can be solved by using tensor decomposition instead of matrix decomposition.

\section{Low Rank Matrix Decomposition}
\label{subsec:lowrank}
Assuming the reader is familiar with the basics of matrix algebra, we will start with reviewing matrix decompositions and matrix rank.
{\em Rank} is a basic property of matrices. A rank-1 matrix can be expressed as the outer product of two vectors as $u v^\top$ \--- its $(i,j)$-th entry is equal to the product of the $i$-th entry of vector $u$ denoted by $u(i)$ and the $j$-th entry of vector $v$ denoted by $v(j)$. Similarly, a matrix $M \in \R^{n\times m}$ is of rank at most $k$ if it can be written as the sum of $k$ rank-1 matrices as
\begin{equation}
M = \sum_{j=1}^k u_j v_j^\top. \label{eq:lowrank}
\end{equation}
Here $u_1,u_2,\dotsc,u_k\in \R^n$ and $v_1,v_2,\dotsc,v_k\in \R^m$ form the {\emph rank-1 components} of the matrix $M$. We call Equation (\ref{eq:lowrank}) a {\em decomposition} of matrix $M$ into rank-1 components; see Figure~\ref{fig:Matrix-decomp} for a graphical representation of this decomposition for a sample matrix $M \in \R^{5 \times 4}$.

\begin{figure}
\bc
$$
\vcenter{\hbox{\includegraphics[width=.13\linewidth]{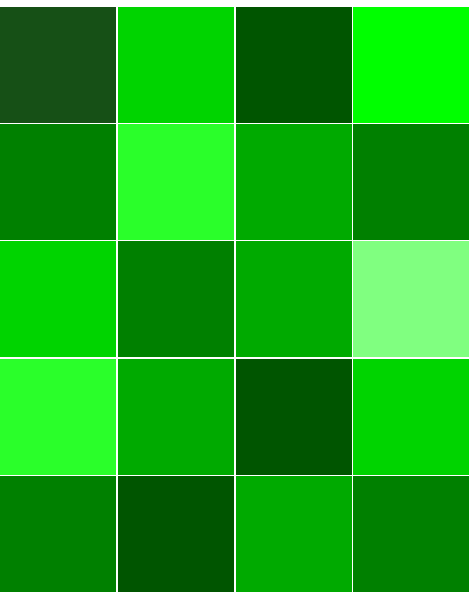}}}
= \vcenter{\hbox{\includegraphics[width=.17\linewidth]{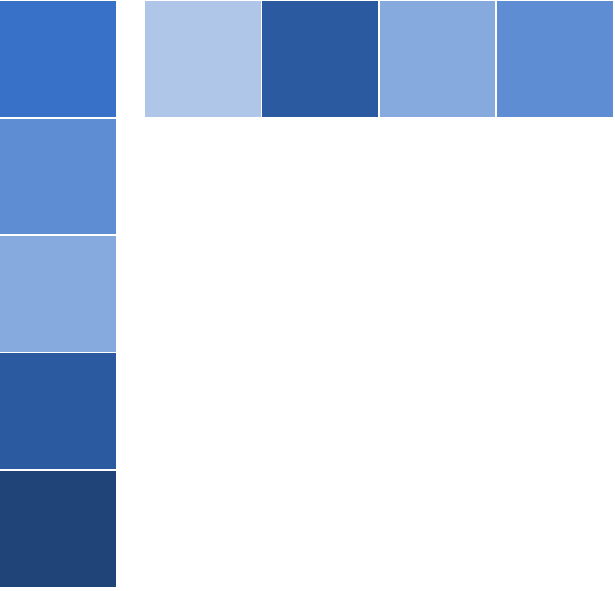}}}
+ \vcenter{\hbox{\includegraphics[width=.17\linewidth]{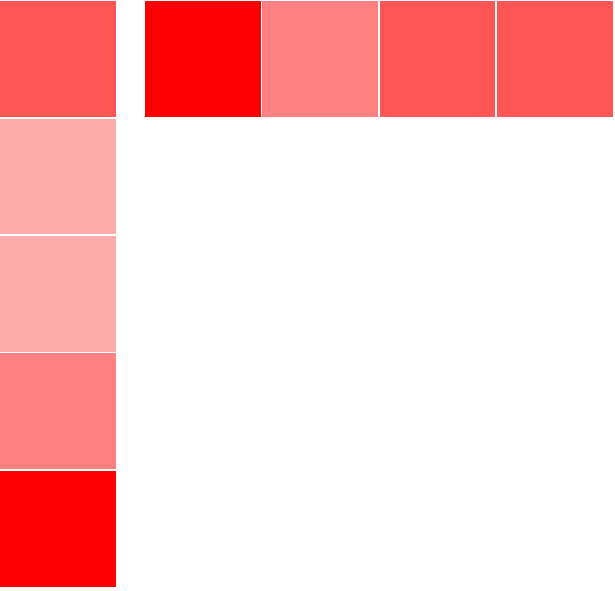}}}
+ \dotsb
$$
\ec
\caption[Matrix Decomposition]{Decomposition of a matrix $M \in \R^{5 \times 4}$ as sum of the rank-1 components. Note that each component is the product of a column vector $u_j$ and a row vector $v_j^\top$.}
\label{fig:Matrix-decomp}
\end{figure}

In many practical applications, the entries of a matrix are often determined by a small number of {\em factors}, and each factor corresponds to a rank-1 matrix; see Equation~\eqref{eq:lowrank}. As a result, many matrices we observe are {\em close to low rank matrices} (for exact definition of closeness see Section~\ref{sec:svd}).

We elaborate the application of low rank matrix decompositions with the following example.
Psychologist Charles Spearman worked on understanding whether human intelligence is a composite of different types of measurable intelligence and analyzed that through a factor analysis~\citep{spearman1904general}. Let us describe a highly simplified version of his method, where the hypothesis is that there are exactly two kinds of intelligence: {\em quantitative} and {\em verbal}. Suppose $n$ students are taking $m$ different tests on distinct subjects. We can summarize the scores that students get in different tests in a matrix $M \in \R^{n\times m}$. Each row lists the scores for a student, and each column the scores for a particular subject; see the score matrix example in Figure~\ref{fig:example-score}.

\begin{figure}
\bc
\[
\begin{blockarray}{ccccc}
& \begin{sideways} Math. \end{sideways} & 
\begin{sideways} Classics \end{sideways} & 
\begin{sideways} Physics \end{sideways} & 
\begin{sideways} Music \end{sideways} \\
\begin{block}{r(cccc)}
\text{Alice} & 19 & 26 & 17 & 21 \\
\text{Bob} & 8 & 17 & 9 & 12 \\
\text{Carol} & 7 & 12 & 7 & 9  \\
\text{Dave} & 15 & 29  & 16 & 21  \\
\text{Eve} & 31 & 40 & 27 & 33 \\
\end{block} 
\end{blockarray}
\begin{array}{c}
\vspace{0.21in} \\ =
\end{array}
\hspace{-0.1in}
\begin{array}{ccc}
\text{Verbal} & & \text{Quantitative} \vspace{0.2in} \\
\left( \begin{array}{c}
4 \\ 3 \\ 2 \\ 5 \\ 6 
\end{array} \right)
\left( \begin{array}{c}
1 \\ 5 \\ 2 \\ 3
\end{array} \right)^\top
& \hspace{-0.2in} + &
\hspace{-0.1in}
\left( \begin{array}{c}
3 \\ 1 \\ 1 \\ 2 \\ 5 
\end{array} \right)
\left( \begin{array}{c}
5 \\ 2 \\ 3 \\ 3
\end{array} \right)^\top
\end{array}
 \]
\ec
\caption{Score Matrix $M$ is an example for the scores of students (indexing the rows) in different tests on distinct subjects (indexing the columns). A corresponding low rank decomposition is also provided where the rank is two in this example.}
\label{fig:example-score}
\end{figure}

According to the simplified hypothesis, each student has different quantitative and verbal strengths. Each subject also requires different levels of quantitative and verbal skills or strength. Intuitively, a student with higher strength on verbal intelligence should perform better on a test that has a high weight on verbal intelligence. Therefore, as a simplest model we can describe the relationship as a bi-linear function:
\begin{align} \label{eqn:scoreExp-decomp}
\text{Score}(\text{student},\text{test}) = & \ \text{student}_{\text{verbal-intlg.}} \times \text{test}_{\text{verbal}} \\
& + \text{student}_{\text{quant-intlg.}}\times \text{test}_{\text{quant.}}. \nn
\end{align}
If we let $u_{\text{verbal}}, u_{\text{quant.}} \in \R^n$ be vectors that describe the verbal/quantitative strength for each student, and let $v_{\text{verbal}}, v_{\text{quant.}}\in \R^m$ be vectors that describe the requirement for each test, then we can write the score matrix $M$ as
\begin{equation}
M = u_{\text{verbal}}v_{\text{verbal}}^\top +u_{\text{quant.}}v_{\text{quant.}}^\top. \label{eq:decomposition}
\end{equation}
Therefore, $M$ is a rank 2 matrix!  Here quantitative and verbal are two {\em factors} that influence the result of the tests. The matrix $M$ is low rank because there are only two different factors. In general, this approach is called {\em factor analysis}. See Figure~\ref{fig:example-score} for an example of matrix $M$ and its corresponding rank 2 decomposition.

\subsection{Ambiguity of Matrix Decomposition} \label{sec:ambiguity}

As we described, decompositions like (\ref{eq:decomposition}) are very useful as they suggest the whole $n\times m$ matrix can be explained by a small number of components. However, if we are not only interested in the number of components, but also the exact values of the components (e.g., which student is strongest in the quantitative tasks), such decompositions are not sufficient because they are {\em not unique}. As an example, in Figure~\ref{fig:nonUnique-decomposition}, we give two different decompositions of the matrix we proposed earlier.

\begin{figure}
\bc
\[
\begin{array}{ccc}
\begin{blockarray}{ccccc}
& \begin{sideways} Math. \end{sideways} & 
\begin{sideways} Classics \end{sideways} & 
\begin{sideways} Physics \end{sideways} & 
\begin{sideways} Music \end{sideways} \\
\begin{block}{r(cccc)}
\text{Alice} & 19 & 26 & 17 & 21 \\
\text{Bob} & 8 & 17 & 9 & 12 \\
\text{Carol} & 7 & 12 & 7 & 9  \\
\text{Dave} & 15 & 29  & 16 & 21  \\
\text{Eve} & 31 & 40 & 27 & 33 \\
\end{block} 
\end{blockarray}
&
\begin{array}{c}
\vspace{0.21in} \\ =
\end{array} &
\hspace{-0.1in}
\begin{array}{ccc}
\text{Verbal} & & \text{Quantitative} \vspace{0.2in} \\
\left( \begin{array}{c}
\tcr{4} \\ \tcr{3} \\ \tcr{2} \\ \tcr{5} \\ \tcr{6} 
\end{array} \right)
\left( \begin{array}{c}
1 \\ 5 \\ 2 \\ 3
\end{array} \right)^\top
& \hspace{-0.2in} + &
\hspace{-0.1in}
\left( \begin{array}{c}
3 \\ 1 \\ 1 \\ 2 \\ 5 
\end{array} \right)
\left( \begin{array}{c}
\tcr{5} \\ \tcr{2} \\ \tcr{3} \\ \tcr{3}
\end{array} \right)^\top
\end{array} \\
 & = &
\hspace{-0.1in}
\begin{array}{ccc}
 \left( \begin{array}{c}
\tcr{1} \\ \tcr{2} \\ \tcr{1} \\ \tcr{3} \\ \tcr{1}
\end{array} \right)
\left( \begin{array}{c}
1 \\ 5 \\ 2 \\ 3
\end{array} \right)^\top
& \hspace{-0.2in} + &
\hspace{-0.1in}
\left( \begin{array}{c}
3 \\ 1 \\ 1 \\ 2 \\ 5 
\end{array} \right)
\left( \begin{array}{c}
\tcr{6} \\ \tcr{7} \\ \tcr{5} \\ \tcr{6}
\end{array} \right)^\top
\end{array}
\end{array}
 \]
\ec
\caption{\label{fig:nonUnique-decomposition}Two possible decompositions of the score matrix $M$ that we originally proposed in Figure~\ref{fig:example-score}. Note that the students verbal intelligence and tests quantitative weights are different between two decompositions.}
\end{figure}

In fact, this phenomena of non-uniqueness of matrix decomposition is very general. Consider a low rank decomposition $M = \sum_{j=1}^k u_j v_j^\top \in \R^{n \times m}$. Let $U \in \R^{n\times k}$ be a matrix whose columns are $u_j$'s, and let $V\in \R^{m\times k}$ be a matrix whose columns are $v_j$'s. Then we can represent $M$ as
$$
M = \sum_{j=1}^k u_j v_j^\top = UV^\top.
$$
Now for any {\em orthonormal matrix} $R\in \R^{k\times k}$ that satisfies $RR^\top = R^\top R = I$, we have
$$
M = UV^\top = URR^\top V^\top = (UR)(VR)^\top.
$$
Therefore, $UR$, $VR$ defines an {\em equivalent} decomposition, and its components (columns of $UR$, $VR$) can be completely different from the components in the original decomposition $UV^\top$. Later in Section~\ref{sec:totensors} we will revisit this example and see why {\em tensor decomposition} can avoid this ambiguity.

\section{Low Rank Matrix Approximation and SVD}
\label{sec:svd}

In practice, the matrix we are working on is often not {\em exactly} low rank. The observed matrix can deviate from the low rank structure for many reasons including but not limited to:
\begin{itemize}
\item The observed values can be {\em noisy}.
\item The factors may not interact linearly.
\item There might be several prominent factors as well as many small factors.
\end{itemize}
Despite all these possible problems, the observed matrix can still be {\em approximately} low rank. In such cases it is beneficial to find the low rank matrix that is the {\em closest} to the observed matrix (in other words, that best approximates it). In this section, we describe Singular Value Decomposition (SVD) method which is an elegant way of finding the closest low rank approximation of a matrix. To do so, we first define matrix norms and provide a concrete notion of closeness in matrix approximation.

\subsection{Matrix Norms}

Before talking about how to find the closest matrix, we need to first define when two matrices are close. Closeness is often defined by a distance function $d(A,B)$ for two same-size matrices $A$ and $B$. For general matrices, the most popular distance functions are based on matrix norms, i.e., $d(A,B) = \|A - B\|$ for some matrix norm $\|\cdot \|$.

There are many ways to define norms of matrices. The {\em Frobenius} norm and {\em spectral/operator} norm are the most popular ones.

\begin{definition}[Frobenius norm] The Frobenius norm of a matrix $M \in \R^{n\times m}$ is defined as
$$
\|M\|_F := \sqrt{\sum_{i=1}^n \sum_{j=1}^m M_{i,j}^2}.
$$
\end{definition}

Frobenius norm is intuitive and easy to compute. However, it ignores the matrix structure and is therefore equivalent to $\ell_2$ norm when we view the matrix as a vector. To understand the property of the matrix, we can view the matrix as a linear operator, and define its operator norm as follows.

\begin{definition}[Matrix spectral/operator norm] The spectral or operator norm of a matrix $M\in \R^{n\times m}$ is defined as
$$
\|M\| := \sup_{\|v\| \le 1} \|Mv\|,
$$
where $\|\cdot\|$ denotes the Euclidean $\ell_2$ norm for vectors.
\end{definition}
The spectral norm measures how much the matrix can stretch a vector that is inside the unit sphere.

Based on the above two norms, we can now define the closest low rank matrices as
\begin{align*}
M_k & := \argmin_{\mbox{rank}(N) \le k} \| M - N\|,\\
M_{k,F} & := \argmin_{\mbox{rank}(N) \le k} \| M - N\|_F.
\end{align*}
Both optimization problems are non-convex and may seem difficult to solve. Luckily, both of them can be solved by Singular Value Decomposition. In fact they have exactly the same solution, i.e., $M_k = M_{k,F}$, as we will see in the following section.

\subsection{Singular Value Decomposition}

For a matrix $M$, the Singular Value Decomposition (SVD) is a special type of low rank decomposition where all the rank-1 components are orthogonal to each other.

\begin{definition}[Singular Value Decomposition(SVD), see \cite{golub1996matrix} 2.5.3 or \cite{horn2012matrix} 7.3.1]  \label{def:svd}
The singular value decomposition of matrix $M\in \R^{n\times m}$ is defined as
$$
M = UDV^\top = \sum_{j=1}^{\min\{n,m\}} \sigma_j u_j v_j^\top,
$$
where $U := [u_1 | u_2 | \dotsb | u_n] \in \R^{n\times n}$ and $V  := [v_1 | v_2 | \dotsb | v_m] \in \R^{m\times m}$ are orthonormal matrices such that $U^\top U = I, V^\top V = I$, and $D \in \R^{n\times m}$  is a diagonal matrix whose diagonal entries are $\sigma_1 \ge \sigma_2 \ge \cdots \ge \sigma_{\min\{n,m\}} \ge 0$. The $u_j$'s (respectively $v_j$'s) are called the left (respectively right) singular vectors of $M$ and $\sigma_j$'s are called the singular values of $M$.
\end{definition}

Note that when $n < m$, we often view $D$ as a $n\times n$ diagonal matrix, and $V$ as an $m\times n$ orthonormal matrix because the extra columns of $V$ (columns indexed by $n< j \leq m$) are not relevant in the decomposition. Similarly when $n > m$, we often view $U$ as a $n\times m$ matrix.

The top singular value $\sigma_1$ is the largest singular value that is often denoted as $\sigma_{\max}(M)$, and the value $\sigma_{\min\{n,m\}}$ is the smallest singular value that is often denoted as $\sigma_{\min}(M)$. We now describe the optimization view-point of SVD where singular values are the maximum values of the quadratic form $u^\top M v$ when both $u$ and $v$ have bounded $\ell_2$ norms, and the corresponding components (called singular vectors) are the maximizers that are orthonormal vectors.

\begin{definition}[Optimization view-point of the SVD, see \cite{horn2012matrix} 7.3.10]\label{def:svdopt}
The top singular value $\sigma_1$ is the maximum of the quadratic form $u^\top M v$ when $u$ and $v$ have bounded $\ell_2$ norm, and the top singular vectors are the maximizers, i.e.,
\begin{align*}
\sigma_1 &= \max_{\|u\|\le 1, \|v\|\le 1} u^\top Mv,\\
u_1,v_1 & = \argmax_{\|u\|\le 1, \|v\|\le 1} u^\top Mv.
\end{align*}
The remaining values/vectors are obtained by maximizing the same quadratic form, while constraining the singular vectors to be orthogonal with all the previous ones, i.e.,
\begin{align*}
\sigma_j &= \max_{\|u\|\le 1, \|v\|\le 1,\forall i<j: u\perp u_i, v\perp v_i} u^\top Mv,\\
u_j,v_j & = \argmax_{\|u\|\le 1, \|v\|\le 1,\forall i<j: u\perp u_i, v\perp v_i} u^\top Mv.
\end{align*}
\end{definition}

As a result, we can also conclude that the spectral norm of $M$ is equal to $\sigma_1$, i.e., $\|M\| = \sigma_1$, since $\|Mv\| = \max_{\|u\| \le 1} u^\top Mv$. The singular values and singular vectors are also closely related to the eigenvalues and eigenvectors as we will demonstrate below.

\begin{lemma}[SVD vs.\ eigen-decomposition, see \cite{horn2012matrix} 7.3.5] For a matrix $M$, the singular values $\sigma_j$'s are the square roots of the eigenvalues of $MM^\top$ or $M^\top M$. The left singular vectors $u_j$'s are eigenvectors of $MM^\top$, and the right singular vectors $v_j$'s are eigenvectors of $M^\top M$.
\end{lemma}

In Section~\ref{sec:ambiguity}, we described how low rank matrix decomposition is not unique under orthogonal transformation of the rank-1 components. For SVD, because of the specific structure of its singular vectors this is not necessarily the case and in most cases Singular Value Decomposition is unique.

\begin{theorem}[Uniqueness of Singular Value Decomposition, see \cite{horn2012matrix} 7.3.5] \label{thm:SVD-unique}
The SVD of matrix $M$ defined in Definition~\ref{def:svd} is unique (for the first $\min(n,m)$ columns of $U,V$) when the singular values $\sigma_j$'s are all distinct and nonzero. \end{theorem}

Note that if $n$ and $m$ are different, say $n < m$, then the last $m - n$ columns of the matrix $V$ can be an arbitrary orthogonal basis that is orthogonal to the previous $n$ right singular vectors, so that is never unique. On the other hand, these columns in $V$ do not change the result of $UDV^\top$, so the decomposition $\sum_{j=1}^{\min\{n,m\}} \sigma_j u_jv_j^\top$ is still unique.
Following the optimization view-point of SVD in Definition~\ref{def:svdopt}, it is standard to sort the diagonal entries of $D$ in descending order. In many applications we only care about the top-$k$ components of the SVD, which suggests the following definition of truncated SVD.

\begin{definition}[Truncated SVD, see \cite{golub1996matrix} 2.5.4] Suppose $M = UDV^\top$ is the SVD of $M \in \R^{n \times m}$ and entries of $D$ are sorted in descending order. Let $U_{(k)} \in \R^{n \times k},V_{(k)}\in \R^{m \times k}$ denote the matrices only including  the first $k$ columns of $U \in \R^{n \times n},V \in \R^{m \times m}$, respectively, and $D_{(k)}$ be the first $k\times k$ submatrix of $D$. Then $U_{(k)}D_{(k)}V_{(k)}^\top$ is called the top-$k$ (rank-$k$) truncated SVD of $M$.
\end{definition}

The truncated SVD can be used to approximate a matrix, and it is optimal in both Frobenius and spectral norms as follows.

\begin{theorem}[Eckart-Young theorem\citep{eckart1936approximation}: optimality of low rank matrix approximation] Let $M = \sum_{j=1}^{\min\{n,m\}} \sigma_j u_jv_j^\top$ be the SVD of matrix $M \in \R^{n \times m}$, and $M_k = \sum_{j=1}^k \sigma_j u_jv_j^\top$ be the {\em truncated} SVD of $M$. Then \(M_k\) is the best rank-$k$ approximation of $M$ in the senses:
\begin{align*}
\|M - M_k\| = \sigma_{k+1} & = \min_{\mbox{rank}(N) \le k} \|M-N\|,\\
\|M - M_k\|_F = \sqrt{\sum_{j=k+1}^{\min\{n,m\}} \sigma_j^2} &  =\min_{\mbox{rank}(N) \le k} \|M-N\|_F.
\end{align*}
\label{thm:eckartyoung}
\end{theorem}

In addition to the above theoretical guarantees on the optimality of low rank matrix approximation, the SVD of a matrix can be computed efficiently. For general matrices the computation takes time $O(nm\min\{n,m\})$. The truncated SVD can usually be computed much faster, especially when the $k$-th singular value $\sigma_k$ is significantly larger than the $(k+1)$-th singular value (see \cite{golub1996matrix} Section 8.2, together with the discussions in Section 8.6).

We conclude this section by stating the application of SVD in computing the pseudo-inverse of a matrix.

\begin{definition}[Moore-Penrose Pseudo-inverse, see~ \cite{moore1920reciprocal,bjerhammar1951application,penrose1955generalized}]\label{def:pseudoinverse} Given a matrix $M$ of rank $k$, suppose its top-$k$ truncated SVD is $M = UDV^\top$, then the pseudo-inverse of $M$ is defined as $M^\dag = VD^{-1}U^\top$.
\end{definition}

Let $P_r(=VV^\top)$ and $P_c (=UU^\top)$ to be the projection matrix to the row-span and column-span of $M$, respectively; the pseudo-inverse is the only matrix that satisfies $MM^\dag = P_c$ and $M^\dag M=P_r$.

In the next few sections we describe some of other major applications of SVD to data analysis.

\section{Principal Component Analysis}
In this section, we describe Principle Component Analysis (PCA) ~\citep{pearson1901liii,hotelling1933analysis}) as one of the very important and useful statistical methods for data analysis and transformation.
Given data points $x_1,x_2,\dotsc,x_n \in \R^d$ that for simplicity we assume are {\em centered} ($\sum_{i=1}^n x_i = 0$), we are often interested in the {\em covariance matrix} $M \in \R^{d \times d}$:
$$M := \frac{1}{n} \sum_{i=1}^n x_i x_i^\top$$
to describe the statistical properties of the data. This matrix measures how different coordinates of the data are correlated with each other. The covariance matrix $M$ is always positive semi-definite (PSD), and for PSD matrices the SVD always has a symmetric structure such that the left and right singular vectors are the same:
$$
M = UDU^\top = \sum_{j=1}^d \sigma_j u_j u_j^\top.
$$

Given the covariance matrix, we can easily compute the {\em variance} of the data when projected to a particular direction. Suppose $v$ is a unit vector, then we have
$$
\Var[\inner{v,x}] = \E[\inner{v,x}^2] = \E[v^\top xx^\top v] = v^\top \E[xx^\top] v = v^\top M v.
$$
Here we used the fact that matrices are linear operators, and the linearity of the expectation. From this calculation and the optimization view-point of SVD in Definition~\ref{def:svdopt}, it is immediately concluded that the top singular vector $u_1$ is the direction where the data has largest {\em variance} when projected to that direction, i.e., yielding the maximum $\frac{1}{n}\sum_{i=1}^n \inner{x_i,v}^2$. This direction is usually called the {\em principal component} as it is the direction where the data is the most ``spread out''. Similarly, the first $k$ singular vectors $u_1,u_2,\dotsc,u_k$ spans a subspace that has the maximum variance of all $k$-dimensional subspaces. Geometrically, we can view the covariance matrix of the data as an {\em ellipsoid}, and $u_1$ corresponds to the longest axis; see Figure~\ref{fig:pca} for such geometrical representation in 2-dimensional space.

\begin{figure}
\bc
\bp
\psfrag{u1}[l]{\tcr{$u_1$}}\psfrag{u2}[l]{\tcr{$u_2$}}
\includegraphics[width=.5\linewidth]{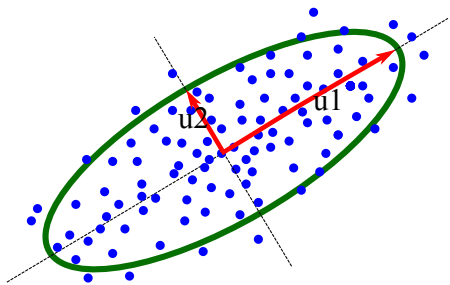}
\ep
\ec
\caption[PCA]{Geometric representation of principle components $u_1$ and $u_2$ for the covariance matrix of data points $x_i$'s.}
\label{fig:pca}
\end{figure}

\subsection{Dimensionality Reduction via PCA}
Principal components can be used to form a lower dimensional subspace and project the data to that subspace. This projection simplifies the data to a much lower dimensional space, while maintaining as much variance of the data as possible as we showed earlier. PCA is the most popular tool for {\em dimensionality reduction} and the main {\em linear} technique for doing that. In the case when the data is inherently low rank (recall the test scores example in Section~\ref{subsec:lowrank}) but may have some noise, doing PCA can often reduce the magnitude of noise (concrete settings where this can be proved includes mixture of Gaussians, see e.g. \cite{blum2016foundations} 3.9.3). 

We now formulate the dimensionality reduction problem more concretely, and provide the guarantee on the optimality of PCA.
Given $n$ data points $x_1,x_2,\dotsc,x_n \in \R^d$, we want to approximate these points with their projection to a lower dimensional subspace in $\Rbb^d$.
The question is what is the best $k$-dimensional ($k<d$) affine subspace in $\Rbb^d$ for such approximations, in the sense that the average distance between the original and approximate points is minimized, i.e.,
\begin{align} \label{eqn:distance min}
(P^*,p_0^*) := & \argmin_{\substack{P \in \Rbb^{d \times d} \\ p_0 \in \Rbb^d}} \frac{1}{n} \sum_{i \in [n]} \left\| x_i - (Px_i + p_0) \right\|^2, \\
\operatorname{s.t.} \ & \rank(P)=k, P^2=P. \nn
\end{align}
Here, $Px_i + p_0$ is the projection of $x_i$ to the $k$-dimensional affine subspace in $\Rbb^d$. This projection is specified by projection operator $P \in \Rbb^{d \times d}$ and displacement vector $p_0 \in \Rbb^d$ (Here, we assume the data points are not necessarily centered).  Note that since the projection is on a $k$-dimensional subspace, we have $\rank(P)=k$. The following theorem shows that PCA is the optimal solution to this problem. This can be proved as a direct corollary of Theorem~\ref{thm:eckartyoung} on optimality of low rank matrix approximation using SVD.

\begin{theorem}[PCA is optimal solution of~\eqref{eqn:distance min}] \label{thm: PCA optimality}
Given $n$ data points $x_1,x_2,\dotsc,x_n \in \R^d$, let $\mu \in \Rbb^d$ and $M \in \Rbb^{d \times d}$ denote the corresponding mean vector and covariance matrix, respectively. Let $M$ have SVD decomposition (the same as eigen-decomposition here) $M = U D U^\top$. Then, the optimal solutions of \eqref{eqn:distance min} are given by
\begin{align*}
P^* & = U_{(k)} U_{(k)}^\top, \\
p_0^* &= (I-P^*) \mu,
\end{align*}
where $U_{(k)} := [u_1 | u_2 | \dotsb | u_k] \in \Rbb^{d \times k}$ is the matrix including the top $k$ eigenvectors of $M$.
\end{theorem}

\begin{figure}
    \centering
    \includegraphics[width=0.5\linewidth]{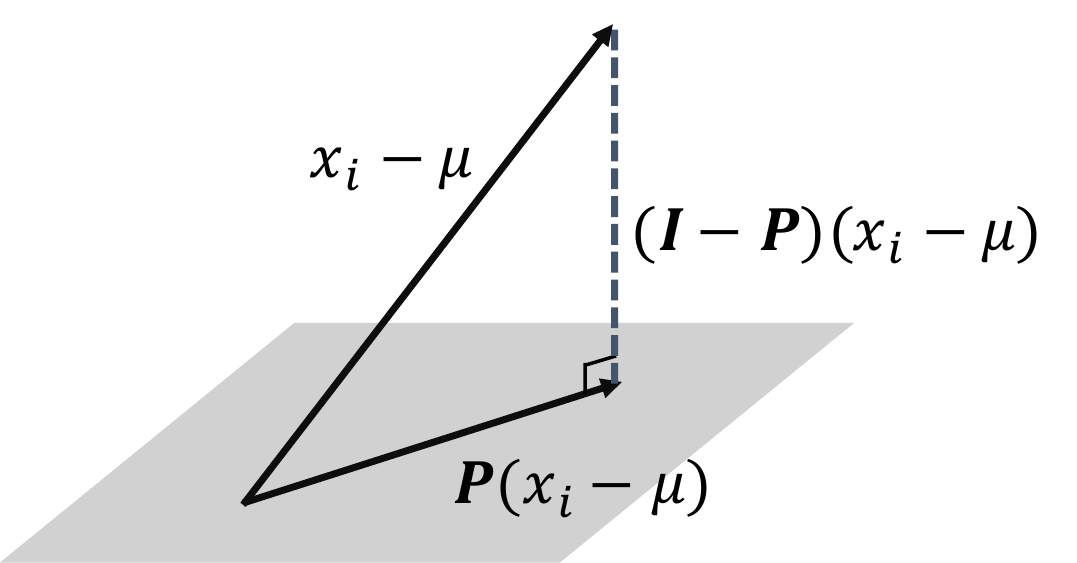}
    \caption{Visualization of the Pythagorean relation used in the proof of Theorem~\ref{thm: PCA optimality}.}
    \label{fig:th2_4}
\end{figure}

\bprf
Fixing $P$, the $p_0$ which minimizes cost function in \eqref{eqn:distance min} is $p_0^* = (I-P) \mu$.
Therefore, we have
\begin{align*}
\sum_{i \in [n]} \left\| x_i - (Px_i + p_0) \right\|^2 
&= \sum_{i \in [n]} \left\| (I-P)(x_i - \mu) \right\|^2 \\
&= \sum_{i \in [n]} \left\| x_i - \mu \right\|^2 - \sum_{i \in [n]} \left\| P(x_i - \mu) \right\|^2,
\end{align*}
where we used Pythagorean theorem in the last equality; see Figure~\ref{fig:th2_4} for its visualization. Therefore, the optimal solution $P^*$ maximizes the variance of projected points into the lower dimensional subspace as
\begin{align*}
\var(PX) = \frac{1}{n} \sum_{i \in [n]} \left\| P(x_i - \mu) \right\|^2
&= \frac{1}{n} \sum_{i \in [n]} (x_i - \mu)^\top P^\top P (x_i - \mu) \\
&= \frac{1}{n} \sum_{i \in [n]} \tr \left[ P (x_i - \mu) (x_i - \mu)^\top P^\top \right] \\
&= \tr \left[ P M P^\top \right].
\end{align*}
From Rayleigh quotient argument, we know that for the case of $k=1$, the $P^*$ which maximizes above is $P^* = u_1 u_1^\top$. Similar argument can be extended to larger $k$ which leads to $P^*=U_{(k)} U_{(k)}^\top$.
\eprf

From the above proof, we again see that PCA selects the lower dimensional subspace which has the maximum variance of projected points. Of course, the quality of this approximation still depends on the rank $k$ that we choose. When the data is assumed to come from a generative model, one can often compute $k$ by looking for a spectral gap (see e.g., Chapters 7 and 9 in \cite{blum2016foundations}). In practice, one can first choose an accuracy and then find the smallest $k$ that achieves the desired accuracy.

\section{Whitening Transformation}
\label{sec:whitening}
Another popular application of Singular Value Decomposition is to transform the data into {\em isotropic} position. We call a data set $z_1,z_2,\dotsc,z_n \in \R^d$ isotropic or {\em whitened} if the covariance matrix $$M_z := \frac{1}{n} \sum_{i=1}^n z_iz_i^\top = I_d,$$ where $I_d$ denotes the $d$-dimensional identity matrix. This basically means that the data has the {\em same} amount of variance in every direction. Whitening transformation has been discovered and applied in many domains~\citep{friedman1987exploratory,koivunen1999feasibility}. The benefit of whitening transformation is that the result is {\em invariant} under linear transformations of the original data. Raw data is often not measured in the most natural way \--- think again about the test score example in Section~\ref{subsec:lowrank}, it is possible that a math exam is graded in 100 points and  a writing exam has points in the range of 0 to 5. A na\"ive algorithm might incorrectly think that correlations with math exam is much more important because the scale is 20 times more than the writing exam. However, change of scaling is also a linear transformation, therefore applying  whitening transformation to the data can avoid these misconceptions.

In order to do this, suppose the original data is $x_1,x_2,\dotsc,x_n \in \R^d$ whose covariance matrix $M_x := \frac{1}{n} \sum_{i=1}^n x_ix_i^\top$ is not the identity matrix. Intuitively, we would like to shrink the directions that have more variance and stretch the directions that have less variance. This can again be done by SVD as follows. Recall $M_x$ is a PSD matrix whose SVD can be written as $M_x = UDU^\top$. Construct the {\em whitening matrix}
\begin{equation} \label{eq:whitening-mat}
W := UD^{-1/2},
\end{equation}
and let $z_i := W^\top x_i$. Now we have
$$
M_z := \frac{1}{n} \sum_{i=1}^n z_iz_i^\top  = W^\top M_xW = D^{-1/2}U^\top (U D U^\top) U D^{-1/2} = I_d,
$$
and hence, the transformed data points $z_i$'s are isotropic. 

Note that the result of whitening transformation can be very fragile if the smallest singular value of the data is very close to 0; see the inversion in $D^{-1/2}$. In practice, whitening is often performed after we identify the important directions using Principle Component Analysis. Similar whitening idea is useful later for designing tensor decomposition algorithms which we describe in Section~\ref{sec:whitening-ten}.

\section{Canonical Correlation Analysis}
\label{sec:cca}
All of the techniques that we have discussed so far (SVD, PCA, whitening) focus on extracting the properties of a single data set. On the other side, often in practice we would also like to understand relationships between two different sets of data. In this section, we describe the Canonical Correlation Analysis (CCA)~\citep{hotelling1992relations}), which is a very useful method to analyze the cross-covariance matrix between two different data sets. Many of the concepts introduced earlier, such as SVD and whitening, are used to describe the CCA.

Consider two sets of data points $x_1,x_2,\dotsc,x_n\in \R^{d_1}$ and $y_1,y_2,\dotsc,y_n\in \R^{d_2}$. If we again use the test scores example from Section~\ref{subsec:lowrank}, the first set of vectors $x_i$ would represent the performance of student $i$ in different exams,  while the second set of vectors $y_i$ would represent other properties of the students, e.g., the student's future salary. A natural question is whether these two data sets are {\em correlated}. Intuitively, the correlation between two directions $u\in \R^{d_1}$ and $v\in \R^{d_2}$ can be defined as 
$\frac{1}{n} \sum_{i=1}^n \inner{u,x_i}\inner{v,y_i}$. However, this definition is not very robust \--- if we apply a linear transformation to $x$ (say we multiply the first coordinate of $x$ by $10^6$), then the maximum correlation is likely to be changed (in this case likely to have more weight on first coordinate of $x$). Therefore, to measure correlations robustly and get rid of the influence from individual data sets, Canonical Correlation Analysis tries to find the maximum correlation after {\em whitening} both $x$ and $y$; see Algorithm~\ref{algo:cca} for the details. In the remaining of this section, we describe how the CCA algorithm is designed and works.

\begin{algorithm}[t]
\caption{Canonical Correlation Analysis (CCA)}
\label{algo:cca}
\begin{algorithmic}[1]
\renewcommand{\algorithmicrequire}{\textbf{input}}
\renewcommand{\algorithmicensure}{\textbf{output}}
\REQUIRE Two data sets $x_1,x_2,\dotsc,x_n\in \R^{d_1}$ and $y_1,y_2,\dotsc,y_n\in \R^{d_2}$
\ENSURE Most cross-correlated directions between whitened pairs
\STATE Compute the covariance matrices $$M_x := \frac{1}{n} \sum_{i=1}^n x_ix_i^\top, \quad M_y := \frac{1}{n} \sum_{i=1}^n y_i y_i^\top.$$
\STATE Use SVD to compute the whitening matrices $W_x,W_y$; see Equation~\eqref{eq:whitening-mat}.
\STATE Compute the correlation matrix
$$M_{\tilde{x}\tilde{y}} := \frac{1}{n} \sum_{i=1}^n W_x^\top x_iy_i^\top W_y.$$
\STATE Use SVD to compute the left and right singular vectors $\{(\tilde{u}_j, \tilde{v}_j)\}$ for $M_{\tilde{x}\tilde{y}}$.
\RETURN $(W_x\tilde{u}_i, W_y\tilde{v}_i)$.
\end{algorithmic}
\end{algorithm}

Let $M_x,M_y$ be the covariance matrices of $\{x_i\}$'s and $\{y_i\}$'s, and let $W_x$ and $W_y$ be the corresponding whitening matrices; see Algorithm~\ref{algo:cca} for the precise definitions. Let
$$\tilde{x}_i := W_x^\top x_i, \quad \tilde{y}_i := W_y^\top y_i$$
be the whitened data. We would like to find the most correlated directions in this pair of whitened data, i.e., we would like to find unit vectors $\tilde{u},\tilde{v}$ such that $u^\top [\frac{1}{n}\sum_{i=1}^n \tilde{x}_i \tilde{y}_i^\top] v$ is maximized, i.e.,
$$
\tilde{u},\tilde{v} := \argmax_{\|u\| = \|v\| = 1} \frac{1}{n} \sum_{i = 1}^n \inner{u,\tilde{x}_i}\inner{v,\tilde{y}_i}.
$$
By Definition~\ref{def:svdopt}, it is immediate to see that $\tilde{u},\tilde{v}$ are actually the left and right top singular vectors of the cross-covariance matrix $M_{\tilde{x}\tilde{y}} := \frac{1}{n} \sum_{i=1}^n \tilde{x}_i \tilde{y}_i^\top$, and this pair of directions are where the two data sets are most correlated. It is also possible to define more pairs of vectors $(\tilde{u}_j, \tilde{v}_j)$'s that correspond to the smaller singular vectors of the same matrix.

Often we would like to interpret the direction in the original data sets instead of the whitened ones. To do that, we would like to find a vector $u$ such that $\inner{u,x_i} = \inner{\tilde{u},\tilde{x}_i}$. That is,
$$u^\top x_i = \tilde{u}^\top  \tilde{x}_i = \tilde{u}^\top W_x^\top x_i,$$
and thus, we need to have
$$u = W_x \tilde{u}.$$
On the other hand, by construction we have $W_x^\top M_x W_x = I$, and therefore, since $\tilde{u}$ has unit norm, $\tilde{u}^\top W_x^\top M_x W_x \tilde{u} = 1$, which is to say $u^\top M_x u = 1$ given above equality.
Similarly, the fact that $\tilde{u}_j$ and $\tilde{u}_l$ are orthogonal if $j\ne l$ means that $u_j^\top M_x u_l = 0$. In general, it is possible to define a different inner product
$$\inner{u_j,u_l}_{M_x} := u_j^\top M_x u_l,$$
and a corresponding vector norm $\|u\|_{M_x} := \sqrt{u^\top M_x u}$, and the vectors $u_j$'s will be orthonormal under this new inner product $<\cdot, \cdot>_{M_x}$. Similarly, $v_j$'s should be orthonormal under the inner product $<\cdot, \cdot>_{M_y}$. Using these constraints, we can describe Canonical Component Analysis more precisely as below

\begin{definition} [Canonical Correlation Analysis]\label{def:cca} Given two data sets $x_1,x_2,\dotsc,x_n\in \R^{d_1}$ and $y_1,y_2,\dotsc,y_n\in \R^{d_2}$ (without loss of generality, assume $d_1\le d_2$), let $$M_x := \frac{1}{n} \sum_{i=1}^n x_ix_i^\top, \quad M_y = \frac{1}{n} \sum_{i=1}^n y_i y_i^\top$$ be the corresponding covariance matrices, respectively. Canonical Correlation Analysis (CCA) finds a set of correlated directions $u_1,u_2,\dotsc,u_{d_1}$ and $v_1,v_2,\dotsc,v_{d_1}$ such that $u_j^\top M_x u_j = 1$, $v_j^\top M_y v_j = 1$. The top correlated directions $u_1,v_1$ are similar to  the top singular vectors as
$$
u_1,v_1 = \argmax_{u^\top M_x u = 1, v^\top M_y v = 1} \frac{1}{n} \sum_{i=1}^n \inner{u,x_i}\inner{v,y_i}.
$$
Similarly, the remaining most correlated directions are defined as the remaining singular vectors
$$
u_j,v_j = \argmax_{\begin{array}{c} u^\top M_x u = 1, v^\top M_y v = 1\\ \forall l < j \quad u^\top M_x u_l = 0, v^\top M_y v_l = 0\end{array}} \frac{1}{n} \sum_{i=1}^n \inner{u,x_i}\inner{v,y_i}.
$$
\end{definition}

The derivation of Canonical Correlation Analysis already gives an efficient algorithm as provided in Algorithm~\ref{algo:cca}. 
It is not hard to verify the correctness of this algorithm, because after the linear transforms $W_x$ and  $W_y$, the objective and constraints in Definition~\ref{def:cca} become exactly the same as those in Definition~\ref{def:svdopt}.

The idea of Canonical Correlation Analysis is widely used in data analysis. In particular, CCA can find directions that are ``aligned'' in two data sets. The same idea is also used in tensor decompositions to ``align'' different views of the data, see Section~\ref{sec:whitening-ten}.

Using the full SVD to compute the CCA can be expensive in practice. Recently there have been several works that give efficient algorithms for computing top CCA vectors over large data sets, e.g., see \citet{wang2016globally, ge2016efficient, allen2016first, allen2016doubly}.
\clearpage{}

\clearpage{}\chapter{Tensor Decomposition Algorithms} \label{ch:tensor-decomp}

In this chapter, we first introduce the basic concepts of tensors and state the tensor notations that we need throughout this monograph. In particular, we highlight why in many cases we need to use tensors instead of matrices and provide the guarantees on uniqueness of tensor decomposition. Then, we describe different algorithms for computing tensor decomposition.

Most of the materials in this chapter has appeared in existing literature, especially in  \cite{AnandkumarEtal:tensor12}. We do give more explanations on the whitening procedure in Section~\ref{sec:whitening-ten} and symmetrization procedure in Section~\ref{subsec:nonsymmetric}, which were used in many previous papers but were not explicitly discussed in their general forms. We also give a new perturbation analysis for tensor power method together with whitening procedure in Section~\ref{sec:robust_with_whiten}, which will be useful for many of the applications later in Section~\ref{ch:applications}. 

\section{Transition from Matrices to Tensors} \label{sec:totensors}

We can think of tensors as multi-dimensional arrays, and one of the easiest ways to get a tensor is by stacking matrices of the same dimensions resulting in third order tensors. Let us recall the test scores example proposed in Section~\ref{subsec:lowrank}. Now suppose each exam has two parts \--- {\em written} and {\em oral}. Instead of the single score matrix $M$ that we had before, we will now have two score matrices $M_{\text{written}}$ and $M_{\text{oral}}$ including the scores for written and oral exams, respectively. Similar to the earlier score matrix $M$,  the rows of these matrices are indexed by students, and their columns are indexed by  subjects/tests. When we stack these two matrices together, we get a $n\times m\times 2$ tensor, where the third dimension is now indexed by the test format (written or oral). See the tensor in the left hand side of Figure~\ref{fig:tenDecomp-example} as the stacking of two score matrices  $M_{\text{written}}$ and $M_{\text{oral}}$.

Recall the simplified hypothesis states that there are two kinds of intelligence \--- {\em quantitative} and {\em verbal}; see Section~\ref{subsec:lowrank} to review it. Different students have different strengths, and different subjects/tests also have different requirements. As a result, the score was assumed to be a bilinear function of these hidden components; see~\eqref{eqn:scoreExp-decomp}. Now with the third dimension, it is also reasonable to expect the two kind of intelligence might behave differently in different formats \--- intuitively, verbal skills might be slightly more important in oral exams. As a result, we can generalize the bilinear function to a {\em tri-linear} form as
\begin{align*}
\text{Score}(\text{student},\text{test}&,\text{format}) = \\
& \text{student}_{\text{verbal-intlg.}} \times \text{test}_{\text{verbal}} \times \text{format}_{\text{verbal}}  \\
& + \text{student}_{\text{quant-intlg.}} \times \text{test}_{\text{quant.}} \times \text{format}_{\text{quant.}},
\end{align*}
where  $\text{format}_{\text{verbal}}$ and $\text{format}_{\text{quant.}}$ denote the importance of verbal and quantitative intelligence in different formats, respectively.
Now similar to what we did for the matrices, we can propose the following formula as {\em decomposing} the tensor into the sum of two rank-1 components as
\begin{align} \label{eqn:example-tenDecomp}
(M_{\text{written}},M_{\text{oral}}) =& \  u_{\text{verbal}}\otimes v_{\text{verbal}}\otimes w_{\text{verbal}} \\
& + u_{\text{quant.}}\otimes v_{\text{quant.}}\otimes w_{\text{quant.}}. \nn
\end{align}
Here $\otimes$ is the {\em tensor/outer product} operator which we will formally define in the next subsection; see~\eqref{eqn:rank-1 tensor}. $u_{\text{verbal}}, u_{\text{quant.}} \in \R^n$ and $v_{\text{verbal}}, v_{\text{quant.}} \in \R^m$ are the same as in Section~\ref{subsec:lowrank}, and the new components $w_{\text{verbal}}, w_{\text{quant}} \in \R^2$ correspond to verbal/quantitative importance for different formats, e.g., $w_{\text{verbal}}(\text{oral})$ denotes the importance of verbal intelligence in tests with oral format. This is a natural generalization of matrix decomposition/rank to tensors, which is commonly referred to as the CP (CANDECOMP/PARAFAC)~\citep{hitchcock1927expression,carroll1970analysis,harshman1994parafac}) decomposition/rank of tensors; we will formally define that in~\eqref{eqn:tensordecomp}. In fact, the tensor CP decomposition in~\eqref{eqn:example-tenDecomp} can be thought as a shared decomposition of matrices $M_{\text{written}}$ and $M_{\text{oral}}$ along the first two modes (corresponding to vectors $u_{\text{verbal}}, u_{\text{quant.}}, v_{\text{verbal}}, v_{\text{quant.}}$) with extra weight factors which are collected in the third mode as vectors $w_{\text{verbal}}, w_{\text{quant}}$. This is graphically represented in Figure~\ref{fig:tenDecomp-example}.

\begin{figure}
\bc
\includegraphics[width=.7\linewidth]{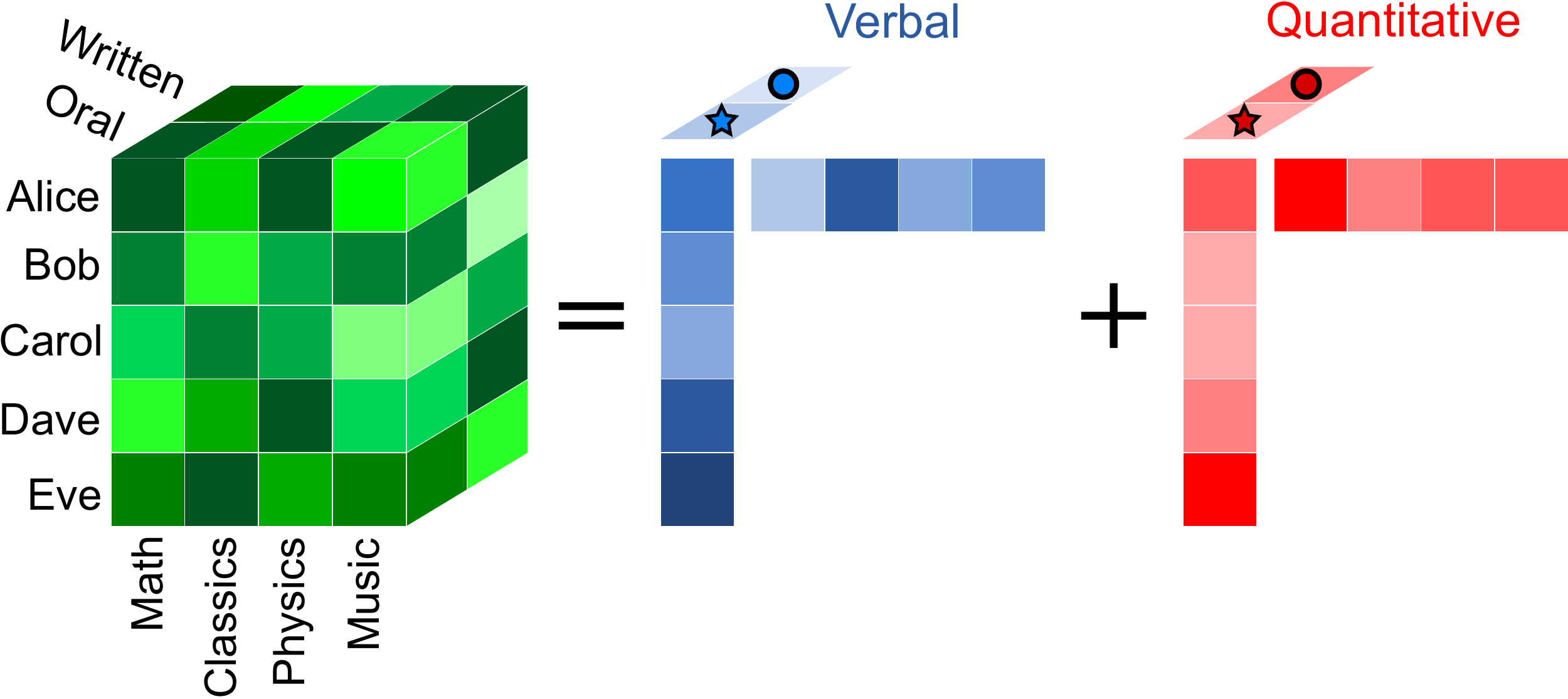}
\ec
\caption[]{Graphical representation of CP decomposition for score tensor in~\eqref{eqn:example-tenDecomp}. The matrix slices (oral and written) share a 2D-decomposition along the first two modes, and having the different weight factors in the third mode. The symbols $\star$ and $\circ$ represent the importance weight factors for oral and written formats, respectively.}
\label{fig:tenDecomp-example}
\end{figure}

\paragraph{Why using tensors instead of matrices?}Until now going to the tensor format just seems to make things more complicated. What additional benefits do we get? One important property of tensor decomposition is {\em uniqueness}. When we have only one matrix of test scores, the matrix decomposition is not unique most of the time; recall Figure~\ref{fig:nonUnique-decomposition} where we provided an example of this situation happening. The ambiguity makes it hard to answer even some of the most basic questions such as: which student has the best quantitative strength? On the other hand, under mild conditions (see Section~\ref{sec:unique-CP} for a formal discussion), the tensor decomposition is unique! Finding the unique decomposition allows us to pin down the vectors for students' strengths.

For learning latent variable models and latent representations, the {\em uniqueness} of tensor decomposition often translates to {\em identifiability}. We say a set of statistics makes the model {\em identifiable}, if there is only a unique set of parameters that can be consistent with what we have observed. Matrix decompositions usually correspond to pairwise correlations. Because of the ambiguities discussed earlier, for most latent variable models, pairwise correlations do not make the model identifiable. On the other hand, since tensor decompositions are unique, once we go to correlations between three or more objects, the models become identifiable. The example of learning a pure topic model was discussed in Section~\ref{sec:intro-LVMs} and many more examples are provided in Section~\ref{sec:applications}.

Tensor decomposition has also applications in many other areas such as chemometrics~\citep{appellof1981strategies}, neuroscience~\citep{mocks1988topographic}, telecommunications~\citep{sidiropoulos2000parallel}, data mining~\citep{acar2005modeling}, image compression and classification~\citep{shashua2001linear}, and so on; see survey paper by \citet{kolda_survey} for more references. 

\paragraph{Difficulties in Working with Tensors:}The benefit of unique decomposition comes at a cost. Although we can usually generalize notions for matrices to tensors, their counterpart in tensors are often not as well-behaved or easy to compute. In particular, tensor (CP) decomposition is much harder to compute than matrix decomposition. In fact, almost all tensor problems are NP-hard in the worst-case~\citep{TensorNPHard}. Therefore, we can only hope to find tensor decomposition in special cases. Luckily, this is usually possible when the rank of the tensor is much lower than the size of its modes which is true for many of the applications. Later in Sections~\ref{sec:orthogonal}--\ref{sec:als}, we will introduce algorithms for low rank tensor decomposition. When the rank of the tensor is high, especially when the rank is larger than the dimensions (which cannot happen for matrices), we may need more complicated techniques which we discuss in Section~\ref{sec:overcomplete-decomp}.

\section{Tensor Preliminaries and Notations} \label{sec:TenNotations}

In this section we describe some preliminary tensor concepts and provide formal tensor notations.

A real-valued \emph{$p$-th order tensor} $$T \in \bigotimes_{i=1}^p \R^{d_i}$$ is a member of the outer product of Euclidean spaces $\R^{d_i}$, $i \in [p]$, where $[p] := \{1,2,\dotsc,p\}$.
For convenience, we restrict to the case where $d_1 = d_2 = \dotsb = d_p = d$, and simply write $T \in \bigotimes^p \R^d$.
As is the case for vectors (where $p=1$) and matrices (where $p=2$), we may
identify a $p$-th order tensor with the $p$-way array of real numbers $[
T_{i_1,i_2,\dotsc,i_p} \colon i_1,i_2,\dotsc,i_p \in [d] ]$, where
$T_{i_1,i_2,\dotsc,i_p}$ is the $(i_1,i_2,\dotsc,i_p)$-th entry of $T$
with respect to a canonical basis.
A tensor is also called {\em symmetric} if the entry values are left unchanged by the permutation of any indices.
For convenience, we provide the concepts and results only for  third order tensors $(p=3)$ in the rest of this section. These can be similarly extended to higher order tensors.

\paragraph{Tensor modes, fibers and slices:}The different dimensions of the tensor are referred to as {\em modes}. For instance, for a matrix, the first mode refers to columns and the second mode refers to rows.
In addition, {\em fibers} are higher order analogues of matrix rows and columns. A fiber is obtained by fixing all but one of the indices of the tensor and is arranged as a column vector. For instance, for a matrix, its mode-$1$ fiber is any matrix column while a mode-$2$ fiber is any row. For a
third order tensor $T\in \R^{d \times d \times d}$, the mode-$1$ fiber is given by $T(:, j, l)$, mode-$2$ by $T(i, :, l)$ and mode-$3$ by $T(i, j, :)$ for fixed indices $i, j , l$.
Similarly, {\em slices} are obtained by fixing all but two of the indices of the tensor and are represented as matrices. For example, for the third order tensor $T\in \R^{d \times d \times d}$, the slices along $3$rd mode are given by $T(:, :, l)$. See Figure~\ref{fig:tenFiberSlices} for a graphical representation of tensor fibers and slices for a third order tensor.

\begin{figure}[t]
\bc
\includegraphics[width=.2\linewidth]{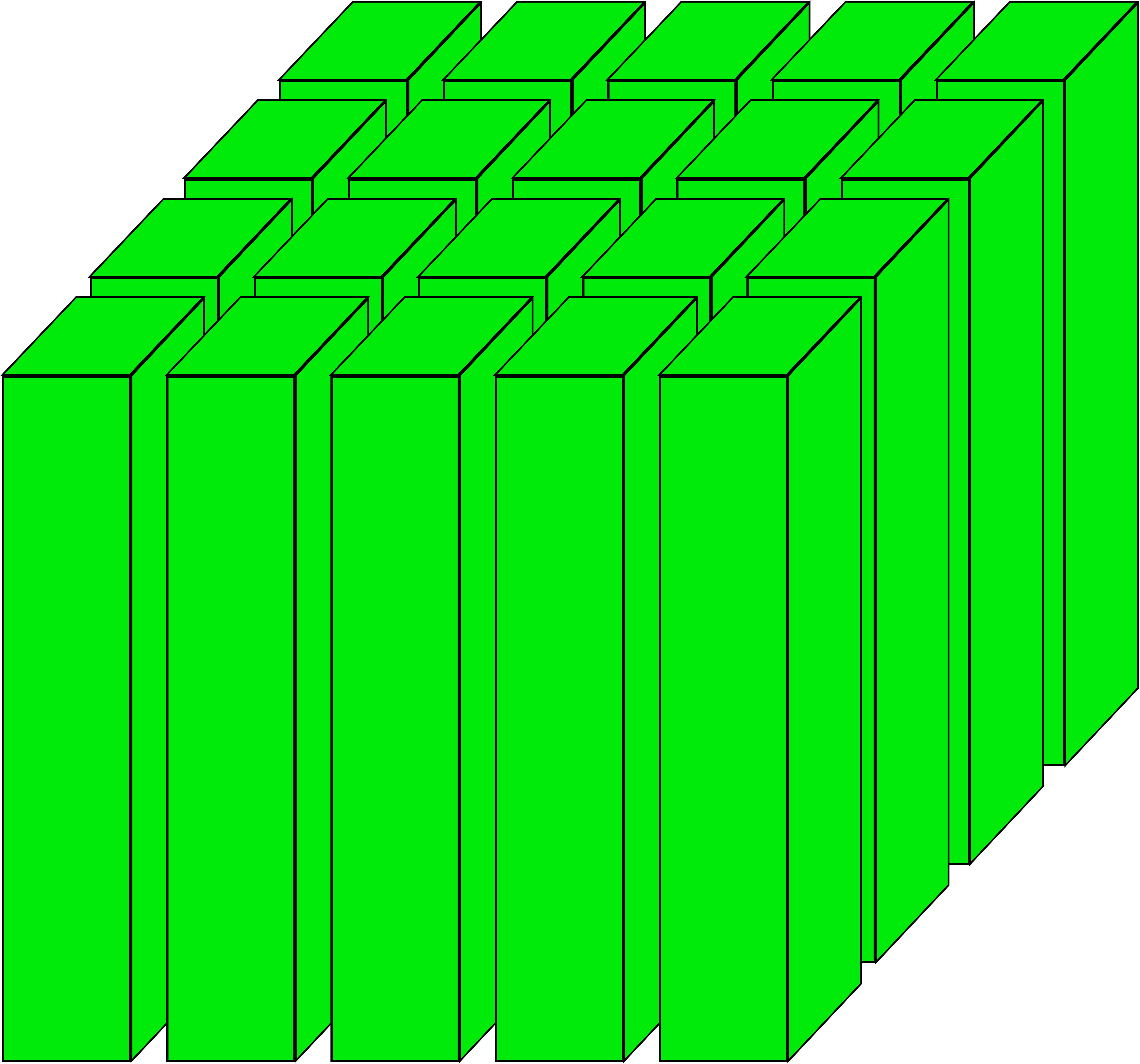}
\hspace{0.5in}
\includegraphics[width=.2\linewidth]{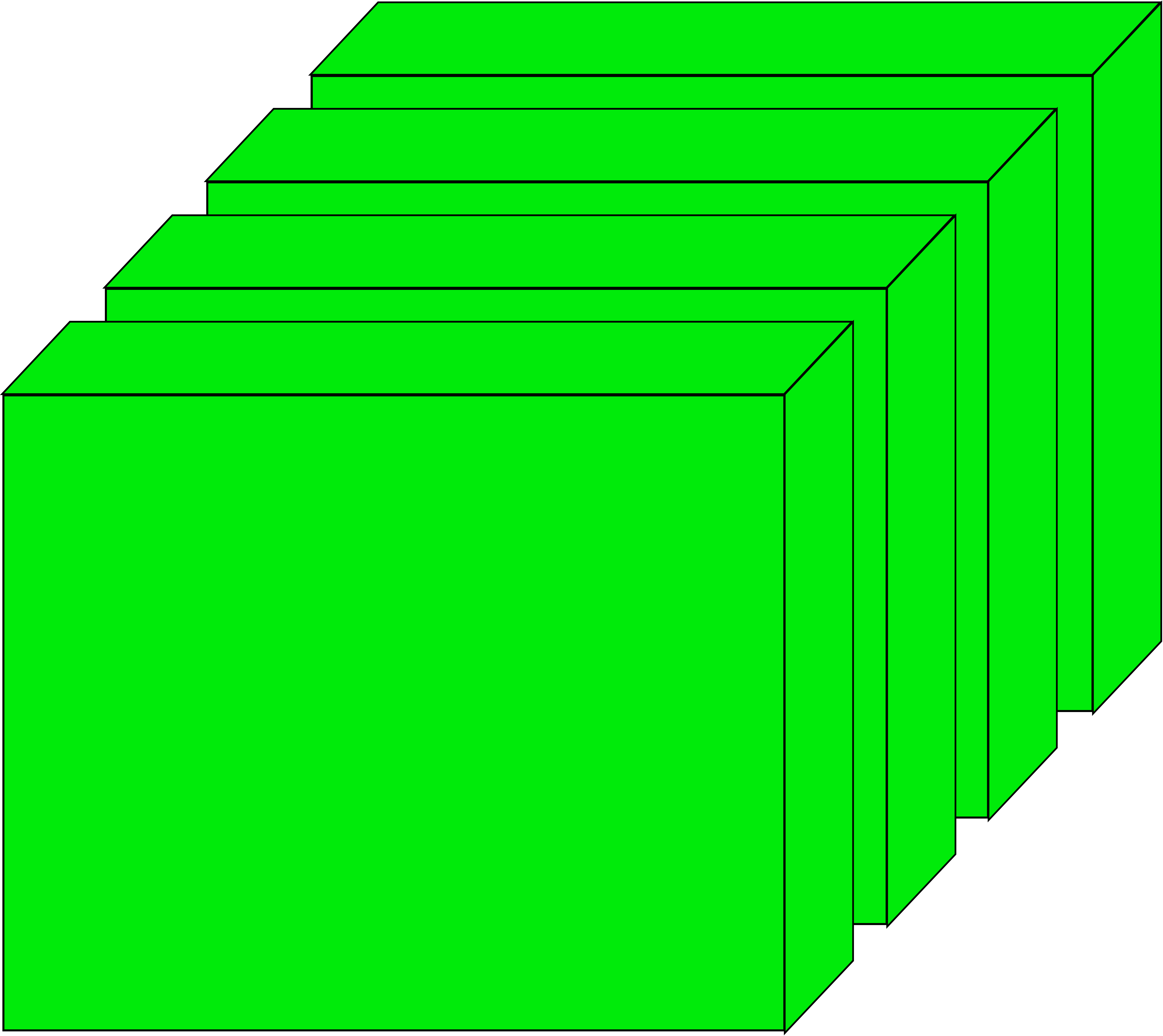}
\ec
\caption[]{Graphical representations of tensor fibers (left) and tensor slices (right) for a third order tensor.}
\label{fig:tenFiberSlices}
\end{figure}

\paragraph{Tensor matricization:}Transforming tensors into matrices is one of the ways to work with tensors.
For $r \in \{1,2,3\}$, the mode-$r$ matricization of a third order tensor $T\in \R^{d \times d \times d}$, denoted by $\mat(T,r) \in \Rbb^{d \times d^2}$, consists of all mode-$r$ fibers arranged as column vectors.
For instance, the matricized version along first mode denoted by $M \in \R^{d \times d^2}$ is defined such that
\begin{equation} \label{eqn:matricization}
T(i,j,l) = M(i,l+(j-1)d), \quad i,j,l \in [d].
\end{equation}

\paragraph{Multilinear transformation:}We view a tensor $T \in \Rbb^{d \times d \times d}$ as a multilinear form. 
Consider matrices $A,B,C \in \R^{d \times k}$. Then tensor $T(A,B,C) \in \R^{k \times k \times k}$ is defined such that
\begin{align} \label{eqn:multilinear form def}
T(A,B,C)_{j_1,j_2,j_3} := \sum_{i_1, i_2,i_3\in[d]} T_{i_1,i_2,i_3} \cdot A(i_1, j_1) \cdot B(i_2, j_2) \cdot C(i_3, j_3).
\end{align}
See Figure~\ref{fig:Tucker-intro} for a graphical representation of multilinear form.
In particular, for vectors $u,v,w \in \R^d$, we have
\begin{equation} \label{eqn:rank-1 update}
 T(I,v,w) = \sum_{j,l \in [d]} v_j w_l T(:,j,l) \ \in \R^d,
\end{equation}
which is a multilinear combination of the tensor mode-$1$ fibers.
Similarly $T(u,v,w) \in \R$ is a multilinear combination of the tensor entries,  and $T(I, I, w) \in \R^{d \times d}$ is a linear combination of the tensor slices. These multilinear forms can be similarly generalized to higher order tensors.

In the matrix case of $M \in \R^{d \times d}$, all above multilinear forms simplify to familiar matrix-matrix and matrix-vector products such that
\begin{align*}
M(A,B) &:= A^\top M B \in \R^{k \times k}, \\
M(I,v) &:= Mv = \sum_{j \in [d]} v_j M(:,j) \in \R^d.
\end{align*}

\begin{figure}
\bc
\includegraphics[width=.5\linewidth]{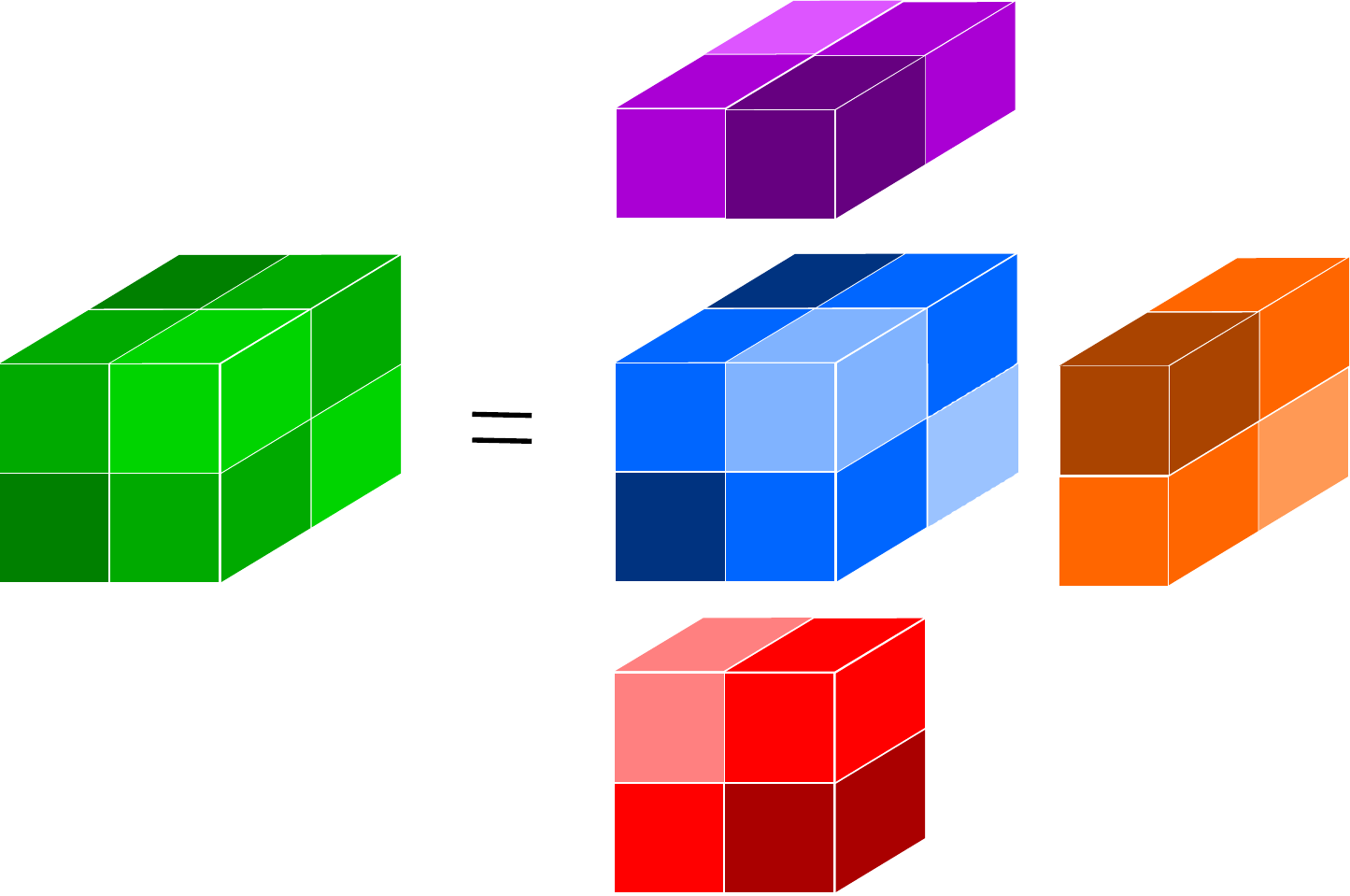}
\ec
\caption[Tensor as a multilinear transformation and representation of Tucker decomposition]{ Tensor as a multilinear transformation and representation of Tucker decomposition of a 3rd order tensor $T = \sum_{j_1,j_2,j_3 \in [k]} S_{j_1,j_2,j_3} \cdot a_{j_1} \otimes b_{j_2} \otimes c_{j_3} = S(A^\top,B^\top,C^\top)$}
\label{fig:Tucker-intro}
\end{figure}

\paragraph{Rank-1 tensor:}A $3$rd order tensor $T \in \Rbb^{d \times d \times d}$ is said to be rank-$1$ if it can be written in the form
\begin{align} \label{eqn:rank-1 tensor}
T= w \cdot a \otimes b\otimes c \Leftrightarrow T(i,j,l) = w \cdot a(i) \cdot b(j) \cdot c(l),
\end{align}
where notation $\otimes$  represents the {\em outer product} and $a \in \Rbb^d$, $b \in \Rbb^d$, $c \in \Rbb^d$ are unit vectors (without loss of generality) and $w \in \R$ is the magnitude factor.

Throughout this monograph, we also use notation $\cdot^{\otimes 3}$ to denote
$$
a^{\otimes 3} := a \otimes a \otimes a,
$$
for vector $a$.

\paragraph{Tensor CP decomposition and rank:}A tensor $T  \in \Rbb^{d \times d \times d}$ is said to have a CP (CANDECOMP/PARAFAC) rank $k\geq 1$ if $k$ is the {\em minimum} number that the tensor can be written as the sum of $k$ rank-$1$ tensors
\begin{equation}\label{eqn:tensordecomp}
T = \sum_{j\in [k]} w_j \ a_j \otimes b_j \otimes c_j, \quad w_j \in \Rbb, \ a_j,b_j,c_j \in \Rbb^d.
\end{equation}
See Figure~\ref{fig:CP-intro} for a graphical representation of CP decomposition for a symmetric 3rd order tensor.
This decomposition is also closely related to the multilinear form. In particular, given $T$ in~\eqref{eqn:tensordecomp} and vectors $\ha,\hb,\hc \in \Rbb^d$, we have
\begin{equation*}
T(\ha,\hb,\hc) = \sum_{j\in [k]} w_j \langle a_j, \ha\rangle\langle b_j, \hb\rangle\langle c_j,\hc\rangle.
\end{equation*}
Consider the decomposition in equation~\eqref{eqn:tensordecomp},
denote matrix $A:=[a_1  | a_2 | \dotsb | a_k] \in \R^{d \times k}$, and similarly $B$ and $C$.    Without loss of generality, we assume that the matrices have normalized columns (in $\ell_2$-norm), since we can always rescale them, and adjust the weights $w_j$ appropriately.

As we mentioned in the previous Section, the CP decomposition is often unique, which is very crucial to many machine learning applications. We will formally discuss that in Section~\ref{sec:unique-CP}.

\begin{figure}
$$
\vcenter{\hbox{\includegraphics[width=.13\linewidth]{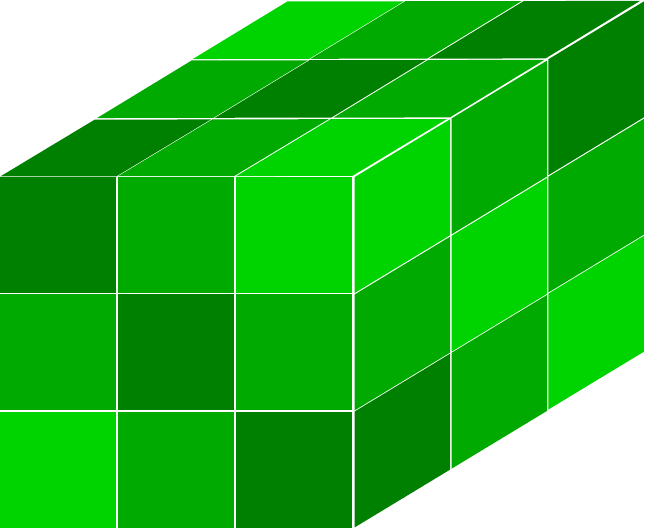}}}
= \vcenter{\hbox{\includegraphics[width=.1\linewidth]{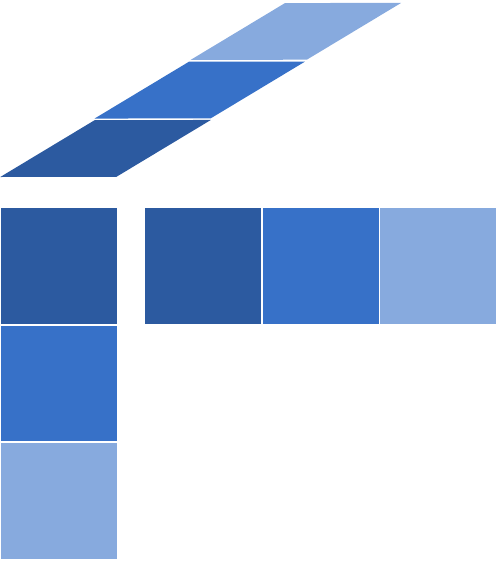}}}
+ \vcenter{\hbox{\includegraphics[width=.1\linewidth]{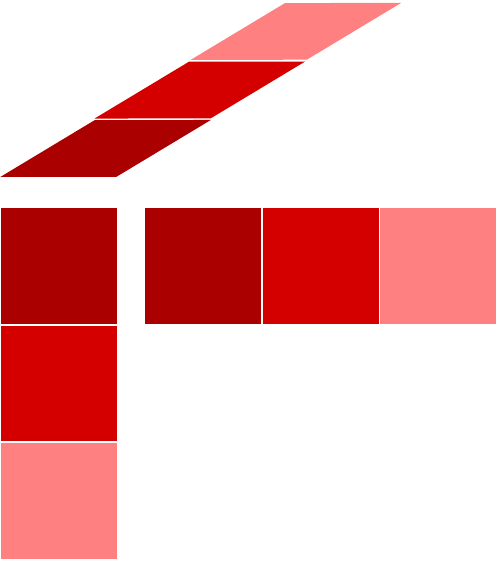}}}
+ \dotsb
$$
\caption[CP decomposition of a symmetric 3rd order tensor]{ CP decomposition of a symmetric 3rd order tensor $T = \sum_j a_j  \otimes a_j \otimes a_j$}
\label{fig:CP-intro}
\end{figure}

\paragraph{Tensor Tucker decomposition:}A tensor $T  \in \Rbb^{d_1 \times d_2 \times d_3}$ is said to have a Tucker decomposition or Tucker representation when given core tensor $S \in \Rbb^{k_1 \times k_2 \times k_3}$ and factor matrices $A \in \R^{d_1 \times k_1}, B \in \R^{d_2 \times k_2}, C \in \R^{d_3 \times k_3}$, it can be written as
\begin{equation}\label{eqn:TuckerDecomp}
T = \sum_{j_1 \in [k_1]} \sum_{j_2 \in [k_2]} \sum_{j_3 \in [k_3]} S_{j_1,j_2,j_3} \cdot a_{j_1} \otimes b_{j_2} \otimes c_{j_3}.
\end{equation}
See Figure~\ref{fig:Tucker-intro} for a graphical representation of Tucker representation.
Note that this is directly related to the multilinear from defined in~\eqref{eqn:multilinear form def} such that the R.H.S.\ of above equation is $S(A^\top,B^\top,C^\top)$.
Note that the CP decomposition is a special case of the Tucker decomposition when the core tensor $S$ is square (all modes having the same dimension) and diagonal. Unlike CP decomposition, Tucker decomposition suffers the same ambiguity problem as matrix decomposition. Therefore, we will focus on CP decomposition in this monograph. On the other hand, Tucker decomposition can be computed efficiently, which makes it a better choice for some applications other than learning latent variable models.

\paragraph{Norms:}For vector $v \in \R^d$,
$$\|v\| := \sqrt{\sum_{i \in [d]} v_i^2}$$ denotes the Euclidean ($\ell_2$) norm, and for matrix $M \in \R^{d \times d}$, the spectral (operator) norm is
$$\|M\| := \sup_{\|u\| = \|v\| = 1} |M(u,v)|,$$
where $|\cdot|$ denotes the absolute value operator.

Furthermore, $\|T\|$ and $\|T\|_F$ denote the spectral (operator) norm and the Frobenius norm of a tensor, respectively. In particular, for a $3$rd order tensor  $T \in \R^{d \times d \times d}$, we have:
\begin{align*}
\|T\| &:= \sup_{\|u\| = \|v\| = \|w\| = 1} |T(u,v,w)|, \\
\|T\|_F &:= \sqrt{\sum_{i,j,l \in [d]} T_{i,j,l}^2}.
\end{align*}

We conclude this section by reviewing some additional matrix notations and operators that we need throughout this monograph.

\paragraph{Matrix notations:} For a matrix $M$ with linearly independent rows, the right pseudo-inverse denoted by $M^\dagger$ (such that $MM^\dagger = I$)  is defined as
\begin{equation} \label{eqn:pseudo-inverse}
M^\dagger = M^\top (M M^\top)^{-1}.
\end{equation}

For matrices $A \in \R^{d_1 \times k}, B \in \R^{d_2 \times k}$, we introduce the following products. The {\em Khatri-Rao product}, also known as column-wise Kronecker product $C := A \odot B \in \R^{d_1d_2 \times k}$ is defined such that
\begin{equation} \label{eqn:KhatriRao}
C(l+ (i-1)d,j) = A(i,j) \cdot B(l,j), \quad i \in [d_1], l \in [d_2], j \in [k].
\end{equation}
Furthermore, when $d_1=d_2=d$, the {\em Hadamard product} $C := A * B \in \R^{d \times k}$ is defined as entry-wise product such that 
\begin{equation} \label{eqn:Hadamard}
C(i,j) = A(i,j) \cdot B(i,j), \quad i \in [d], j \in [k].
\end{equation}

\section{Uniqueness of CP decomposition} \label{sec:unique-CP}

When we are talking about the uniqueness of tensor CP decomposition, there are still some inherent uncertainties even in the formulation of the CP decomposition. For the following decomposition
\begin{equation*}
T = \sum_{j\in [k]} w_j \ a_j \otimes b_j \otimes c_j, \quad w_j \in \Rbb, \ a_j,b_j,c_j \in \Rbb^d,
\end{equation*}
we can obviously {\em permute} different rank-1 components, and the result will be the same tensor. We can also {\em scale} vectors $a_j,b_j,c_j$ and the weight $w_j$ simultaneously, as long as the product of all the scalings is equal to 1 and again the result will be the same tensor. The permutation and scaling ambiguities are inherent, and can often be addressed by the particular application. In the test scores example that we have revisited throughout this work, the permutation ambiguity means we get the two rank-1 components, but we do not know which one corresponds to the quantitative and which one corresponds to the verbal factor. In this case, intuitively we know a math test should require more quantitative skill, while a writing test should require more verbal skill. Therefore it should not be hard for a human to give names to the two hidden components. The scaling ambiguity is very similar to measuring quantities using different units, and we can often choose the appropriate scaling, e.g., by enforcing the strengths of students to be within 0-100. Regardless of the scaling/units we choose, the comparison between different students/subjects still makes sense \--- we can still safely answer questions like which student has the best quantitative strength.

Apart from above inherent ambiguities, there are several sufficient conditions for uniqueness of tensor decomposition. The most well-known condition is formulated by \citet{Kruskal:76,Kruskal:77}. We first provide the definition of Kruskal rank and then state this uniqueness condition.

\begin{definition}[Kruskal rank] \label{def:kruskal}
The Kruskal rank or $\krank$ of a matrix $A$ denoted by $\krank(A)$ is the maximum number $r$ such that every subset of $r$ columns of $A$ is linearly independent. 
\end{definition}

\begin{theorem}[\citep{Kruskal:76,Kruskal:77}] \label{thm:kruskal}
The CP decomposition in~\eqref{eqn:tensordecomp} is unique (up to permutation and scaling), if we let $A := [a_1 \ a_2 \ \dotsc \ a_k]$ (similarly for $B,C$, all with the same number of column $k$) satisfy the condition
\begin{equation*}
\krank(A) + \krank(B) + \krank(C) \geq 2k+2.
\end{equation*}
\end{theorem}

This is a mild condition when the rank of the tensor is not too high. As a comparison, matrix decomposition can be unique only when the matrix is rank 1, or we require strong conditions like orthogonality among components. In general, for non-degenerate cases when the components are in general position (with probability 1 for any continuous probability distribution on the components $A,B,C$), the $\krank$ of the matrices $A,B,C$ are equal to $\min\{k,d\}$ (\cite{Kruskal:76}, see a more robust version in \cite{bhaskara2014smoothed}). Therefore, when $2\le k\le d$, the Kruskal condition is always satisfied leading to unique tensor CP decomposition. Even when $k > d$ (rank is higher than the dimension), the Kruskal condition can be satisfied as long as $k \le 1.5d - 1$.

\section{Orthogonal Tensor Decomposition} \label{sec:orthogonal}

Tensor decomposition is in general a challenging problem. As a special and more tractable kind of decomposition, we introduce orthogonal tensor decomposition in this section. We review some useful properties of tensors that have orthogonal decomposition, and in the next section, we show how these properties lead to tensor power iteration as a natural algorithm for orthogonal tensor decomposition. It is worth mentioning here that not all tensors have orthogonal decomposition, and as we discussed in the previous section, the tensor decomposition can still be unique even when the tensor rank-1 components are not orthogonal.

We first review the spectral decomposition of symmetric matrices, and
then discuss a generalization to higher-order tensors.

\subsection{Review: Matrix Decomposition}

We first build intuition by reviewing the matrix setting, where the
desired decomposition is the eigen-decomposition of a symmetric rank-$k$
matrix $M = V \Lambda V^\t$, where $V = [v_1 | v_2 | \dotsb | v_k] \in
\R^{d \times k}$ is the matrix with orthonormal  ($V^\top V = I$) eigenvectors as columns,
and $\Lambda = \diag(\lambda_1,\lambda_2,\dotsc,\lambda_k) \in \R^{k \times
k}$ is diagonal matrix of non-zero eigenvalues.
In other words,
\begin{equation}
M = \sum_{j=1}^k \lambda_j \ v_j v_j^\t  = \sum_{j=1}^k \lambda_j
\  v_j^{\otimes 2}.
\label{eq:spectral-decomp}
\end{equation}
Such a decomposition is guaranteed to exist for every symmetric matrix; see~\cite{golub1996matrix}, Chapter 8.

Recovery of the $v_j$'s and $\lambda_j$'s can be viewed in at least two ways: {\em fixed} point and {\em variational} characterizations.

\subsubsection*{Fixed-point characterization}

First, each $v_i$ is a {\em fixed point} under the mapping $u \mapsto Mu$, up to a
scaling factor $\lambda_i$:
\[
Mv_i = \sum_{j=1}^k \lambda_j (v_j^\t v_i) v_j = \lambda_i v_i
\]
as $v_j^\t v_i = 0$ for all $j \neq i$ by orthogonality.
The $v_i$'s are not necessarily the only such fixed points.
For instance, with the multiplicity $\lambda_1 = \lambda_2 = \lambda$, then any linear combination
of $v_1$ and $v_2$ is also fixed under $M$.
However, in this case, the decomposition in~\eqref{eq:spectral-decomp} is
not unique, as $\lambda_1 v_1 v_1^\t + \lambda_2 v_2 v_2^\t$ is equal to $\lambda
(u_1 u_1^\t + u_2 u_2^\t)$ for any pair of orthonormal vectors $u_1$ and
$u_2$ spanning the same subspace as $v_1$ and $v_2$.
Nevertheless, the decomposition is unique when $\lambda_1, \lambda_2,
\dotsc, \lambda_k$ are distinct, whereupon the $v_i$'s are the only
directions fixed under $u \mapsto Mu$ up to non-trivial scaling; see Theorem~\ref{thm:SVD-unique}.

\subsubsection*{Variational characterization}

The second view of recovery is via the {\em variational characterization} of the
eigenvalues.
Assume $\lambda_1 > \lambda_2 > \dotsb > \lambda_k$; the case of repeated
eigenvalues again leads to similar non-uniqueness as discussed above.
Then the \emph{Rayleigh quotient}
\[ \frac{u^\t M u}{u^\t u} \]
is maximized over non-zero vectors by $v_1$.
Furthermore, for any $s \in [k]$, the maximizer of the Rayleigh quotient,
subject to being orthogonal to $v_1, v_2, \dotsc, v_{s-1}$, is $v_s$.
Another way of obtaining this second statement is to consider the
\emph{deflated} Rayleigh quotient
\[ 
\frac{u^\t \Bigl( M - \sum_{j=1}^{s-1} \lambda_j v_j v_j^\t
\Bigr) u}{u^\t u}, \]
and observe that $v_s$ is the maximizer. Also see that this statement is closely related to the optimization view-point of SVD provided in Definition~\ref{def:svdopt}.

Efficient algorithms for finding these matrix decompositions are well
studied \citep[Section 8.2.3]{GVL96}, and iterative power methods are one
effective class of algorithms.

We remark that in our multilinear tensor notation, we may write the maps $u
\mapsto Mu$ and $u \mapsto u^\t M u / \|u\|_2^2$ as
\begin{align}
u \mapsto Mu
& \ \equiv \ u \mapsto M(I,u) ,
\label{eq:matrix-map}
\\
u \mapsto \frac{u^\t M u}{u^\t u}
& \ \equiv \ u \mapsto \frac{M(u,u)}{u^\t u}
.
\label{eq:matrix-rq}
\end{align}

\subsection{The Tensor Case}
\label{sec:tensor-decomp}

Decomposing general tensors is a delicate issue; tensors may not even have
unique decomposition. But as we discussed earlier, tensors with orthogonal decomposition have a structure which permits a unique decomposition under a mild non-degeneracy condition.

An \emph{orthogonal decomposition} of a symmetric tensor $T \in \R^{d \times d \times d}$ is a collection of orthonormal (unit) vectors $\{ v_1,
v_2, \dotsc, v_k \}$
together with corresponding positive scalars $\lambda_j > 0$ such that
\begin{align}
T & = \sum_{j=1}^k \lambda_j v_j^{\otimes 3}
.
\label{eq:orthogonal-decomp}
\end{align}

In general, we say a $p$-th order symmetric tensor has an orthogonal decomposition if there exists a collection of  orthonormal (unit) vectors $\{ v_1,
v_2, \dotsc, v_k \}$
together with corresponding scalars $\lambda_j$ such that

\begin{align}
T & = \sum_{j=1}^k \lambda_j v_j^{\otimes p}
.\nonumber
\end{align}

Note that for odd order tensors (especially $p=3$)
, we can add the requirement that the $\lambda_j$ be positive.
This convention can be followed without loss of generality since
$-\lambda_j v_j^{\otimes p} = \lambda_j (-v_j)^{\otimes p}$ whenever $p$ is
odd.
Also, it should be noted that orthogonal decompositions do not necessarily
exist for every symmetric tensor. 

In analogy to the matrix setting, we consider two ways to view this
decomposition: a fixed-point characterization and a variational
characterization.
Related characterizations based on optimal rank-$1$ approximations are given
by~\citet{ZG01}.

\subsubsection*{Fixed-point characterization}

For a tensor $T$, consider the vector-valued map
\begin{equation} \label{eq:tensor-map}
u \mapsto T(I,u,u)
\end{equation}
which is the third-order generalization of~\eqref{eq:matrix-map}.

From the definition of multilinear form in~\eqref{eqn:rank-1 update}, this can be explicitly written as
\[
T(I,u,u) = \sum_{i \in [d]} \sum_{j,l \in [d]} T_{i,j,l} (e_j^\t u) (e_l^\t u) e_i,
\]
where $e_i$ denotes the $d$-dimensional basis vector with $i$-th entry equal to 1 and the rest of entries being zero.
Observe that~\eqref{eq:tensor-map} is \emph{not} a linear map, which is a
key difference compared to the matrix case where $Mu = M(I,u)$ is a linear map of $u$.

An eigenvector $u$ for a matrix $M$ satisfies $M(I,u) = \lambda u$, for
some scalar $\lambda$.
We say a unit vector $u \in \R^d$ is an \emph{eigenvector} of $T$, with
corresponding \emph{eigenvalue} $\lambda \in \R$, if
\[
T(I,u,u) = \lambda u .
\]
To simplify the discussion, we assume throughout that eigenvectors have
unit norm; otherwise, for scaling reasons, we replace the above equation with $T(I,u,u) =
\lambda \|u\| u$.
This concept was originally introduced
by~\citet{lim_tensoreig} and \citet{qi2005eigenvalues}\footnote{Note that there are many definitions of tensor eigenvalues and eigenvectors, see for example \cite{qi2005eigenvalues}. The definition we used here is called Z-eigenvalues/Z-eigenvectors in \cite{qi2005eigenvalues}.}.
For orthogonally decomposable tensors $T =
\sum_{j=1}^k \lambda_j v_j^{\otimes 3}$,
\[
T(I,u,u) = \sum_{j=1}^k \lambda_j (u^\t v_j)^2 v_j \ .
\]
By the orthogonality of the $v_i$'s, it is clear that $T(I,v_i,v_i) =
\lambda_i v_i$ for all $i \in [k]$.
Therefore, each $(v_i,\lambda_i)$ is an eigenvector/eigenvalue pair.

There are a number of subtle differences compared to the matrix case that
arise as a result of the non-linearity of~\eqref{eq:tensor-map}.
First, even with the multiplicity $\lambda_1 = \lambda_2 = \lambda$, a
linear combination $u := c_1 v_1 +c_2 v_2$ is not an eigenvector except in very special cases.
In particular,
\[ T(I,u,u)
= \lambda_1 c_1^2 v_1 +\lambda_2 c_2^2 v_2
= \lambda ( c_1^2 v_1 + c_2^2 v_2)
\]
may not be a multiple of $u = c_1 v_1 +c_2 v_2$.
This indicates that the issue of repeated eigenvalues does not have the
same status as in the matrix case.
Second, even if all the eigenvalues are distinct, it turns out that the $v_i$'s are not the only eigenvectors.
For example, set $u := (1/\lambda_1) v_1+(1/\lambda_2) v_2$.
Then,
\[ T(I,u,u)
= \lambda_1 (1/\lambda_1)^2 v_1 +\lambda_2 (1/\lambda_2)^2 v_2 = u
,
\]
so $u / \|u\|$ is an eigenvector with corresponding eigenvalue $\|u\|$.
More generally, for any subset $S \subseteq [k]$, the vector
\[ \sum_{j
\in S} \frac1{\lambda_j} \cdot v_j \]
is an eigenvector after normalization.

As we now see, these additional eigenvectors can be viewed as spurious.
We say a unit vector $u$ is a \emph{robust eigenvector} of $T$ if there
exists an $\epsilon > 0$ such that for all $\theta \in \{ u' \in \R^d :
\|u' - u\| \leq \epsilon \}$, repeated iteration of the map
\begin{equation} \label{eq:map2}
\bar \theta \mapsto \frac{
T(I,\bar\theta,\bar\theta)}{\|T(I,\bar\theta,\bar\theta)\|} \ ,
\end{equation}
starting from $\theta$ converges to $u$.
Note that the map~\eqref{eq:map2} re-scales the output to have unit
Euclidean norm.
Robust eigenvectors are also called attracting fixed points
of~\eqref{eq:map2}; see, e.g.,~\cite{SIMAX-080148-Tensor-Eigenvalues}.

The following theorem implies that if $T$ has an orthogonal decomposition
as given in~\eqref{eq:orthogonal-decomp}, then the set of robust
eigenvectors of $T$ are precisely the set $\{v_1,v_2,\ldots v_k\}$,
implying that the orthogonal decomposition is unique.
For even order tensors, the uniqueness is true up to sign-flips of the
$v_i$'s.

\begin{theorem}[Uniqueness of orthogonal tensor decomposition]\label{thm:robust}
Let $T$ have an orthogonal decomposition as given in
\eqref{eq:orthogonal-decomp}. Then,
\begin{enumerate}
\item The set of $\theta \in \R^d$ which do not converge to some $v_i$
under repeated iteration of \eqref{eq:map2} has measure zero.

\item The set of robust eigenvectors of $T$ is equal to $\{ v_1, v_2,
\dotsc, v_k \}$.

\end{enumerate}
\end{theorem}
See \citet{AnandkumarEtal:tensor12} for the proof of the theorem which follows readily from simple
orthogonality considerations.
Note that \emph{every} $v_i$ in the orthogonal tensor decomposition is
\emph{robust}, whereas for a symmetric matrix $M$, for almost all initial points,
the map $\bar\theta \mapsto \frac{M\bar\theta}{\|M\bar\theta\|}$ converges
only to an eigenvector corresponding to the largest magnitude eigenvalue.
Also, since the tensor order is odd, the signs of the robust eigenvectors
are fixed, as each $-v_i$ is mapped to $v_i$ under~\eqref{eq:map2}.

\subsubsection*{Variational characterization}

We now discuss a variational characterization of the orthogonal
decomposition.
The \emph{generalized Rayleigh quotient}~\citep{ZG01} for a third-order
tensor is given by
\[
u \mapsto \frac{T(u,u,u)}{(u^\t u)^{3/2}}
,
\]
which can be compared to~\eqref{eq:matrix-rq}.
For an orthogonally decomposable tensor, the following theorem shows that a
non-zero vector $u \in \R^d$ is an \emph{isolated local
maximizer}~\citep{NW99} of the generalized Rayleigh quotient if and only if
$u = v_i$ for some $i \in [k]$.

\begin{theorem} \label{thm:variational}
Assume $d \geq 2$.
Let $T$ have an orthogonal decomposition as given in
\eqref{eq:orthogonal-decomp}, and consider the optimization problem
\[ \max_{u \in \R^d} \ T(u,u,u) \quad \text{s.t.} \ \|u\| = 1 . \]
\begin{enumerate}
\item The stationary points are eigenvectors of $T$.

\item A stationary point $u$ is an isolated local maximizer if and only if
$u = v_i$ for some $i \in [k]$.

\end{enumerate}
\end{theorem}
See \citet{ge2015escaping}[Section C.1] for the proof of the theorem. It is similar to local optimality analysis for ICA methods using
fourth-order cumulants~\citep{delfosse1995adaptive,frieze1996learning}.

Again, we see similar distinctions to the matrix case.
In the matrix case, the only local maximizers of the Rayleigh quotient are
the eigenvectors with the largest eigenvalue (and these maximizers take on
the globally optimal value).
For the case of orthogonal tensor forms, the robust eigenvectors are
precisely the isolated local maximizers.

An important implication of the two characterizations is that, for
orthogonally decomposable tensors $T = \sum_{j \in [k]} \lambda_j v_j^{\otimes 3}$, (i) the local maximizers of the
objective function $T(u,u,u) / (u^\t u)^{3/2}$ correspond
precisely to the vectors $v_j$ in the decomposition, and (ii) these local
maximizers can be reliably identified using a simple fixed-point iteration as in~\eqref{eq:map2}, i.e., the tensor analogue of the matrix power method.
Moreover, a second-derivative test based on $T(I,I,u)$ can be employed to
test for local optimality and rule out other stationary points.

\subsection{Beyond Orthogonal Tensor Decomposition} \label{sec:whitening-ten}

So far, we have considered tensors with orthogonal decomposition as in~\eqref{eq:orthogonal-decomp}. We now discuss how the problem of non-orthogonal tensor decomposition can be reduced to the orthogonal tensor decomposition, and therefore, we can use the orthogonal tensor decomposition algorithms to recover the rank-1 components.

As we alluded in Section~\ref{sec:whitening}, we can pre-process the tensor using a whitening procedure, which is described in more detail in Procedure~\ref{algo:whitening}. This procedure orthogonalizes the components of the input tensor. After recovering the rank-1 components of the orthogonal decomposition, we apply {\em un-whitening} procedure proposed in Procedure~\ref{algo:Un-whitening} to recover the rank-1 components of the original non-orthogonal tensor decomposition. The whitening procedure only works when the components of the original non-orthogonal tensor are linearly independent. Luckily for many machine learning applications (such as topic models, mixtures of high dimensional Gaussians) that we will talk about in Section~\ref{ch:applications}, it is natural to expect the true components to be linearly independent.

\floatname{algorithm}{Procedure}
\begin{algorithm}[t]
\caption{Whitening}
\label{algo:whitening}
\begin{algorithmic}[1]
\renewcommand{\algorithmicrequire}{\textbf{input}}
\renewcommand{\algorithmicensure}{\textbf{output}}
\REQUIRE Tensor $T = \sum_{j \in [k]} \lambda_j \cdot a_j^{\otimes 3} \in \Rbb^{d \times d \times d}$; matrix $M = \sum_{j \in [k]} \tl{\lambda}_j \cdot a_j^{\otimes 2} \in \R^{d \times d}$. Assume $a_j$'s are linearly independent.
\ENSURE Lower dimensional tensor with orthogonal rank-1 components
\STATE Compute the rank-k SVD, $M=U \Diag(\gamma) U^\top$, where $U \in \R^{d \times k}$ and $\gamma \in \R^k$.
\STATE Compute the whitening matrix $W:=U \Diag(\gamma^{-1/2}) \in \R^{d \times k}$.
\RETURN $T\left(W,W,W\right) \in \R^{k \times k \times k}$.
\end{algorithmic}
\end{algorithm}

We first elaborate on the whitening step, and analyze how the proposed Procedure~\ref{algo:whitening} works and orthogonalizes the components of input tensor. We then analyze the inversion of whitening operator showing how the components in the whitened space are translated back to the original space as stated in Procedure~\ref{algo:Un-whitening}.

\subsubsection{Whitening procedure}
Consider the non-orthogonal rank-$k$ tensor 
\begin{equation} \label{eqn:third-moment}
T = \sum_{j \in [k]} \lambda_j \cdot a_j^{\otimes 3},
\end{equation}
where the goal of whitening procedure is to reduce it to an orthogonal tensor form. To do this, we exploit a matrix $M$ which has the same rank-1 components as $T$ such that
\begin{equation} \label{eqn:second-moment}
M = \sum_{j \in [k]} \tl{\lambda}_j \cdot a_j^{\otimes 2}.
\end{equation}
 In case we do not have such matrix, we can generate it as random combination of slices of $T$ such that $M:=T(I,I,\theta) \in \R^{d \times d}$, where $\theta \sim \mathcal{N}(0,I_{d})$ is a random standard Gaussian vector. It is also worth mentioning that although we refer to the rank-$k$ SVD of matrix $M$ as $U \Diag(\gamma) U^\top$, it might be the case that matrix $M$ is not positive semi-definite and does not necessarily have such symmetric SVD. In that case, $U \Diag(\gamma) U^\top$ is basically the eigen-decomposition of symmetric matrix $M$ where the entries of vector $\gamma$ can be also negative. We can modify the whitening matrix as $W:=U \Diag(|\gamma|^{-1/2})$, where $|\cdot|$ denotes the entry-wise absolute value, and the rest of analysis in this section would go through with minor modifications. So, in the rest of this section, we assume the entries of $\gamma$ are all positive.
 
 Another complication is that given the tensor $T$ in \eqref{eqn:third-moment} and $M$ in \eqref{eqn:second-moment}, it is still impossible to uniquely determine $\norm{a_j}$, $\tl{\lambda}_j$ and $\lambda_j$. Indeed, if we scale the $j$-th component to $C\cdot a_j$ using a constant $C \ne 0$, one just needs to scale $\tl{\lambda}_j$ by a factor of $1/C^2$ and $\lambda_j$ by a factor of $1/C^3$ and both the tensor $T$ and matrix $M$ are preserved. We discuss this ambiguity in more details in Remark~\ref{rmk:whitening-scale}.
 
Let matrix $W \in \R^{d \times k}$ denote the whitening matrix, i.e.,
the whitening matrix $W$ in Procedure~\ref{algo:whitening} is constructed such that $W^\top M W = I_k$. Applying whitening matrix $W$ to the tensor $T = \sum_{j \in [k]} \lambda_j \cdot a_j^{\otimes 3}$, we have
\begin{align}
T(W,W,W)
&= \sum_{j \in [k]} \lambda_j \left( W^\top a_j \right)^{\otimes 3} \nn \\
&= \sum_{j \in [k]} \frac{\lambda_j}{\tl{\lambda}_j^{3/2}} \left( W^\top a_j \sqrt{\tl{\lambda}_j} \right)^{\otimes 3} \nn \\
&= \sum_{j \in [k]} \mu_j \cdot v_j^{\otimes 3}, \label{eqn:tensor-whitening-noiseless}
\end{align}
where 
we defined
\begin{equation} \label{eqn:mu}
\mu_j := \frac{\lambda_j}{\tl{\lambda}_j^{3/2}}, \quad v_j := W^\top a_j \sqrt{\tl{\lambda}_j}, \quad j \in [k],
\end{equation}
in the last equality.
Let  $V := [v_1 | v_2 | \dotsb | v_k] \in \R^{k \times k}$ denote the factor matrix for $T(W,W,W)$. Then,  we have
\begin{equation} \label{eqn:V}
V := W^\top A \Diag(\tl{\lambda}^{1/2}),
\end{equation}
and thus,
$$
V V^\top = W^\top A \Diag(\tl{\lambda}) A^\top W = W^\top M W =   I_k.
$$
Since $V$ is a square matrix, it is also concluded that $V^\top V = I_k$, and therefore, tensor $T(W,W,W)$ is whitened such that its rank-1 components $v_j$'s form an orthonormal basis. This discussion clarifies how the whitening procedure works.

\floatname{algorithm}{Procedure}
\begin{algorithm}[t]
\caption{Un-whitening}
\label{algo:Un-whitening}
\begin{algorithmic}[1]
\renewcommand{\algorithmicrequire}{\textbf{input}}
\renewcommand{\algorithmicensure}{\textbf{output}}
\REQUIRE Orthogonal rank-1 components $v_j \in \R^k, j \in [k]$.
\ENSURE Rank-1 components of the original non-orthogonal tensor
\STATE Consider matrix $M$ which was exploited for whitening in Procedure~\ref{algo:whitening}, and let $\tl{\lambda}_j, j \in [k]$ denote the corresponding coefficients as $M =A \Diag(\tl{\lambda}) A^\top$; see~\eqref{eqn:second-moment}. Note that we don't know $A$ so we need to get $\tl{\lambda}_j$ from other information, see Remark~\ref{rmk:whitening-scale}.
\STATE Compute the rank-k SVD, $M=U \Diag(\gamma) U^\top$, where $U \in \R^{d \times k}$ and $\gamma \in \R^k$.
\STATE Compute
$$a_j = \frac{1}{\sqrt{\tl{\lambda}_j}} U \Diag(\gamma^{1/2}) v_j, \quad j \in [k].$$
\RETURN $\left\{ a_j \right\}_{j \in [k]}$.
\end{algorithmic}
\end{algorithm}

\subsubsection{Inversion of the whitening procedure}
Let us also analyze the inversion procedure on how to transform $v_j$'s to $a_j$'s. The main step is stated in Procedure~\ref{algo:Un-whitening}. According to whitening Procedure~\ref{algo:whitening}, let $M=U \Diag(\gamma) U^\top$, $U \in \R^{d \times k}$, $\gamma \in \R^k$, denote the rank-k SVD of $M$. Substituting whitening matrix $W:= U \Diag(\gamma^{-1/2})$ in~\eqref{eqn:V}, and multiplying $U \Diag(\gamma^{1/2})$ from left, we have
$$
U \Diag(\gamma^{1/2}) V = U U^\top A \Diag(\tl{\lambda}^{1/2}).
$$
Since the column spans of $A \in \R^{d \times k}$ and $U \in \R^{d \times k}$ are the same (given their relations to $M$), $A$ is a fixed point for the projection operator on the subspace spanned by the columns of $U$. This projector operator is $UU^\top$ (since columns of $U$ form an orthonormal basis), and therefore, $UU^\top A = A$. Applying this to the above equation, we have
$$
A = U \Diag(\gamma^{1/2}) V \Diag(\tl{\lambda}^{-1/2}),
$$
i.e.,
\begin{equation} \label{eqn:un-whitening}
a_j = \frac{1}{\sqrt{\tl{\lambda}_j}} U \Diag(\gamma^{1/2}) v_j, \quad j \in [k].
\end{equation}

As we discussed before, in general one needs some additional information to determine the coefficients $\tl{\lambda}_j$'s.

\begin{remark}\label{rmk:whitening-scale}[Scaling Ambiguity in Whitening Procedure]
If one only has access to tensor $T$ in \eqref{eqn:third-moment} and matrix $M$ in \eqref{eqn:second-moment}, there is no way to uniquely determine $\norm{a_j}$, $\tl{\lambda}_j$ or $\lambda_j$. Between these three parameters, we already have two equations for any $j \in [k]$: 1) $\mu_j = \lambda_j/\tl{\lambda}_j^{3/2}$ as in~\eqref{eqn:mu}, and 2) $a_j = \frac{1}{\sqrt{\tl{\lambda}_j}} U \Diag(\gamma^{1/2}) v_j$ as in~\eqref{eqn:un-whitening}. Note that all other variables such as $\mu_j, v_j, \gamma, U$ can be computed from the tensor decomposition of the whitened tensor in \eqref{eqn:tensor-whitening-noiseless} and the SVD decomposition of matrix $M$. Therefore, the three parameters still have one degree of freedom which is captured by a scaling such that if ($a_j$, $\tl{\lambda}_j$, $\lambda_j$) is a set of parameters that is consistent with $M$ and $T$, then for any $C \ne 0$, ($Ca_j$, $\tl{\lambda}_j/C^2$, $\lambda_j/C^3$) is also a set of parameters that is consistent with $M$ and $T$.

There are many cases where one might have additional information to determine the exact values of $\norm{a_j}$, $\tl{\lambda}_j$ or $\lambda_j$. For some applications (such as topic modeling in Section~\ref{subsec:lda}), the components $a_j$'s may have unit $\ell_1$ or $\ell_2$ norm, in which case we should scale $a_j$ accordingly. For some other applications such as pure topic model in Section~\ref{sec:singletopic} or mixture of Gaussians in Section~\ref{sec:spherical-Gaussian}, we know $\lambda_j = \tl{\lambda}_j$, and therefore, both of them are equal to $\mu_j^{-2}$.

When $\tl{\lambda}_j$'s are unknown at the time of running Procedure~\ref{algo:Un-whitening}, one can simply choose $\tl{\lambda}_j = 1$. If there is no additional information the results will give one set of parameters that are consistent with $M$ and $T$. If additional information is available one can apply correct normalization afterwards.
\end{remark}

\subsection{Beyond Symmetric Tensor Decomposition}
\label{subsec:nonsymmetric}
In the previous sections, we considered symmetric tensor decompositions as in~\eqref{eqn:third-moment}. In some applications, the tensor we have access to might be asymmetric. Consider 
\begin{equation} \label{eqn:tenDecomp-asymm}
T = \sum_{j\in[k]} \lambda_j \ a_j\otimes b_j\otimes c_j,
\end{equation}
where $\{a_j\},\{b_j\},\{c_j\}$'s are three groups of vectors that are linearly independent within the group. Here, we cannot directly apply the techniques for symmetric tensor decomposition. However, there is a way to transform this tensor to a symmetric one if we have access to some extra matrices. We discuss this process in this section which is a combination of whitening approach proposed in the previous section and the idea of CCA for matrices stated in Section~\ref{sec:cca}.
Similar to the {\em whitening} procedure described earlier, the symmetrization step in this section only works if the tensor components ($\{a_j,j\in[k]\}$, $\{b_j,j\in[k]\}$, $\{c_j,j\in[k]\}$) are all linearly independent within their own mode. Again for many machine learning applications that requires this procedure (such as the Multi-view model and Noisy-Or networks in Section~\ref{ch:applications}), it is natural to assume that the components are indeed linearly independent.

We first elaborate on the symmetrization step, and discuss how the proposed Procedure~\ref{algo:symm} works by orthogonalizing and symmetrizing the components of the input tensor. We then analyze the inversion of this process showing how the components in the whitened/symmetrized space are translated back to the original space as stated in Procedure~\ref{algo:Un-symm}.

\floatname{algorithm}{Procedure}
\begin{algorithm}[t]
\caption{Whitening and Symmetrization}
\label{algo:symm}
\begin{algorithmic}[1]
\renewcommand{\algorithmicrequire}{\textbf{input}}
\renewcommand{\algorithmicensure}{\textbf{output}}
\REQUIRE Tensor $T = \sum_{j \in [k]} \lambda_j \cdot a_j \otimes b_j \otimes c_j \in \Rbb^{d_1 \times d_2 \times d_3}$
\REQUIRE matrix $M_a = \sum_{j \in [k]} \tl{\lambda}_{a,j} \cdot a_j^{\otimes 2}$ (similarly $M_b$ and $M_c$)
\REQUIRE matrix $M_{a,b} = \sum_{j \in [k]} \tl{\lambda}_{ab,j} \cdot a_j \otimes  b_j$ (similarly $M_{a,c}$)\\
\COMMENT{Note that only $T, M_a, M_b, M_c, M_{a,b}, M_{b,c}, M_{a,c}$ are known (usually from moment estimates), we don't observe the components $a_j,b_j,c_j$.}
\ENSURE Lower dimensional symmetric tensor with orthogonal rank-1 components
\STATE Compute the rank-k SVD, $M_a=U_a \Diag(\gamma_a) U_a^\top$, where $U_a \in \R^{d_1 \times k}$ and $\gamma \in \R^k$; and similarly for $M_b$ and $M_c$.
\STATE Compute the whitening matrix $W_a:=U_a \Diag(\gamma_a^{-1/2}) \in \R^{d_1 \times k}$; and similarly $W_b$ and $W_c$.
\STATE Compute matrices $R_{a,b} := W_a^\top M_{a,b} W_b$, and $R_{a,c} :=  W_a^\top M_{a,c} W_c$.
\RETURN $T\left(W_a,W_b R_{a,b}^\top, W_c R_{a,c}^\top\right) \in \R^{k \times k \times k}$.
\end{algorithmic}
\end{algorithm}

\subsubsection{Symmetrization procedure}
The whitening and symmetrization in Procedure~\ref{algo:symm} is adapted from whitening procedure for symmetric tensors stated in Procedure~\ref{algo:whitening} with two modifications: first, the whitening is performed for an asymmetric tensor vs.\ a symmetric tensor in Procedure~\ref{algo:whitening}, and second, an extra step for symmetrization of the tensor is added. Similar to the whitening procedure, there are also additional scaling issues (as in Remark~\ref{rmk:whitening-scale}) introduced by the symmetrization procedure, we discuss that later in Remark~\ref{rmk:symmetrization-scale}.

In order to transform the asymmetric tensor $T$ in~\eqref{eqn:tenDecomp-asymm} to a symmetric and orthogonal tensor, we first whiten the three modes of the tensor. Similar to the whitening argument in the previous section, let $W_a,W_b,W_c$ be the whitening matrices for different modes of the tensor; see~Procedure~\ref{algo:symm} for precise definitions. Following the same calculations as in the whitening section, we have
$$
T(W_a,W_b,W_c) = \sum_{j\in[k]} \hat{\mu}_j \ \hat{a}_j\otimes \hat{b}_j\otimes \hat{c}_j,
$$
where
\begin{equation*}
\hat{\mu}_j := \frac{\lambda_j}{\sqrt{\tl{\lambda}_{a,j}\tl{\lambda}_{b,j}\tl{\lambda}_{c,j}}}, \quad \hat{a}_j := W_a^\top a_j \sqrt{\tl{\lambda}_{a,j}}, \quad j \in [k].
\end{equation*}
$\hat{b}_j$ and $\hat{c}_j$ are similarly defined.
Same as before, we have transformed the tensor so that each mode now has orthogonal components, but the only difference is $\hat{a}_j$ may not be the same as $\hat{b}_j$ (or $\hat{c}_j$), and therefore, the tensor is not symmetric yet. We will resolve this by using the cross matrices $M_{a,b}, M_{b,c}$. The idea is very similar to CCA stated in~Section~\ref{sec:cca}. More precisely we have:

\begin{claim}
Let $R_{a,b} := W_a^\top M_{a,b} W_b$, then we have
$$
R_{a,b} = \sum_{j\in[k]} \tl{\mu}_j \cdot  \hat{a}_j \hat{b}_j^\top,
$$
where $\tl{\mu}_j:=\frac{\tl{\lambda}_{ab,j}}{\sqrt{\tl{\lambda}_{a,j}\tl{\lambda}_{b,j}}}$.
In particular, $R_{a,b}  \hat{b}_j = \tl{\mu}_j \hat{a}_j$.
\end{claim}

The claim follows from similar calculation as above for $T(W_a,W_b,W_c)$. 
Define $R_{a,c} := W_a^\top M_{a,c}W_c$, we can then use these matrices to transform between the vectors $\hat{a}_j$, $\hat{b}_j$ and $\hat{c}_j$. 
More precisely
\begin{align*}
T(W_a,W_b R_{a,b}^\top, W_c R_{a,c}^\top) & = \sum_{j\in[k]} \lambda_j (W_a^\top a_j)\otimes (R_{a,b}W_b^\top b_j)\otimes (R_{a,c}W_c^\top c_j)\\
& = \sum_{j\in[k]} \hat{\mu}_j \cdot \hat{a}_j\otimes (R_{a,b}\hat{b}_j) \otimes (R_{a,c}\hat{c}_j)\\
& = \sum_{j\in[k]} \mu_j \cdot \hat{a}_j^{\otimes 3},
\end{align*}
where $\mu_j := \frac{\lambda_j\tl{\lambda}_{ab,j}\tl{\lambda}_{ac,j}}{\tl{\lambda}_{a,j}^{3/2}\tl{\lambda}_{b,j}\tl{\lambda}_{c,j}}$. We have now transformed the tensor to a symmetric and orthogonal tensor whose components are $\{\hat{a}_j\}$'s, and techniques for symmetric orthogonal tensors can be applied to do the decomposition.

\floatname{algorithm}{Procedure}
\begin{algorithm}[t]
\caption{Inversion of Whitening and Symmetrization}
\label{algo:Un-symm}
\begin{algorithmic}[1]
\renewcommand{\algorithmicrequire}{\textbf{input}}
\renewcommand{\algorithmicensure}{\textbf{output}}
\REQUIRE Orthogonal rank-1 components $\hat{a}_j \in \R^k, j \in [k]$
\ENSURE Rank-1 components of the original non-orthogonal and asymmetric tensor
\STATE For all $j \in [k]$, compute
\begin{align*}
a_j &= \frac{1}{\sqrt{\tl{\lambda}_{a,j}}} U_a \Diag(\gamma_a^{1/2}) \hat{a}_j, \\
b_j &= \frac{\sqrt{\tl{\lambda}_{a,j}}}{\tl{\lambda}_{ab,j}} U_b \Diag(\gamma_b^{1/2}) R_{a,b}^\top \hat{a}_j, \\
c_j &= \frac{\sqrt{\tl{\lambda}_{a,j}}}{\tl{\lambda}_{ac,j}} U_c \Diag(\gamma_c^{1/2}) R_{a,c}^\top \hat{a}_j,
\end{align*}
where the variables are the same as in~Procedure~\ref{algo:symm}.
\RETURN $\left\{ (a_j, b_j, c_j) \right\}_{j \in [k]}$.
\end{algorithmic}
\end{algorithm}

\subsubsection{Inversion of the symmetrization procedure}
The inversion steps are provided in Procedure~\ref{algo:Un-symm}. The analysis of the algorithm and why it works is very similar to the inversion of whitening discussed in the previous section.
This technique is particularly useful for multi-view models that we will discuss in Section~\ref{sec:multi}. As we mentioned before, there are also uncertainties about the scaling in the case of symmetrization:

\begin{remark}
\label{rmk:symmetrization-scale}[Scaling Ambiguity in Whitening and Symmetrization Procedure]
If one only has access to tensor $T$, matrices $M_a$, $M_b$, $M_c$, $M_{a,b}$, $M_{b,c}$, $M_{a,c}$, there is no way to uniquely determine the 10 parameters ($\norm{a_j}$, $\norm{b_j}$, $\norm{c_j}$, $\tl{\lambda}_{a,j}$, $\tl{\lambda}_{b,j}$, $\tl{\lambda}_{c,j}$, $\tl{\lambda}_{ab,j}$, $\tl{\lambda}_{bc,j}$, $\tl{\lambda}_{ac,j}$, $\lambda_j$). The 7 known quantities $T$, $M_a$, $M_b$, $M_c$, $M_{a,b}$, $M_{b,c}$, $M_{a,c}$ give 7 equations over these 10 parameters. The additional degrees of freedom can be described as ($C_a \norm{a_j}$, $C_b \norm{b_j}$, $C_c \norm{c_j}$, $\tl{\lambda}_{a,j}/C_a^2$, $\tl{\lambda}_{b,j}/C_b^2$, $\tl{\lambda}_{c,j}/C_c^2$, $\tl{\lambda}_{ab,j}/(C_aC_b)$, $\tl{\lambda}_{bc,j}/(C_bC_c)$, $\tl{\lambda}_{ac,j}/(C_aC_c)$, $\lambda_j/(C_aC_bC_c)$), where $C_a,C_b,C_c$ are arbitrary nonzero constants.

As before, there are special cases where the scaling of $a_j,b_j,c_j$ is known, which leads to three additional equations to uniquely determine all the scalings. There are also special cases where all the coefficients are the same, in which case they are all going to be equal to $\hat{\mu}_j^{-2}$.

When $\tl{\lambda}_{a,j}$'s (and similarly, $\tl{\lambda}_{ab,j}$'s and $\tl{\lambda}_{ac,j}$'s) are unknown at the time of running Procedure~\ref{algo:Un-symm}, one can simply choose all of them to be equal to 1. If there is no additional information the results will give one set of parameters that are consistent with all the observed matrices and tensors. If additional information is available one can apply correct normalization afterwards.
\end{remark}

\section{Tensor Power Iteration}

In the previous section, we discussed that the robust fixed-points of the tensor power iteration in~\eqref{eq:map2}
$$
\bar \theta \mapsto \frac{
T(I,\bar\theta,\bar\theta)}{\|T(I,\bar\theta,\bar\theta)\|} \ ,
$$
 correspond to the rank-1 components of orthogonal tensor decomposition in~\eqref{eq:orthogonal-decomp}; see Theorem~\ref{thm:robust}. Therefore, the power iteration is a natural and useful algorithm to recover the rank-1 components of an orthogonal tensor decomposition~\citep[Remark
3]{SHOPM}.
We first state a simple convergence analysis for an orthogonally
decomposable tensor $T$, and then discuss the analysis for approximately orthogonally
decomposable tensors.

When only an approximation $\hat{T}$ to an orthogonally decomposable
tensor $T$ is available (\emph{e.g.}, when empirical moments are used
to estimate population moments), an orthogonal decomposition need not
exist for this perturbed tensor (unlike the case for matrices), and
a more robust approach is required to extract the approximate
decomposition.  Here, we propose such a variant in
Algorithm~\ref{alg:robustpower} and provide a detailed perturbation
analysis.

\subsection{Convergence analysis for orthogonally decomposable tensors}
\label{sec:power}

The following lemma establishes the quadratic convergence of the tensor
power method, \emph{i.e.}, repeated iteration of~\eqref{eq:map2}, for
extracting a single component of the orthogonal decomposition.
Note that the initial vector $\theta_0$ determines which robust eigenvector
will be the convergent point.
Computation of subsequent eigenvectors can be computed with deflation,
\emph{i.e.}, by subtracting appropriate terms from $T$.

\begin{lem} \label{lemma:fixed-point} [Tensor power iteration for orthogonally decomposable tensors]
Let $T \in \R^{d \times d \times d}$ have an orthogonal decomposition as given in
\eqref{eq:orthogonal-decomp}.
For a vector $\theta_0 \in \R^d$, suppose that the set of numbers $|\lambda_1
v_1^\t\theta_0|, |\lambda_2 v_2^\t\theta_0|, \dotsc, |\lambda_k
v_k^\t\theta_0|$ has a unique largest element;
without loss of generality, say $|\lambda_1 v_1^\t\theta_0|$ is this
largest value and $|\lambda_2 v_2^\t\theta_0|$ is the second largest value.
For $t = 1, 2, \dotsc$, let
\begin{equation*}
\theta_t
\ := \
\frac{T(I,\theta_{t-1},\theta_{t-1})}{\|T(I,\theta_{t-1},\theta_{t-1})\|} .
\end{equation*}
Then
\begin{equation*}
\|v_1 - \theta_t\|^2
\leq \biggl( 2 \lambda_1^2 \sum_{i=2}^k \lambda_i^{-2} \biggr)
\cdot \biggl|
\frac{\lambda_2 v_2^\t\theta_0}{\lambda_1 v_1^\t\theta_0} \biggr|^{2^{t+1}}
.
\end{equation*}
That is, repeated iteration of~\eqref{eq:map2} starting from $\theta_0$
converges to $v_1$ at a quadratic rate.
\end{lem}

To obtain all eigenvectors, we may simply proceed iteratively using
deflation, executing the power method on $T - \sum_{j \in [s]} \lambda_j v_j^{\otimes
3}$ after having obtained robust eigenvector/eigenvalue pairs $\{ (v_j,\lambda_j), j \in [s] \}$.

\begin{proof}
Let $\overline\theta_0, \overline\theta_1, \overline\theta_2, \dotsc$ be the
sequence given by 
$$\overline\theta_0 := \theta_0, \quad \overline\theta_t :=
T(I,\theta_{t-1},\theta_{t-1}), t \geq 1.$$
Let $c_i := v_i^\t\theta_0$ for all $i \in [k]$.
It is easy to check that
\begin{enumerate}
\item $\theta_t = \overline\theta_t /
\|\overline\theta_t\|$,
\item $\overline\theta_t = \sum_{i=1}^k \lambda_i^{2^t-1} c_i^{2^t} v_i$.
\end{enumerate}
Indeed, 
$$\overline\theta_{t+1}
= \sum_{i=1}^k \lambda_i (v_i^\t\overline\theta_t)^2 v_i
= \sum_{i=1}^k \lambda_i (\lambda_i^{2^t-1} c_i^{2^t})^2 v_i
= \sum_{i=1}^k \lambda_i^{2^{t+1}-1} c_i^{2^{t+1}} v_i.$$
Then
\begin{align*}
1 - (v_1^\t\theta_t)^2
& =
1 - \frac{(v_1^\t\overline\theta_t)^2}{\|\overline\theta_t\|^2}
= 1 - \frac{\lambda_1^{2^{t+1}-2} c_1^{2^{t+1}}}
{\sum_{i=1}^k \lambda_i^{2^{t+1}-2} c_i^{2^{t+1}}} \\
& \leq \frac{\sum_{i=2}^k \lambda_i^{2^{t+1}-2} c_i^{2^{t+1}}}
{\sum_{i=1}^k \lambda_i^{2^{t+1}-2} c_i^{2^{t+1}}} \\
& \leq \lambda_1^2 \sum_{i=2}^k \lambda_i^{-2} \cdot \biggl|
\frac{\lambda_2c_2}{\lambda_1c_1} \biggr|^{2^{t+1}}
.
\end{align*}
Since $\lambda_1 > 0$, we have $v_1^\t \theta_t > 0$ and hence,
$$\|v_1 - \theta_t\|^2 = 2(1 - v_1^\t \theta_t) \leq 2(1 - (v_1^\t \theta_t)^2),$$ as required.
\end{proof}

\subsection{Perturbation analysis of a robust tensor power method}
\label{sec:perturbation}

Now we consider the case where we have an approximation $\hat{T}$ to an
orthogonally decomposable tensor $T$.
Here, a more robust approach is required to extract an approximate
decomposition.
\cite{AnandkumarEtal:tensor12} gave such an algorithm (Algorithm~\ref{alg:robustpower}), and
provided a detailed perturbation analysis. We summarize the perturbation result here and give a generalization later in Section~\ref{sec:robust_with_whiten}.
For simplicity, we assume the tensor $\hat{T}$ is of size $k \times k
\times k$ as per the reduction from Section~\ref{sec:whitening-ten} where whitening procedure has been applied to the original tensor.
In some applications, it may be preferable to work directly with a $d
\times d \times d$ tensor of rank $k \leq d$ (as in
Lemma~\ref{lemma:fixed-point}); these results apply in that setting with
little modification.

\floatname{algorithm}{Algorithm}
\begin{algorithm}[t]
\caption{Robust Tensor Power Method}
\label{alg:robustpower}
\begin{algorithmic}[1]
\renewcommand{\algorithmicrequire}{\textbf{input}}
\renewcommand{\algorithmicensure}{\textbf{output}}
\REQUIRE symmetric tensor $\tilde{T} \in \R^{k \times k \times k}$, number of
iterations $L$, $N$.

\ENSURE the estimated eigenvector/eigenvalue pair; the deflated tensor.

\FOR{$\tau = 1$ to $L$}

\STATE Draw $\th{0}^{(\tau)}$ uniformly at random from the unit sphere in
$\R^k$.

\FOR{$t = 1$ to $N$}

\STATE Compute power iteration update
\begin{eqnarray}
\th{t}^{(\tau)} & := &
\frac{\tilde{T}(I, \th{t-1}^{(\tau)}, \th{t-1}^{(\tau)})}
{\|\tilde{T}(I, \th{t-1}^{(\tau)}, \th{t-1}^{(\tau)})\|}
\label{eq:power-update}
\end{eqnarray}

\ENDFOR

\ENDFOR

\STATE Let $\tau^* := \arg\max_{\tau \in [L]} \{ \tilde{T}(\th{N}^{(\tau)},
\th{N}^{(\tau)}, \th{N}^{(\tau)}) \}$.

\STATE Do $N$ power iteration updates \eqref{eq:power-update} starting from
$\th{N}^{(\tau^*)}$ to obtain $\hat\theta$, and set $\hat\lambda :=
\tilde{T}(\hat\theta,\hat\theta,\hat\theta)$.

\RETURN the estimated eigenvector/eigenvalue pair
$(\hat\theta,\hat\lambda)$; the deflated tensor $\tilde{T} - \hat\lambda \
\hat\theta^{\otimes 3}$.

\end{algorithmic}
\end{algorithm}

Assume that the symmetric tensor $T \in \R^{k \times k \times k}$ is
orthogonally decomposable, and that $\hat{T} = T + E$, where the
perturbation $E \in \R^{k \times k \times k}$ is a symmetric tensor with
small operator norm:
\[ 
\|E\| := \sup_{\|\theta\| = 1} |E(\theta,\theta,\theta)| . 
\]
In our applications that we will describe in Section~\ref{sec:applications}, $\hat{T}$ is the tensor formed
by using empirical moments, while $T$ is the orthogonally decomposable
tensor derived from the population moments for the given model.

The following theorem is similar to Wedin's perturbation theorem for
singular vectors of matrices~\citep{wedin1972perturbation} in that it bounds
the error of the (approximate) decomposition returned by
Algorithm~\ref{alg:robustpower} on input $\hat{T}$ in terms of the size of
the perturbation, provided that the perturbation is small enough. 

\begin{theorem}[\cite{AnandkumarEtal:tensor12}]
\label{thm:robustpower}
Let $\hat{T} = T + E \in \R^{k \times k \times k}$, where $T$ is a
symmetric tensor with orthogonal decomposition $T = \sum_{i=1}^k \lambda_i
v_i^{\otimes 3}$ where each $\lambda_i > 0$, $\{ v_1, v_2, \dotsc, v_k \}$
is an orthonormal basis, and $E$ is a symmetric tensor with operator norm
$\|E\| \leq \eps$.
Define $\lambdamin := \min\{ \lambda_i : i \in [k] \}$, and $\lambdamax :=
\max\{ \lambda_i : i \in [k] \}$.
There exists universal constants $C_1, C_2, C_3 > 0$ such that the
following holds.
Pick any $\eta \in (0,1)$, and suppose
\[
\epsilon \leq C_1 \cdot \frac{\lambdamin}{k} ,
\qquad
N \geq C_2 \cdot \biggl( \log(k) + \log\log\Bigl(
\frac{\lambdamax}{\eps} \Bigr) \biggr)
,
\]
and
\begin{multline*}
\sqrt{\frac{\ln(L/\log_2(k/\eta))}{\ln(k)}}
\cdot \Biggl( 1 - \frac{\ln(\ln(L/\log_2(k/\eta))) +
C_3}{4\ln(L/\log_2(k/\eta))} -
\sqrt{\frac{\ln(8)}{\ln(L/\log_2(k/\eta))}} \Biggr)
\\
\geq 1.02 \Biggl( 1 + \sqrt{\frac{\ln(4)}{\ln(k)}}
\Biggr)
.
\end{multline*}
(Note that the condition on $L$ holds with $L = \poly(k) \log(1/\eta)$.)
Suppose that Algorithm~\ref{alg:robustpower} is iteratively called $k$
times, where the input tensor is $\hat{T}$ in the first call, and in each
subsequent call, the input tensor is the deflated tensor returned by the
previous call.
Let $(\hat{v}_1,\hat\lambda_1), (\hat{v}_2,\hat\lambda_2), \dotsc,
(\hat{v}_k,\hat\lambda_k)$ be the sequence of estimated
eigenvector/eigenvalue pairs returned in these $k$ calls.
With probability at least $1-\eta$, there exists a permutation $\pi$ on
$[k]$ such that
\[
\|v_{\pi(j)}-\hat{v}_j\| \leq 8 \epsilon/\lambda_{\pi(j)}
, \qquad
|\lambda_{\pi(j)}-\hat\lambda_j| \leq 5\epsilon , \quad \forall j \in [k]
,
\]
and
\[
\biggl\|
T - \sum_{j=1}^k \hat\lambda_j \hat{v}_j^{\otimes 3} 
\biggr\| \leq 55\eps .
\]
\end{theorem}

One important difference from Wedin's theorem is that this is an
algorithm dependent perturbation analysis, specific to
Algorithm~\ref{alg:robustpower} (since the perturbed tensor need not
have an orthogonal decomposition).
Furthermore, note that Algorithm~\ref{alg:robustpower} uses multiple
restarts to ensure (approximate) convergence---the intuition is that by
restarting at multiple points, we eventually start at a point in which the
initial contraction towards some eigenvector dominates the error $E$ in our
tensor.
The proof shows that we find such a point with high probability within $L =
\poly(k)$ trials.
It should be noted that for large $k$, the required bound on $L$ is very
close to linear in $k$.

A final consideration is that for specific applications, it may be possible
to use domain knowledge to choose better initialization points.
For instance, in the topic modeling applications
(\emph{cf}.~Section~\ref{sec:singletopic}), the eigenvectors are related to
the topic word distributions, and many documents may be primarily composed
of words from just single topic.
Therefore, good initialization points can be derived from these
single-topic documents themselves, as these points would already be close
to one of the eigenvectors.

\subsection{Perturbation analysis of tensor power method with whitening}
\label{sec:robust_with_whiten}

A limitation of Theorem~\ref{thm:robustpower} is that it only applies to orthogonal decompositions, while in most applications one would need to first apply the whitening transformation in Procedure~\ref{algo:whitening} described in Section~\ref{sec:whitening-ten}. With matrix perturbation bounds, it is possible to analyze the robustness of the combined procedure of whitening and orthogonal tensor decomposition. Variants of such analysis has appeared before in several papers, such as \cite{AnandkumarEtal:community12COLT, janzamin2014matrix}, however they are specialized to the specific setting. In this subsection we will give guarantees for such a combined procedure in the general setting.

\begin{theorem}
\label{thm:robust_with_whiten}
Suppose the true matrix $M$ and tensor $T$ have the forms
\[
M = \sum_{i=1}^k \tl{\lambda}_i a_i a_i^\top, \quad T = \sum_{i=1}^k \lambda_i a_i^{\otimes 3},
\]
where $\{ a_1, a_2, \dotsc, a_k \}$ is not necessarily a set of orthogonal components.
Assume our algorithm only has access to noisy/perturbed versions 
\[
\hat{M} = M + E_M, \ \hat{T} = T+E_T, \quad \text{where} \ \|E_M\| \le \epsilon_M, \ \|E_T\| \le \epsilon_T.
\]
Let $\sigma_{\min}(M)$ be the smallest nonzero singular value of $M$.
Suppose $\epsilon_M \le \sigma_{\min}(M)/4$, let $\Lambda_{\min} := \min\{\lambda_i\tl{\lambda}_i^{-3/2}:i\in[k]\}$, $\Lambda_{\max} := \max\{\lambda_i\tl{\lambda}_i^{-3/2}:i\in[k]\}$, then there exists a universal constant $C$ such that 
\[
\epsilon_{T_W} := C\left(\frac{\epsilon_T}{\sigma_{\min}(M)^{3/2}} + \Lambda_{\max}\frac{\epsilon_M}{\sigma_{\min}(M)}\right).
\]
If $\epsilon_{T_W}$ (as $\epsilon$), $\Lambda_{\max}$ (as $\lambda_{\max}$), $\Lambda_{\min}$ (as $\lambda_{\min}$), $N$, $L$, $\eta$ satisfies the conditions in Theorem~\ref{thm:robustpower}, then Algorithm~\ref{alg:robustpower} combined with whitening/un-whitening Procedures~\ref{algo:whitening}~and~ \ref{algo:Un-whitening} finds pairs $(\hat{a}_1,\hat{\Lambda}_1), (\hat{a}_2,\hat{\Lambda}_2), ..., (\hat{a}_k,\hat{\Lambda}_k)$, such that with probability $\eta$ there exists a permutation $\pi$ on $[k]$ such that for all $j\in [k],$
\begin{align*}
\|\sqrt{\tl{\lambda}_{\pi(j)}}a_{\pi(j)} - \hat{a}_j\| & \le \frac{9\tilde{\lambda}_{\pi(j)}^{3/2} \|M\|^{1/2}}{\lambda_{\pi(j)}} \epsilon_{T_W}, \\
|\lambda_{\pi(j)}\tl{\lambda}_{\pi(j)}^{-3/2} - \hat{\Lambda}_{j}| & \le 5\epsilon_{T_W}.
\end{align*}
\end{theorem}

Note that as we discussed earlier in Remark~\ref{rmk:whitening-scale}, without additional assumptions it is impossible to determine the scaling of $a_i$ together with $\lambda_i$ and $\tl{\lambda}_i$. The two quantities that we give perturbation bounds on ($\lambda_{\pi(j)}\tl{\lambda}_{\pi(j)}^{-3/2}$ and $\sqrt{\tl{\lambda}_{\pi(j)}}a_{\pi(j)}$) are the two quantities that are not effected by the scaling issue. In the special case when $\lambda_i = \tl{\lambda}_i$, the pair $(\hat{a}_j, \hat{\Lambda}_j)$ that we estimate allows us to estimate $a_{\pi(j)} \approx \hat{\Lambda}_j\hat{a}_j$ and $\lambda_{\pi(j)} \approx \hat{\Lambda}_j^{-2}$. 

From Theorem~\ref{thm:robust_with_whiten}, it is also clear that the error comes from both the whitening process ($\Lambda_{\max} \epsilon_M/\sigma_{\min}(M)$) and the estimation error in estimating the tensor ($\epsilon_T/\sigma_{\min}(M)^{3/2}$). If the second order moment estimate $M$ is not accurate enough, using this algorithm can suffer additional error. Empirically, it is often observed that alternating least squares (see Section~\ref{sec:als}) may perform better than using orthogonal tensor decomposition with whitening. However, we do want to emphasize that alternating least squares does not have the same provable guarantee as Theorem~\ref{thm:robust_with_whiten}.

To prove Theorem~\ref{thm:robust_with_whiten}, we first need to analyze the perturbation of the whitening matrix. We use Weyl's Theorem and Davis-Kahan Theorem to do that. We state special cases of these two theorems for the setting that we are interested in. For more general forms of these theorems and other matrix perturbation inequalities, see~\cite{stewartsun}.

\begin{theorem}[Weyl's Theorem~\citep{weyl1912asymptotische}]\label{thm:weyl}
Let $M\in \R^{d\times d}$ be a symmetric matrix, and $E\in \R^{d\times d}$ be a symmetric perturbation with $\|E\| \leq \epsilon$. Let $\lambda_i(M)$ be the $i$-th eigenvalue of $M$. We have 
\[
|\lambda_i(M) - \lambda_i(M+E)| \le \epsilon, \quad i \in [d].
\]
\end{theorem}

\begin{theorem}[Davis-Kahan Theorem~\citep{davis1970rotation}]\label{thm:daviskahan}
Let $M \in \R^{d\times d}$ and $\hat{M} = M+E \in \R^{d\times d}$ be symmetric PSD matrices with $\| E \| \leq \epsilon$. Suppose $M$ is rank $k$ and its truncated SVD is $M = UDU^\top$, where $U\in \R^{d\times k}$ and $D\in \R^{k\times k}$. The truncated (top-$k$) SVD of $\hat{M}$ is $\hat{U}\hat{D}\hat{U}^\top$. Let $U^\perp \in \R^{d\times (d-k)}$ be the orthogonal subspace of $U$ (that is, $UU^\top + (U^\perp)(U^\perp)^\top = I_d$). Then, we have
\[
\|(U^\perp)^\top \hat{U}\| \le \|E\|/\lambda_k(\hat{M}),
\]
where $\lambda_k(\hat{M})$ denotes the $k$-th eigenvalue of $\hat{M}$.
\end{theorem}

Using these two theorems, we will prove the following guarantees for the whitening procedure.

\begin{lemma}\label{lem:whitening_error}
Suppose $M \in \R^{d\times d}$ is a symmetric PSD matrix with rank k and $\sigma_{\min}(M)$ denotes its smallest (nonzero) singular value. Let $\hat{M} = M+E$ is also a symmetric matrix and $\epsilon:= \|E\| \le \sigma_{\min}(M)/4$. Let the truncated (top-$k$) SVD of $M$ and $\hat{M}$ be $UDU^\top$ and $\hat{U}\hat{D}\hat{U}^\top$, respectively. Then, there exists an orthonormal matrix $R\in \R^{k\times k}$ such that if we define $W:= UD^{-1/2}R$, $\hat{W}:= \hat{U}\hat{D}^{-1/2}$, $B:= UD^{1/2}R$, $\hat{B} = \hat{U}\hat{D}^{1/2}$, these matrices satisfy
\begin{align*}
    \|W - \hat{W}\| & \le \frac{5\epsilon}{\sigma_{\min}(M)^{3/2}},\\
    \|B^\top (W-\hat{W})\| & \le \frac{3\epsilon}{\sigma_{\min}(M)},\\
    \|B - \hat{B}\| &\le \frac{3\epsilon\sqrt{\|M\|}}{\sigma_{\min}(M)}.
\end{align*}
\end{lemma}

\begin{proof}
We first show that $U$ and $\hat{U}$ span similar subspace.
Let $U^\perp$ be the orthonormal subspace of $U$ (as in Theorem~\ref{thm:daviskahan}). By Weyl's Theorem (Theorem~\ref{thm:weyl}), we know 
\begin{equation} \label{eqn:Weyl-aux}
\lambda_k(\hat{M}) \ge \sigma_{\min}(M) - \|E\| \ge 3\sigma_{\min}(M)/4.    
\end{equation}
Therefore, by Davis Kahan Theorem (Theorem~\ref{thm:daviskahan}) we have
\begin{equation} \label{eqn:DavisKahan-aux}
\| (U^\perp)^\top \hat{U}\| \le 4\epsilon/3\sigma_{\min}(M).
\end{equation}
Now for $W-\hat{W}$, we have
\begin{align}
\|W-\hat{W}\| & = \|(UU^\top + (U^\perp)(U^\perp)^\top) (W-\hat{W})\| \nonumber \\
&\le \|U^\top (W-\hat{W})\| + \|(U^\perp)^\top (W-\hat{W})\|. \label{eqn:power-whitening-proof-aux1}
\end{align}
No matter what $R$ is, the second term can be bounded as
\begin{align}
\|(U^\perp)^\top (W-\hat{W})\| 
&= \|(U^\perp)^\top\hat{U}\hat{D}^{-1/2}\| \nonumber \\
& \leq \lambda_k(\hat{M})^{-1/2} \|(U^\perp)^\top\hat{U}\| \nonumber \\
& \le \frac{2\epsilon}{\sigma_{\min}(M)^{3/2}}, \label{eqn:power-whitening-proof-aux2}
\end{align}
where we used~\eqref{eqn:Weyl-aux}~and~\eqref{eqn:DavisKahan-aux} in the last inequality\footnote{Note that the exact constant in the last inequality is $(\frac{4}{3})^{3/2}$, and we replace it by 2 for simplicity. We will do similar relaxations to constants several times more later in the proof.}. Therefore, we only need to show that there exists an $R$ such that the first term $\|U^\top (W-\hat{W})\|$ is small.

Let $\bar{M} = \hat{U}\hat{D}\hat{U}^\top$, by Eckart-Young Theorem (Theorem~\ref{thm:eckartyoung}), we know $\|\bar{M} -\hat{M}\| \le \|E\| = \epsilon$, and thus, $\|\bar{M} - M\| \le 2\epsilon$. Now for $\hat{W}^\top (\bar{M} - M) \hat{W}$, we have
\[
\| \hat{W}^\top (\bar{M} - M) \hat{W} \| \le \frac{2\epsilon}{\sigma_k(\hat{M})},
\]
where we also used the fact that $\|\hat{W}\| = \sigma_k(\hat{M})^{-1/2}$.
Given $\hat{W}^\top \bar{M} \hat{W} = I_k$, the above inequality can be rewritten as
\begin{align*}
\| I - \hat{W}^\top M \hat{W} \|
\le \frac{2\epsilon}{\sigma_k(\hat{M})}
\le \frac{3\epsilon}{\sigma_{\min}(M)}.
\end{align*}

Let $P := U^\top \hat{W}$, then $\hat{W}^\top M \hat{W} = P^\top DP$. Since $3\epsilon/\sigma_{\min}(M) \le 3/4 < 1$, by Weyl's Theorem we know the eigenvalues of $P^\top DP$ are between $1\pm 3\epsilon/\sigma_{\min}(M) \in [1/4,7/4]$. There exists a diagonal matrix $\Delta$ (with $\|\Delta - I\|\le 3\epsilon/\sigma_{\min}(M)$) and an orthonormal matirx $R_1$ such that
\[
P^\top DP = R_1 \Delta R_1^\top.
\]
In other words, let $R_2 = D^{1/2}PR_1\Delta^{-1/2}$ (equivalently, $P = D^{-1/2}R_2\Delta^{1/2}R_1^\top$), we have $R_2^\top R_2 = \Delta^{-1/2}R_1^\top (P^\top DP) R_1D^{-1/2} =  I$, so $R_2$ is also orthonormal.

Now we can choose $R = R_2R_1^\top$, and therefore, the first term in~\eqref{eqn:power-whitening-proof-aux1} can be bounded as
\begin{align*}
    \|U^\top (W-\hat{W})\| & = \|D^{-1/2}R - P\| \\
    & = \|D^{-1/2}R_2R_1^\top - D^{-1/2}R_2\Delta^{1/2}R_1^\top\| \\
    & = \|D^{-1/2}R_2(I - \Delta^{1/2})R_1^\top\| \\
    & \le \|D^{-1/2}\| \|I - \Delta^{1/2}\| \\
    & \le \frac{3\epsilon}{\sigma_{\min}(M)^{3/2}}.
\end{align*}
The last step uses the fact that $|I-\Delta^{1/2}\| \le \|I-\Delta\|$ for diagonal matrix $\Delta$, which just follows from $|1-\sqrt{x}| \le |1-x|$ for every $x \ge 0$. 
Combining this bound with \eqref{eqn:power-whitening-proof-aux1} and \eqref{eqn:power-whitening-proof-aux2}, we prove the first desired inequality as
\[\|W-\hat{W}\| \le \frac{5\epsilon}{\sigma_{\min}(M)^{3/2}}.\]

With the choice of $R$, the second inequality is easier to prove:
\begin{align*}
    \|B^\top (W-\hat{W})\| & = \|D^{1/2}U^\top (W-\hat{W})\| \\
    & = \|R - D^{1/2}P\| \\
    & = \| R_2R_1^\top  - R_2 \Delta^{1/2}R_1^\top\| \\
    & \le \|I - \Delta^{1/2}\| \\
    & \le \frac{3\epsilon}{\sigma_{\min}(M)}.
\end{align*}

To prove the third equation, we observe that $B = MW$ and $\hat{B} = \bar{M}\hat{W}$. Therefore,
\begin{align*}
    \|B-\hat{B}\| & = \|MW - \bar{M}\hat{W}\| \\
    & \le \|M(W-\hat{W})\| + \|(M-\bar{M})\hat{W}\|.
\end{align*}
Here, the second term is bounded as
\begin{align*}
\|(M-\bar{M})\hat{W}\| \leq \|M-\bar{M}\| \|\hat{W}\|
\le \frac{3 \epsilon}{\sigma_{\min}(M)^{1/2}},
\end{align*}
where in the second inequality we used~\eqref{eqn:Weyl-aux}, the fact that $\|M-\bar{M}\| \le 2\epsilon$ and an upper bound on the constant term.
For the first term, it can be bounded as
\begin{align*}
    \|M(W-\hat{W})\| & = \|UD^{1/2}R - UDP\| \\
    & = \|UD^{1/2} (R-D^{1/2}P)\| \\
    & \le \|UD^{1/2}\| \|R-D^{1/2}P\| \\
    & \le \frac{3\epsilon\sqrt{\|M\|}}{\sigma_{\min}(M)} .
\end{align*}
Note that the bound for $\|R-D^{1/2}P\|$ is the same as the second inequality.
\end{proof}

Finally, we are ready to prove Theorem~\ref{thm:robust_with_whiten}.

\begin{proof} [Proof of Theorem~\ref{thm:robust_with_whiten}]
We first construct the whitening matrices $W$ and $\hat{W}$, and un-whitening matrices $B$ and $\hat{B}$ for the exact matrix $M$ and observed matrix $\hat{M}$ as described in Lemma~\ref{lem:whitening_error}. The ideal tensor that we want to perform orthogonal tensor decomposition is $T(W,W,W)$, however we only have access to $\hat{T}(\hat{W},\hat{W},\hat{W})$. Therefore the main part of the proof is to bound the difference between these two tensors.

Let $v_i = \sqrt{\tl{\lambda}_i} W^\top a_i$. As we argued in Section~\ref{sec:whitening-ten}, $v_i$'s are orthonormal vectors and we have (see~\eqref{eqn:tensor-whitening-noiseless}~and~\eqref{eqn:mu})
\[T(W,W,W) = \sum_{i=1}^k \Lambda_i v_i^{\otimes 3},
\]
where we defined $\Lambda_i := \lambda_i\tl{\lambda}_i^{-3/2}$. Since $T(W,W,W)$ is an orthogonal tensor, its spectral norm is equal to $\Lambda_{\max}$. Let $T_W := T(W,W,W)$, $\hat{T}_W := \hat{T}(\hat{W},\hat{W},\hat{W})$, and $Q := B^\top (\hat{W}-W)$, then we have
\begin{align*}
    \hat{T}_W-T_W & = T_W(I+Q,I+Q,I+Q) - T_W + E_T(\hat{W},\hat{W},\hat{W}) \\
    & = T_W(Q, I, I) + T_W(I,Q,I)+T_W(I,I,Q) \\
    &\quad + T_W(Q,Q,I) + T_W(Q,I, Q)+T_W(I,Q,Q) \\
    &\quad + T_W(Q,Q,Q) +E_T(\hat{W},\hat{W},\hat{W}).
\end{align*}
By the second inequality of Lemma~\ref{lem:whitening_error}, we know $\|Q\| \le 3\epsilon_M/\sigma_{\min}(M) < 1$. Thus, the first 7 terms of above equation all have spectral norm bounded by $\|T_W\| \|Q\| \le 3\Lambda_{\max}\epsilon_M/\sigma_{\min}(M)$.

The last term $E_T(\hat{W},\hat{W},\hat{W})$ has norm bounded by $\|E_T\| \|\hat{W}\|^3$. From~\eqref{eqn:Weyl-aux}, we know $\|\hat{W}\| \le 2\sigma_{\min}(M)^{-1/2}$. Combining these bounds, we can say that for a large enough constant $C$, we have
\[
\|\hat{T}_W-T_W\| \le \epsilon_{T_W} := C\left(\frac{\epsilon_T}{\sigma_{\min}(M)^{3/2}} + \Lambda_{\max}\frac{\epsilon_M}{\sigma_{\min}(M)}\right).
\]

By Theorem~\ref{thm:robustpower}, Algorithm~\ref{alg:robustpower} will return a set of pairs $\{(\hat{v}_j, \hat{\Lambda}_j): j \in [k] \}$, where with probability at least $1 - \eta$, there exists a permutation $\pi$ on $[k]$ such that
\[
\|v_{\pi(j)} - \hat{v}_j\| \le 8\epsilon_{T_W}/\Lambda_{\pi(j)}, \quad |\Lambda_{\pi(j)} - \hat{\Lambda}_j| \le 5\epsilon_{T_W}, \quad
j \in [k].
\]

The final outputs of the algorithm are $\hat{a}_j = \hat{B}\hat{v}_j$ and $\hat{\Lambda}_j$, for $j \in [k]$. The estimation guarantees of the eigenvalues $\Lambda_j$'s are already concluded above. We only need to analyze the perturbation of the unwhitening procedure for $\hat{a}_j$'s.

Note that $\sqrt{\tl{\lambda}_i}a_i = Bv_i$. Therefore, to compare $\sqrt{\tl{\lambda}_{\pi(j)}}a_{\pi(j)}$ with $\hat{a}_j$, we only need to compare $Bv_{\pi(j)}$ with $\hat{B}\hat{v}_j$.
\begin{align*}
    \|\sqrt{\tl{\lambda}_{\pi(j)}}a_{\pi(j)} - \hat{a}_j\| &= \|Bv_{\pi(j)} - \hat{B}\hat{v}_j\| \\
    & \le \|B(v_{\pi(j)} - \hat{v}_j)\| + \|(B-\hat{B})\hat{v}_j\| \\
    & \le \|B\|\|(v_{\pi(j)} - \hat{v}_j)\| + \|B-\hat{B}\| \\
    & \le 8\epsilon_{T_W}\sqrt{\|M\|}/\Lambda_{\pi(j)} + 3\epsilon_M \sqrt{\|M\|}/\sigma_{\min}(M) \\
    & \le 9\epsilon_{T_W}\sqrt{\|M\|}/\Lambda_{\pi(j)},
\end{align*}
where we used the fact that $\|\hat{v}_j\| = 1$ in the second inequality. $\|B-\hat{B}\|$ is bounded by the result of Lemma~\ref{lem:whitening_error}. The final step is true because $\epsilon_{T_W} \ge C\Lambda_{\max}\epsilon_M/\sigma_{\min}(M)$, and therefore, $3\epsilon_M \sqrt{\|M\|}/\sigma_{\min}(M) \le \epsilon_{T_W}\sqrt{\|M\|}/\Lambda_{\pi(j)}$ for all $j$ as long as $C \ge 3$.
\end{proof}

\section{Simultaneous Diagonalization}
\label{sec:simdiag}

In this section, we describe simultaneous diagonalization algorithm which is one of the first algorithms with provable guarantees for tensor decomposition. It was discovered in \citet{harshman1970foundations} (and credited to Dr. Robert Jenrich), with generalizations in \citet{leurgans1993decomposition}. Simultaneous diagonalization method for tensor decomposition is provided in Algorithm~\ref{algo:simdiag}.

\floatname{algorithm}{Algorithm}
\begin{algorithm}[t]
\caption{Simultaneous Diagonalization for Tensor Decomposition}
\label{algo:simdiag}
\begin{algorithmic}[1]
\renewcommand{\algorithmicrequire}{\textbf{input}}
\renewcommand{\algorithmicensure}{\textbf{output}}
\REQUIRE tensor $T = \sum_{j \in [k]} \lambda_j \ a_j \otimes b_j \otimes c_j \in \Rbb^{d_1 \times d_2 \times d_3}$
\ENSURE rank-1 components of tensor $T$
\STATE Pick two random vectors $x,y\sim \N(0,I_{d_3})$. 
\STATE Compute matrices $$M_x := T(I, I, x) \in \R^{d_1 \times d_2}, \quad M_y := T(I,I,y) \in \R^{d_1 \times d_2};$$ see~\eqref{eqn:multilinear form def} for the definition of the multilinear form.
\STATE Let 
\begin{itemize}
\item $\{(\hat{\alpha}_j,\hat{a}_j)\}$'s be the eigenvalues \& eigenvectors of $M_xM_y^{\dagger}$.
\item $\{(\hat{\beta}_j,\hat{b}_j)\}$'s be the eigenvalues \& eigenvectors of $M_y^\top (M_x^{\dagger})^\top$.
\end{itemize}
Here $\cdot^\dagger$ denotes the pseudo-inverse matrix; see Definition~\ref{def:pseudoinverse}.
\STATE For any $j \in [k]$, pair $(\hat{a}_j,\hat{b}_j)$ if the corresponding eigenvalues satisfy $\hat{\alpha}_j \hat{\beta}_j = 1$. 
\STATE Fixing $(\hat{a}_j,\hat{b}_j), j \in [k]$, solve the linear system $T = \sum_{j=1}^k \hat{a}_j\otimes \hat{b}_j\otimes \hat{c}_j$ in terms of variables $\hat{c}_j$'s.
\STATE Set $\hat{\lambda}_j = \|\hat{c}_j\|$ and $\hat{c}_j = \hat{c}_j/\|\hat{c}_j\|$.
\RETURN $\{(\hat{\lambda}_j; \hat{a}_j,\hat{b}_j,\hat{c}_j): j \in [k]\}$
\end{algorithmic}
\end{algorithm}

Comparing to the power method, simultaneous diagonalization is much easier to analyze, does not require the whitening procedure and can work even when the third dimension $d_3$ is smaller than $k$. However, the straightforward implementation of simultaneous diagonalization is not very robust to perturbations. We give the guarantees for the simultaneous diagonalization algorithm in the noiseless setting as follows.

\begin{theorem}\label{thm:simdiag}[Simultaneous 
Diagonalization Guarantees in Noiseless Setting] Suppose tensor $T$ has a rank-$k$ decomposition
\begin{equation} \label{eqn:tenDecomp-simDiag}
T = \sum_{j \in [k]} \lambda_j \ a_j \otimes b_j \otimes c_j \in \Rbb^{d_1 \times d_2 \times d_3}.
\end{equation}
In addition, suppose vectors $\{a_j\}$'s and $\{b_j\}$'s are both linearly independent, and vectors $\{c_j\}$'s have Kruskal rank at least 2, i.e., no two $c_i$ and $c_j$ for $i\neq j$ are on the same line or parallel; see Definition~\ref{def:kruskal} for the definition of Kruskal rank. Then, with probability 1 (over the randomness of vectors $x$ and $y$ in the algorithm), Algorithm~\ref{algo:simdiag} returns a group of 4-tuples $(\hat{\lambda}_j;\hat{a}_j,\hat{b}_j,\hat{c}_j)$ such that 
$$T = \sum_{j \in [k]} \hat{\lambda}_j \ \hat{a}_j\otimes \hat{b}_j\otimes \hat{c}_j.$$ 
Furthermore, $(\hat{\lambda}_j; \hat{a}_j,\hat{b}_j,\hat{c}_j)$ is equivalent to $(\lambda_j; a_j,b_j,c_j)$ up to permutation and scaling.
\end{theorem}

In the rest of this section, we illustrate the ideas in different steps of the algorithm which clarifies how it decomposes rank-$k$ tensor $T$ in~\eqref{eqn:tenDecomp-simDiag}, and also provides an informal proof for the above theorem.

First, we describe the structure and properties of matrices $M_x,M_y$ (see Step 2 of the algorithm), which also clarifies why this algorithm is called simultaneous diagonalization. Following the above tensor decomposition structure for tensor $T$, and given the multilinear form as a linear combination of tensor slices through weight vectors $x$ and $y$, we have
\begin{align*}
M_x &= \sum_{j \in [k]} \lambda_j\inner{x,c_j} a_j b_j^\top = A D_x B^\top, \\
M_y &= \sum_{j \in [k]} \lambda_j\inner{y,c_j} a_j b_j^\top = A D_y B^\top.
\end{align*} 
Here $A \in \R^{d_1 \times k}$, $B \in \R^{d_2 \times k}$ are matrices whose columns are $\{a_j\}$'s and $\{b_j\}$'s, respectively. We also define $D_x \in \R^{k \times k}$ as a diagonal matrix whose $(j,j)$-the entry denoted by $d_{x,j}$ is equal to $\lambda_j\inner{x,c_j}$, and similarly, $D_y \in \R^{k \times k}$ as a diagonal matrix whose $(j,j)$-the entry denoted by $d_{y,j}$ is equal to $\lambda_j\inner{y,c_j}$. These equations are called {\em diagonalizations} of $M_x, M_y$, and they share the {\em same} matrices $A,B$. That is why this algorithm is called {\em simultaneous diagonalization}.

With the above forms of $M_x$ and $M_y$, we can compute the two matrices used in Step 3 as
\begin{align*}
M_xM_y^\dagger &= A D_x D_y^{-1} A^\dagger, \\
M_y^\top (M_x^\dagger)^\top &= B D_y D_x^{-1} B^\dagger.
\end{align*}
Given this, for any $a_j$ in~\eqref{eqn:tenDecomp-simDiag}, we have
$$
M_xM_y^\dagger a_j = A D_x D_y^{-1} A^\dagger a_j = A D_x D_y^{-1} e_j^{(k)} = d_{x,j}d_{y,j}^{-1} a_j.
$$
Here in the second equality, we used the fact that $A^\dagger A = I_k$ (with $I_k$ denoting the $k$-dimensional identity matrix), and hence, $A^\dagger a_j = e_j^{(k)}$, where $e_j^{(k)}$ denotes the $j$-th basis vector in the $k$-dimensional space, i.e, the $j$-th column of $I_k$.
Similarly for any $b_j$ in~\eqref{eqn:tenDecomp-simDiag}, we have 
$$M_y^\top (M_x^\dagger)^\top b_j = d_{x,j}^{-1}d_{y,j} b_j.$$
Therefore, by the definition of matrix eigenvectors, $\{a_j\}$'s and $\{b_j\}$'s are exactly the eigenvectors of the two matrices, and the corresponding eigenvalues $\alpha_j:=d_{x,j}d_{y,j}^{-1}$ and $\beta_j:=d_{x,j}^{-1}d_{y,j}$ satisfy 
$$\alpha_j \beta_j =1, \quad j \in [k].$$
Thus, as long as the values $\alpha_j$'s (and inherently $\beta_j$'s) are unique for $j\in[k]$, Steps 3 and 4 of Algorithm~\ref{algo:simdiag} correctly find the set of $\{(a_j,b_j)\}$'s up to permutation. The values of $\alpha_j$'s (and inherently $\beta_j$'s) rely on the randomness of vectors $x$ and $y$, and when $\{c_j\}$'s have Kruskal rank at least two, the values are distinct with probability 1.

Finally in Step 5, the algorithm recovers the components of the last mode of the tensor  by fixing $\{a_j\}$'s, $\{b_j\}$'s, and solving a system of linear equations in terms of $\{c_j\}$'s. Note that the same idea appears in Alternating Least Squares algorithm that we will discuss in the next section with more details on how to efficiently solve such a system of linear equations. When $\{a_j\}$'s and $\{b_j\}$'s are linearly independent, this step will have a unique solution. Hence, the algorithm finds the unique decomposition of tensor $T$.

Note that simultaneous diagonalization algorithm corresponds to a special tight case of Kruskal's condition in Theorem~\ref{thm:kruskal}, where $\krank(A) = k$, $\krank(B) = k$, $\krank(C) \ge 2$. Weakening the assumption on any of the matrices $A$, $B$ or $C$ may make the tensor decomposition non-unique. Compared to tensor power method (with the symmetrization procedure in Section~\ref{subsec:nonsymmetric}), the simultaneous diagonalization method does not need access to second moment matrices, and can allow one of the factors to have krank 2.

\section{Alternating Least Squares}
\label{sec:als}
One of the most popular algorithms for  tensor decomposition is {\em Alternating Least Squares (ALS)} method, which has been described as the ``workhorse'' of tensor decomposition~\citep{kolda_survey}. This involves solving the least squares problem on a {\em mode} of the tensor, while keeping the other modes fixed, and alternating between the tensor modes. This becomes clearer as we describe the details of ALS as follows.

Given rank-$k$ tensor
$$
T = \sum_{j \in [k]} \lambda_j \ a_j \otimes b_j \otimes c_j \in \R^{d_1 \times d_2 \times d_3},
$$
the goal is to recover tensor rank-1 components $a_j \in \R^{d_1}, b_j \in \R^{d_2}, c_j \in \R^{d_3}$, and coefficients $\lambda_j \in \R$, for  $j \in [k]$. As before, we assume the rank-1 components have unit norm. The problem can be formulated as the {\em least squares} optimization 
\begin{align} \label{eqn:ALSopt}
\left( \lambda_j; a_j, b_j, c_j \right)_{j \in [k]} := & \argmin_{\tl{\lambda}_j, \tl{a}_j, \tl{b}_j, \tl{c}_j} \Bigl\| T- \sum_{j \in [k]} \tl{\lambda}_j \ \tl{a}_j \otimes \tl{b}_j \otimes \tl{c}_j \Bigr\|_F, \\
& \text{s.t.} \ \tl{\lambda}_j \in \R, \tl{a}_j \in \R^{d_1}, \tl{b}_j \in \R^{d_2}, \tl{c}_j \in \R^{d_3}, \nn
\end{align}
where the error between tensor $T$ and its rank-$k$ estimation is minimized in the sense of Frobenius norm. This is a multilinear optimization program and  a {\em non-convex} optimization problem. 

Alternating Least Squares method provides an approach to overcome the non-convexity challenge. This mainly involves modifying the optimization problem such that the optimization is performed for only one of the components while all other components are assumed to be fixed. The same step is performed by {\em alternating} among different components (modes). Thus, the problem in~\eqref{eqn:ALSopt} is solved through an alternating least squares approach.  

We now describe the main step of ALS by fixing the second and third mode rank-1 components, i.e., matrices $\tl{B} := [\tl{b}_1 | \tl{b}_2 | \dotsb | \tl{b}_k]$ and $\tl{C}  := [\tl{c}_1 | \tl{c}_2 | \dotsb | \tl{c}_k]$, and optimizing over first mode, i.e., matrix $\tl{A}  := [\tl{a}_1 | \tl{a}_2 | \dotsb | \tl{a}_k]$.
As the first step, since the Frobenius norm in~\eqref{eqn:ALSopt} is an entry-wise tensor norm, we can reshape the tensor inside without chainging the norm. More specifically, we can rewrite the optimization problem in~\eqref{eqn:ALSopt} into the following equivalent form:
\begin{equation} \label{eqn:ALSoptMod}
\min_{\tl{A} \in \R^{d_1 \times k}} \bigl\| \mat(T,1) - \tl{A} \cdot \diag(\tl{\lambda}) \cdot (\tl{B} \odot \tl{C})^\top \bigr\|_F,
\end{equation}
where $\mat(T,1)$ denotes the mode-1 matricization of $T$, and $\odot$ denotes the Khatri-Rao product; see~\eqref{eqn:matricization} and \eqref{eqn:KhatriRao}, respectively. Note that here we assume $\tl{B}$, $\tl{C}$, and $\tl{\lambda}$ are fixed. We also used the following matricization property such that for vectors $u,v,w$, we have
\begin{equation*}
\mat(u \otimes v \otimes w,1) = u \cdot (v \odot w)^\top. 
\end{equation*}

The optimization problem in~\eqref{eqn:ALSoptMod} is now a {\em linear} least squares problem, and the analysis is very similar to linear regression with the additional property that matrix $(\tl{B} \odot \tl{C})^\top$ is highly-structured which is crucial for the computational efficiency of ALS as we see below. The (right) pseudo-inverse of $(\tl{B} \odot \tl{C})^\top$ is (see~\eqref{eqn:pseudo-inverse} for the definition)
\begin{equation*}
\left[(\tl{B} \odot \tl{C})^\top\right]^\dagger = (\tl{B} \odot \tl{C}) \cdot \left[(\tl{B} \odot \tl{C})^\top (\tl{B} \odot \tl{C})\right]^{-1}.
\end{equation*}
Computing the inverse matrix could be the computationally-expensive part of the iterations, but the specific Khatri-Rao structure of the matrix enables us to write it as 
\begin{equation*}
(\tl{B} \odot \tl{C})^\top (\tl{B} \odot \tl{C}) =  \tl{B}^\top \tl{B} * \tl{C}^\top \tl{C},
\end{equation*}
where $*$ denotes the Hadamard (entry-wise) product.
Thus, we only need to compute the inverse of $k \times k$ matrix $\tl{B}^\top \tl{B} * \tl{C}^\top \tl{C}$, and when $k$ is small (compared to $d_t, t \in \{1,2,3\}$), the inverse can be computed much faster. In practice we can also compute the Hadamard product and then solve a linear system of equations rather than explicitly computing the inverse.

Finally, using the above property and imposing the unit norm constraint on the columns of $\tl{A}$, we update the rank-1 components of the first mode as
\begin{equation} \label{eqn:ALS-mode1}
\tl{A} \mapsto \Norm \left( \mat(T,1) \cdot (\tl{B} \odot \tl{C}) \cdot  \left( \tl{B}^\top \tl{B} * \tl{C}^\top \tl{C} \right)^{-1} \right),
\end{equation}
where operator $\Norm(\cdot)$ normalizes the columns of input matrix, i.e., for vector $v$, we have $\Norm(v) := v / \|v\|$.

\floatname{algorithm}{Algorithm}
\begin{algorithm}[t]
\caption{Alternating Least Squares for Tensor Decomposition}
\label{algo:als}
\begin{algorithmic}[1]
\renewcommand{\algorithmicrequire}{\textbf{input}}
\renewcommand{\algorithmicensure}{\textbf{output}}
\REQUIRE tensor $T = \sum_{j \in [k]} \lambda_j \ a_j \otimes b_j \otimes c_j \in \Rbb^{d_1 \times d_2 \times d_3}$
\ENSURE rank-1 components of tensor $T$
\STATE Set initial estimates for $\tl{A} \in \R^{d_1 \times k},\tl{B} \in \R^{d_2 \times k},\tl{C} \in \R^{d_3 \times k}$.
\WHILE{converged} 
\STATE Let $\overline{A} = \mat(T,1) \cdot (\tl{B} \odot \tl{C}) \cdot  \left( \tl{B}^\top \tl{B} * \tl{C}^\top \tl{C} \right)^{-1}$, and set
$$
\tl{\lambda}_j = \|\overline{A}_j\|, \quad \tl{A}_j = \overline{A}_j/\tilde{\lambda}_j, \quad j \in [k].
$$
\STATE Let $\overline{B} = \mat(T,2) \cdot (\tl{A} \odot \tl{C}) \cdot  \left( \tl{A}^\top \tl{A} * \tl{C}^\top \tl{C} \right)^{-1}$, and set
$$
\tl{\lambda}_j = \|\overline{B}_j\|, \quad \tl{B}_j = \overline{B}_j/\tilde{\lambda}_j, \quad j \in [k].
$$
\STATE Let $\overline{C} = \mat(T,3) \cdot (\tl{A} \odot \tl{B}) \cdot  \left( \tl{A}^\top \tl{A} * \tl{B}^\top \tl{B} \right)^{-1}$, and set
$$
\tl{\lambda}_j = \|\overline{C}_j\|, \quad \tl{C}_j = \overline{C}_j/\tilde{\lambda}_j, \quad j \in [k].
$$
\ENDWHILE
\RETURN $(\tilde{\lambda};\tilde{A},\tilde{B},\tilde{C})$ as the estimation of tensor decomposition components.
\end{algorithmic}
\end{algorithm}

By alternating between different modes of the tensor and with similar calculations, we update the second and third modes as
\begin{align*}
\tl{B} & \mapsto \Norm \left( \mat(T,2) \cdot (\tl{A} \odot \tl{C}) \cdot  \left( \tl{A}^\top \tl{A} * \tl{C}^\top \tl{C} \right)^{-1} \right), \\
\tl{C} & \mapsto \Norm \left( \mat(T,3) \cdot (\tl{A} \odot \tl{B}) \cdot  \left( \tl{A}^\top \tl{A} * \tl{B}^\top \tl{B} \right)^{-1} \right).
\end{align*}
For the coefficient vector $\lambda$, we update it appropriately such that the rank-1 components have unit norm. We have summarized the ALS steps in Algorithm~\ref{algo:als}.

\paragraph{ALS vs.\ power iteration:}The ALS updates in the rank-1 form are strongly related to the power iteration updates. Recall tensor power iteration in~\eqref{eq:map2} which can be adapted to the asymmetric setting such that the update corresponding to the first component is (ignoring the normalization)
$$
\tl{a} \mapsto T(I,\tl{b},\tl{c}).
$$
The update in the right hand side can be also rewritten as
$$
T(I,\tl{b},\tl{c}) = \mat(T,1) \cdot (\tl{b} \odot \tl{c}) \propto \mat(T,1) \cdot \left( (\tl{b} \odot \tl{c})^\top\right)^\dagger,
$$
which is basically the rank-1 form of ALS updates that we described in this section. In rank-$k$ ALS, all components are simultaneously updated at each iteration, while the rank-1 version only updates one component at a time which is basically the ALS update that we described in this section, but only for one of the components. The process then needs to be repeated for each remaining component on the deflated tensor; see Algorithm~\ref{alg:robustpower}. By contrast, in the ALS algorithm we introduced here, all components are simultaneously updated at each iteration. Also note that if the tensor does not have an orthogonal decomposition, ALS can still work while tensor power iteration requires additional whitening step (as in Section~\ref{sec:whitening-ten}) even if the components are linearly independent. However, the benefit of tensor power iteration is that we do have guarantees for it (see Section~\ref{sec:power}, while ALS is not known to converge from a random starting point even if the tensor has an orthogonal decomposition.

\paragraph{Regularized ALS:}Since ALS involves solving linear least squares problems, we can also propose the regularized version of ALS. It is derived by adding a regularization term to the optimization in~\eqref{eqn:ALSoptMod}. The most popular form of regularization is the $\ell_2$-regularization which adds a term $\alpha \|\tl{A}\|_F^2$ to the optimization problem, where $\alpha \geq 0$ is the regularization parameter. This leads to the ALS updates being changed as
\begin{equation} \label{eqn:ALS-model-l2}
\tl{A} \mapsto \Norm \left( \mat(T,1) \cdot (\tl{B} \odot \tl{C}) \cdot  \left( \tl{B}^\top \tl{B} * \tl{C}^\top \tl{C}  + \alpha I \right)^{-1} \right);
\end{equation}
and similarity the updates for $\tl{B}$ and $\tl{C}$ are changed.
This is specifically helpful when the non-regularized pseudo-inverse matrix is not well-behaved. We can obviously add other forms of regularization terms to the optimization problem which lead to variants of regularized ALS.

\paragraph{ALS for symmetric tensors:}The ALS algorithm is naturally proposed for asymmetric tensor decomposition, where at each iteration only one component is updated while all other components are fixed. The next natural question is whether ALS can be adapted to the decomposition of symmetric tensors such as $T = \sum_{j \in [k]} \lambda_j \ a_j^{\otimes 3}$. Here we have to only estimate one matrix $A$. We review two heuristics to do this. Let $\tl{A}_t$ denote the update variable in the left hand side of~\eqref{eqn:ALS-mode1} at iteration $t$. The first heuristic consists in, at iteration $t$, substituting $\tl{B}$ and $\tl{C}$ in the right hand side of~\eqref{eqn:ALS-mode1} by $\tl{A}_{t-1}$ and $\tl{A}_{t-2}$, respectively. The second heuristic consists in substituting both with $\tl{A}_{t-1}$.

\clearpage{}

\clearpage{}\chapter{Applications of Tensor Methods} \label{sec:applications}
\label{ch:applications}
In Chapter~\ref{sec:intro}, we gave a few examples of latent variable models that can be learned by tensor decomposition techniques. In this chapter, we elaborate on this connection and give more examples on how to learn many probabilistic models by tensor decomposition. We cover both unsupervised and supervised settings in this chapter. 
We hope these examples provide a good understanding of how tensor methods are applied in the existing literature and can help in generalizing tensor decomposition techniques to  learning more models.
Of course, there are still many more applications of tensor decomposition techniques to learn probabilistic models in the literature, and we give a brief survey in Section~\ref{sec:othermodels}.

In the unsupervised setting, we discuss models including the Gaussian mixtures, multiview mixture model, Independent Component Analysis (ICA), Latent Dirichlet Allocation (LDA) and Noisy-Or models. To this end, the observed moment is formed as a low order tensor (usually third or fourth order), and by decomposing the tensor to its rank-1 components we are able to learn the parameters of the model; see Sections~\ref{sec:singletopic}-\ref{sec:noisy-or-network} which describe this connection.
The basic form is demonstrated in Theorem~\ref{thm:single-topic} for the
first example, and the general pattern will emerge from subsequent
examples.

Then in Section~\ref{sec:moment-supervised}, we show how the tensor techniques can be adapted to supervised setting, and in particular, for learning neural networks and mixtures of generalized linear models. Here, we exploit the cross-moment between the output and a specific non-linear transformation of the input. By decomposing that cross-moment into rank-1 components, we learn the parameters of the model.

\section{Pure Topic Model Revisited}\label{sec:singletopic}

We start by explaining the pure topic model in more details, where it was originally introduced in Section~\ref{sec:intro-example}. 
Recall the model is a simple bag-of-words model for documents in which the
words in the document are assumed to be \emph{exchangeable}\--- a collection of random variables $x_1, x_2, \dotsc, x_\ell$ are
exchangeable if their joint probability distribution is invariant to
permutation of the indices.
The well-known De Finetti's theorem~\citep{austin2008exchangeable} implies
that such exchangeable models can be viewed as mixture models in which
there is a latent variable $h$ such that $x_1, x_2, \dotsc, x_\ell$ are
conditionally i.i.d.~given $h$ (see Figure~\ref{fig:exchangeable} for the
corresponding graphical model) and the conditional distributions are
identical at all the nodes, i.e., for all $x$'s.

In our simplified topic model for documents, the latent variable $h$ is
interpreted as the (sole) topic of a given document, and it is assumed to
take only a finite number of distinct values.
Let $k$ be the number of distinct topics in the corpus, $d$ be the number
of distinct words in the vocabulary, and $\ell \geq 3$ be the number of
words in each document.
The generative process for a document is as follows: the document's
topic is drawn according to the discrete distribution specified by the
probability vector $w := (w_1,w_2,\dotsc,w_k) \in \Delta^{k-1}$, where $\Delta^{k-1} :=
\{ v \in \R^k : \forall j \in [k], v_j \in [0,1], \ \sum_{j \in [k]} v_j = 1 \}$ denotes the probability simplex, i.e., the hidden topic $h$ is modeled as a discrete random variable $h$ such that
\[ \Pr[h = j] = w_j , \quad j \in [k] . \]
Given the topic $h$, the document's $\ell$ words are drawn independently
according to the discrete distribution specified by the probability
vector $\mu_h \in \Delta^{d-1}$.
It will be convenient to represent the $\ell$ words in the document by
$d$-dimensional random \emph{vectors} $x_1, x_2, \dotsc, x_\ell \in \R^d$.
Specifically, we set
\begin{align*}
    x_t = e_i \quad \text{if and only if} \quad
& \text{the $t$-th word in the document is $i$}, \\
& t \in [\ell] , i \in [d],
\end{align*}
where $\{e_1,e_2,\ldots, e_d\}$ is the standard coordinate basis for $\R^d$. This is basically equivalent to {\em one-hot encoding of words} using standard basis vectors in the $d$-dimensional space.

As we did in Section~\ref{sec:intro-example}, we will consider the {\em cross} moments of these vectors which means we will compute $\E[x_1x_2^\top]$ instead of $\E[x_1x_1^\top]$. The advantage of the above encoding of words and the choice of moments is that the moments will correspond to joint probabilities over words. For instance, observe that
\begin{align*}
\E[ x_1 \otimes x_2 ]
& = \sum_{i,j \in [d]} \Pr[ x_1 = e_i, x_2 = e_j ] \ e_i \otimes e_j
\\
& = \sum_{i,j \in [d]} \Pr[ \text{$1$st word} = i, \text{$2$nd word} =
j ] \ e_i \otimes e_j ,
\end{align*}
and thus, the $(i,j)$-the entry of the moment matrix $\E[ x_1 \otimes x_2 ]$ is $\Pr[
\text{$1$st word} = i, \text{$2$nd word} = j ]$.
More generally, the $(i_1,i_2,\dotsc,i_\ell)$-th entry in the tensor $\E[
x_1 \otimes x_2 \otimes \dotsb \otimes x_\ell ]$ is $\Pr[ \text{$1$st word}
= i_1, \text{$2$nd word} = i_2, \dotsc, \text{$\ell$-th word} = i_\ell ]$.
This means that estimating cross moments, say, of $x_1 \otimes x_2 \otimes
x_3$, is the same as estimating joint probabilities of the first three
words over all documents; recall that we assume that each document has at least three words.

The second advantage of the vector encoding of words is that the conditional
expectation of $x_t$ given $h = j$ is simply $\mu_j$, the vector of word
probabilities for topic $j$. This can be shown as
\begin{align*}
\E[ x_t | h = j ]
&= \sum_{i \in [d]} \Pr[ \text{$t$-th word} = i | h = j] \ e_i \\
&= \sum_{i \in [d]} [ \mu_j ]_i \ e_i
= \mu_j,
\quad j \in [k],
\end{align*}
where $[\mu_j]_i$ is the $i$-th entry of the vector $\mu_j$.
Because the words are conditionally independent given the topic, we can use
this same property with conditional cross moments, say, of $x_1$ and $x_2$:
\[
\E[ x_1 \otimes x_2 | h = j ]
= \E[ x_1 | h = j] \otimes \E[ x_2 | h = j]
= \mu_j \otimes \mu_j,
\quad j \in [k] .
\]

Now using the law of total expectations, we know

$$
\E[ x_1 \otimes x_2]
=\sum_{j=1}^k \Pr[h=j] \E[ x_1 \otimes x_2 | h = j ] = \sum_{j=1}^k w_j \ \mu_j\otimes \mu_j.
$$

This and similar calculations lead to the following theorem.

\begin{theorem}[\citealp{AHK12}] \label{thm:single-topic}
For the above exchangeable single topic model, if
\begin{align*}
M_2 & := \E[ x_1 \otimes x_2 ], \\
M_3 & := \E[ x_1 \otimes x_2 \otimes x_3 ],
\end{align*}
then
\begin{align*}
M_2 & = \sum_{j \in [k]} w_j \ \mu_j \otimes \mu_j, \\
M_3 & = \sum_{j \in [k]} w_j \ \mu_j \otimes \mu_j \otimes \mu_j.
\end{align*}
\end{theorem}

The structure of $M_2$ and
$M_3$ revealed in Theorem~\ref{thm:single-topic} implies that the topic
vectors $\mu_1, \mu_2, \dotsc, \mu_k$ can be estimated by computing a
certain symmetric tensor decomposition.
Moreover, due to exchangeability, any triples (resp., pairs) of words in a
document---and not just the first three (resp., two) words---can be used in
forming $M_3$ (resp., $M_2$).

\section{Beyond Raw Moments}
\label{sec:beyond}

In the above exchangeable single topic model, the {\em raw} (cross) moments of the observed
words directly yield the desired symmetric tensor structure.
In some other models, the raw moments do not explicitly have this form.
In this section, we show that the desired tensor structure can be found through
various manipulations of different moments for some other latent variable models.

\subsection{Spherical Gaussian mixtures} \label{sec:spherical-Gaussian}

We now consider a mixture of $k$ Gaussian distributions with spherical covariances.
We start with the simpler case where all of the covariances are identical;
this probabilistic model is closely related to the (non-probabilistic)
$k$-means clustering problem~\citep{kmeans}.
We then consider the case where the spherical variances may differ.

\paragraph{Common covariance.}

Let $w_j \in (0,1)$ be the probability of choosing component $j \in [k]$,
$\mu_1, \mu_2, \dotsc, \mu_k \in \R^d$ be the component mean
vectors, and $\sigma^2 I_d \in \R^{d \times d}$ be the common covariance matrix ($\sigma \in \R$) for the spherical Gaussian mixtures model.  Then an observation vector $x$ in this model is given by
\begin{align*}
x & := \mu_h + z ,
\end{align*}
where $h$ is the discrete random variable with $\Pr[h = j] = w_j$ for $j
\in [k]$ (similar to the exchangeable single topic model), and $z \sim
\N(0,\sigma^2 I_d)$ is an independent multivariate Gaussian random vector in
$\R^d$ with zero mean and spherical covariance matrix $\sigma^2 I_d$.

The Gaussian mixtures model differs from the exchangeable single topic model
in the way observations are generated.
In the single topic model, we observe multiple draws (words in a particular
document) $x_1, x_2, \dotsc, x_\ell$ given the same fixed $h$ (the topic of
the document).
In contrast, for the Gaussian mixtures model, every realization of $x$
corresponds to a different realization of $h$. The following theorem shows that how we can get the desired tensor decomposition form by modifying the raw moments.

\begin{theorem}[\citealp{HK13-mog}] \label{thm:spherical_same}
Assume $d \geq k$.
The variance $\sigma^2$ is the smallest eigenvalue of the
covariance matrix $\E[ x \otimes x ] - \E[x]\otimes \E[x]$.
Furthermore, if
\begin{align*}
M_2 & := \E[ x \otimes x ] - \sigma^2 I_d, \\
M_3 & :=  \E[ x \otimes x \otimes x] \\
& \quad - \sigma^2 \sum_{i \in [d]} \bigl(
\E[x] \otimes e_i \otimes e_i + e_i \otimes \E[x] \otimes e_i
+ e_i \otimes e_i \otimes \E[x] \bigr) ,
\end{align*}
then,
\begin{align*}
M_2 & = \sum_{j \in [k]} w_j \ \mu_j \otimes \mu_j, \\
M_3 & = \sum_{j \in [k]} w_j \ \mu_j \otimes \mu_j \otimes \mu_j.
\end{align*}
\end{theorem}

\paragraph{Differing covariances.}
The general case is where each component may have a
\emph{different} spherical covariance.
An observation in this model is again $x = \mu_h + z$, but now $z \in \R^d$
is a random vector whose conditional distribution given $h = j$ for some
$j \in [k]$ is a multivariate Gaussian $\N(0,\sigma_j^2 I_d)$ with zero mean
and spherical covariance $\sigma_j^2 I_d$.

\begin{theorem}[\citealp{HK13-mog}] \label{thm:spherical}
Assume $d \geq k$.
The average variance $\bar\sigma^2 := \sum_{j \in [k]} w_j \sigma_j^2$ is
the smallest eigenvalue of the covariance matrix $\E[ x \otimes x ] -
\E[x]\otimes \E[x]$.
Let $v$ be any unit norm eigenvector corresponding to the eigenvalue
$\bar\sigma^2$.
If
\begin{align*}
M_1 & := \E[ \langle v, x - \E[x] \rangle^2 x], \\
M_2 & := \E[ x \otimes x ] - \bar\sigma^2 I_d, \\
M_3 & := \E[ x \otimes x \otimes x] \\
& \quad - \sum_{i \in [d]} \bigl( M_1 \otimes e_i \otimes e_i + e_i \otimes M_1 \otimes e_i + e_i \otimes e_i \otimes M_1 \bigr) ,
\end{align*}
where $\langle \cdot,\cdot \rangle$ denotes the inner-product operator.
Then
\begin{align*}
M_2 = & \sum_{j \in [k]} w_j \ \mu_j \otimes \mu_j, \\
M_3 = & \sum_{j \in [k]} w_j \ \mu_j \otimes \mu_j \otimes \mu_j.
\end{align*}
\end{theorem}
As shown by~\citet{HK13-mog}, $M_1 = \sum_{j \in [k]} w_j \sigma_j^2 \mu_j$.
Note that for the common covariance case, where $\sigma_j^2 = \sigma^2$, we
have that $M_1 = \sigma^2 \E[x]$; see~Theorem~\ref{thm:spherical_same}.

\subsection{Independent component analysis (ICA)}

The standard model for ICA~\citep{Comon94,CC96,HO00,Comon:book}, in which
independent signals are linearly mixed and corrupted with Gaussian noise
before being observed, is specified as follows.
Let $h \in \R^k$ be a latent random \emph{vector} with independent
coordinates, $A \in \R^{d \times k}$ the mixing matrix, and $z \in \R^d$ be a
multivariate Gaussian random vector. The random vectors $h$ and $z$ are
assumed to be independent.
The observed random vector $x$ in this model is given by
\begin{align*}
x & := A h + z .
\end{align*}
Let $\mu_j$ denote the $j$-th column of the mixing matrix $A$.

\begin{theorem}[\citealp{Comon:book}]
Define
\begin{equation*}
M_4 := \E[ x \otimes x \otimes x \otimes x ] - T,
\end{equation*}
where $T \in \R^{d \times d \times d \times d}$ is the fourth-order tensor with
\begin{multline*}
[T]_{i_1,i_2,i_3,i_4} := \E[ x_{i_1} x_{i_2} ] \E[ x_{i_3} x_{i_4} ] + \E[
x_{i_1} x_{i_3} ] \E[ x_{i_2} x_{i_4} ] \\
+ \E[ x_{i_1} x_{i_4} ] \E[ x_{i_2}
x_{i_3} ] ,
\quad 1 \leq i_1, i_2, i_3, i_4 \leq d,
\end{multline*}
\emph{i.e.}, $T$ is the fourth derivative tensor of the function $v
\mapsto 8^{-1} \E[ (v^\t x)^2 ]^2$, and so, $M_4$ is the fourth cumulant tensor.
Let $\kappa_j := \E[h_j^4] - 3$ for each $j \in [k]$.
Then
\begin{equation*}
M_4 = \sum_{j \in [k]} \kappa_j \ \mu_j \otimes \mu_j \otimes \mu_j \otimes \mu_j.
\end{equation*}
\end{theorem}
Note that $\kappa_j$ corresponds to the excess kurtosis, a measure of
non-Gaussianity as $\kappa_j = 0$ if $h_j$ is a standard normal random
variable.
Hence, mixing matrix $A$ is not identifiable if $h$ is a multivariate
Gaussian.

We may derive forms similar to that of $M_2$ and $M_3$ in
Theorem~\ref{thm:single-topic} using $M_4$ by observing that
\begin{align*}
M_4(I,I,u,v)
& = \sum_{j \in [k]} \kappa_i (\mu_j^\t u) (\mu_j^\t v) \ \mu_j \otimes \mu_j, \\
M_4(I,I,I,v)
& = \sum_{j \in [k]} \kappa_j (\mu_j^\t v) \ \mu_j \otimes \mu_j \otimes \mu_j,
\end{align*}
for any vectors $u,v \in \R^d$.

\subsection{Latent Dirichlet Allocation}
\label{subsec:lda}
An increasingly popular class of latent variable models are \emph{mixed
membership models}, where each datum may belong to several different latent
classes simultaneously.
Latent Dirichlet Allocation (LDA, \cite{blei2003latent}) is one such model for the case of document modeling; here, each
document corresponds to a mixture over topics (as opposed to just a single
topic that we discussed in Section~\ref{sec:singletopic}).
The distribution over such topic mixtures is a Dirichlet distribution
$\Dir(\alpha)$ with parameter vector $\alpha \in \R_{++}^k$ with strictly
positive entries; its density function over the probability simplex $\Delta^{k-1} :=
\{ v \in \R^k : \forall j \in [k], v_j \in [0,1], \ \sum_{j \in [k]} v_j = 1 \}$
is given by
\[
p_\alpha(h) =
\frac{\Gamma(\alpha_0)}{\prod_{j \in [k]} \Gamma(\alpha_j)}
\prod_{j \in [k]} h_j^{\alpha_j-1}
, \quad h \in \Delta^{k-1},
\]
where
\begin{equation*}
\alpha_0 := \alpha_1 + \alpha_2 + \dotsb + \alpha_k,
\end{equation*}
and $\Gamma(\cdot)$ denotes the Gamma function.

As before, the $k$ topics are specified by probability vectors $\mu_1,
\mu_2, \dotsc, \mu_k \in \Delta^{d-1}$ for generating words.
To generate a document, we first draw the topic mixture $h =
(h_1,h_2,\dotsc,h_k) \sim \Dir(\alpha)$, and then conditioned on $h$, we
draw $\ell$ words $x_1,x_2,\dotsc,x_\ell$ independently from the discrete
distribution specified by the probability vector $\sum_{j \in [k]} h_j
\mu_j$, \emph{i.e.}, for each  word $x_t$, we independently sample a topic
$j$ according to the topic proportion vector $h$ and then sample $x_t$ according to $\mu_j$.
Again, we encode a word $x_t$ by setting $x_t = e_i$ if and only if the $t$-th word in
the document is $i$.

The parameter $\alpha_0$ (the sum of the ``pseudo-counts'') characterizes
the concentration of the distribution.
As $\alpha_0\rightarrow 0$, the distribution degenerates to a single
topic model, i.e., the limiting density has, with probability
$1$, exactly one entry of $h$ being $1$ and the rest are $0$.
At the other extreme, if $\alpha = (c,c,\dotsc,c)$ for some scalar $c > 0$,
then as $\alpha_0 = ck \to \infty$, the distribution of $h$ becomes peaked
around the uniform vector $(1/k,1/k,\dotsc,1/k)$, and furthermore, the
distribution behaves like a product distribution.
We are typically interested in the case where $\alpha_0$ is small
(\emph{e.g.}, a constant independent of $k$), whereupon $h$ typically has
only a few large entries.
This corresponds to the setting where the documents are mainly comprised of
just a few topics.

\begin{theorem}[\citealp{SpectralLDA}] \label{thm:lda}
Define
\begin{align*}
M_1 & := \E[x_1], \\
M_2 & := \E[x_1 \otimes x_2 ] - \frac{\alpha_0}{\alpha_0+1} M_1\otimes M_1, \\
M_3 & : = \E[ x_1 \otimes x_2 \otimes x_3 ] \\
& \quad -\frac{\alpha_0}{\alpha_0+2}
\Bigl(\E[x_1 \otimes x_2 \otimes M_1]
+ \E[x_1 \otimes M_1 \otimes x_2]  + \E[M_1 \otimes x_1 \otimes x_2]
\Bigr) \\
& \quad + \frac{2\alpha_0^2}{(\alpha_0+2)(\alpha_0+1)}
M_1 \otimes M_1 \otimes M_1.
\end{align*}
Then,
\begin{align*}
M_2 & = \sum_{j \in [k]} \frac{\alpha_j}{(\alpha_0+1)\alpha_0} \ \mu_j \otimes \mu_j, \\
M_3 & =
\sum_{j \in [k]} \frac{2\alpha_j}{(\alpha_0+2)(\alpha_0+1)\alpha_0} \ \mu_j
\otimes \mu_j \otimes \mu_j.
\end{align*}
\end{theorem}

Note that $\alpha_0$ needs to be known to form $M_2$ and $M_3$ from
the raw moments.
This, however, is a much weaker assumption than assuming that the entire distribution
of $h$ is known, \emph{i.e.}, knowledge of the whole parameter vector
$\alpha$.

\section{Multi-view Models}
\label{sec:multi}

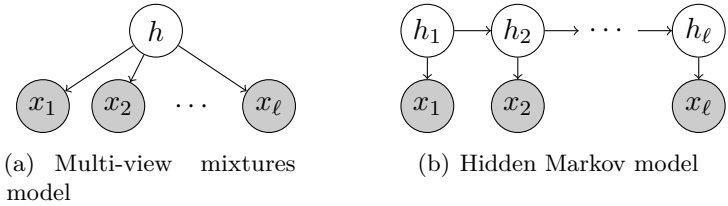
\begin{figure}
\begin{center}
\subfigure[Multi-view mixtures model]{\label{fig:exchangeable}
\begin{tikzpicture}
  [
    scale=1.0,
    observed/.style={circle,minimum size=0.7cm,inner sep=0mm,draw=black,fill=black!20},
    hidden/.style={circle,minimum size=0.7cm,inner sep=0mm,draw=black},
  ]
  \node [hidden,name=h] at ($(0,0)$) {$h$};
  \node [observed,name=x1] at ($(-1.5,-1)$) {$x_1$};
  \node [observed,name=x2] at ($(-0.5,-1)$) {$x_2$};
  \node at ($(0.5,-1)$) {$\dotsb$};
  \node [observed,name=xl] at ($(1.5,-1)$) {$x_\ell$};
  \draw [->] (h) to (x1);
  \draw [->] (h) to (x2);
  \draw [->] (h) to (xl);
\end{tikzpicture}}
\hfil\subfigure[Hidden Markov model]{\label{fig:hmm}
\begin{tikzpicture}
  [
    scale=1.0,
    observed/.style={circle,minimum size=0.7cm,inner sep=0mm,draw=black,fill=black!20},
    hidden/.style={circle,minimum size=0.7cm,inner sep=0mm,draw=black},
  ]
  \node [hidden,name=h1] at ($(-1.2,0)$) {$h_1$};
  \node [hidden,name=h2] at ($(0,0)$) {$h_2$};
  \node [name=hd] at ($(1.2,0)$) {$\dotsb$};
  \node [hidden,name=hl] at ($(2.4,0)$) {$h_\ell$};
  \node [observed,name=x1] at ($(-1.2,-1)$) {$x_1$};
  \node [observed,name=x2] at ($(0,-1)$) {$x_2$};
  \node [observed,name=xl] at ($(2.4,-1)$) {$x_\ell$};
  \draw [->] (h1) to (h2);
  \draw [->] (h2) to (hd);
  \draw [->] (hd) to (hl);
  \draw [->] (h1) to (x1);
  \draw [->] (h2) to (x2);
  \draw [->] (hl) to (xl);
\end{tikzpicture}}
\end{center}
\caption{Examples of latent variable models
}
\label{fig:graphical-model}
\vspace{-1mm}
\end{figure}

Multi-view models (also sometimes called na\"ive Bayes models) are a
special class of Bayesian networks in which observed variables $x_1, x_2,
\ldots, x_\ell$ are conditionally independent given a latent variable $h$.
This is similar to the exchangeable single topic model, but here we do not
require the conditional distributions of the $x_t, t \in [\ell]$, to be
identical.
Techniques developed for this class can be used to handle a number of
widely used models including hidden Markov models (HMMs) \citep{MR06,AHK12},
phylogenetic tree models~\citep{Chang96,MR06}, certain tree mixtures~\citep{AnandkumarHsuKakade:graphmixturesNIPS12}, and certain probabilistic
grammar models~\citep{unmixing}.

As before, we let $h \in [k]$ be a discrete random variable with $\Pr[h =
j] = w_j$ for all $j \in [k]$.
Now consider random vectors $x_1 \in \R^{d_1}$, $x_2 \in \R^{d_2}$, and
$x_3 \in \R^{d_3}$ which are conditionally independent given $h$ (see Figure~\ref{fig:exchangeable} for the corresponding graphical  model), and
\begin{align*}
\E[x_t | h = j] & = \mu_{t,j}
, \quad j \in [k], \ t \in \{1,2,3\},
\end{align*}
where $\mu_{t,j} \in \R^{d_t}$ are the conditional means of $x_t$
given $h = j$.
Thus, we allow the observations $x_1, x_2, \dotsc, x_\ell$ to be random vectors,
parameterized only by their conditional means. Importantly, these conditional
distributions may be discrete, continuous, or even a mix of both.

We first note the form for the raw (cross) moments.
\begin{prop} We have
\begin{align*}
\E[ x_t \otimes x_{t'} ]
& = \sum_{j \in [k]} w_j \ \mu_{t,j} \otimes \mu_{t',j} ,
\quad \{t,t'\} \subset \{1,2,3\} , t \neq t', \\
\E[ x_1 \otimes x_2 \otimes x_3]
& = \sum_{j \in [k]} w_j \ \mu_{1,j} \otimes \mu_{2,j} \otimes \mu_{3,j} .
\end{align*}
\end{prop}

The cross moments do not possess a symmetric tensor form when the
conditional distributions are different. We can either apply asymmetric tensor decomposition techniques to estimate conditional mean vectors $\mu_{t,j}$, or symmetrize the tensors by the following trick and then apply symmetric tensor decomposition techniques.
Nevertheless, the moments can be ``symmetrized'' via a simple linear
transformation of $x_1$ and $x_2$ (roughly speaking, this relates $x_1$ and
$x_2$ to $x_3$); this leads to an expression from which the conditional
means of $x_3$ (\emph{i.e.}, $\mu_{3,1}, \mu_{3,2}, \dotsc, \mu_{3,k}$) can
be recovered.
For simplicity, we assume $d_1 = d_2 = d_3 = k$; the general case (with
$d_t \geq k$) is easily handled using low-rank singular value
decompositions.

\begin{theorem}[\citealp{SpectralLDA}] \label{thm:multiview}
Assume that $\{ \mu_{t,1}, \mu_{t,2}, \dotsc, \mu_{t,k} \}$ are
linearly independent for each $t \in \{1,2,3\}$.
Define
\begin{align*}
\tl x_1 & :=  \E[x_3 \otimes x_2] \E[x_1 \otimes x_2]^{-1}  x_1, \\
\tl x_2 & :=  \E[x_3 \otimes x_1] \E[x_2 \otimes x_1]^{-1}  x_2, \\
\end{align*}
and
\begin{align*}
M_2 & := \E[\tl x_1 \otimes \tl x_2], \\
M_3 & := \E[\tl x_1 \otimes \tl x_2 \otimes x_3] .
\end{align*}
Then,
\begin{align*}
M_2 & = \sum_{j \in [k]} w_j \ \mu_{3,j} \otimes \mu_{3,j}, \\
M_3 & = \sum_{j \in [k]} w_j \ \mu_{3,j} \otimes \mu_{3,j} \otimes \mu_{3,j} .
\end{align*}
\end{theorem}

We now discuss three examples mostly taken from \citet{AHK12} where the
above observations can be applied.
The first two concern mixtures of product distributions, and
the last one is the time-homogeneous hidden Markov model.

\subsection{Mixtures of axis-aligned Gaussians and other product
distributions}

The first example is a mixture of $k$ product distributions in $\R^d$ under
a mild incoherence assumption~\citep{AHK12}.
Here, we allow each of the $k$ component distributions to have a different
product distribution (\emph{e.g.}, Gaussian distribution with an axis-aligned
covariance matrix), but require the matrix of component means $A := [ \mu_1
| \mu_2 | \dotsb | \mu_k ] \in \R^{d \times k}$ to satisfy a certain (very
mild) incoherence condition.
The role of the incoherence condition is explained below.

For a mixture of product distributions, any partitioning of the dimensions
$[d]$ into three groups creates three (possibly asymmetric) ``views'' which
are conditionally independent once the mixture component is selected.
However, recall that Theorem~\ref{thm:multiview} requires that for each
view, the $k$ conditional means be linearly independent.
In general, this may not be achievable; consider, for instance, the case
$\mu_i = e_i$ for each $i \in [k]$.
Such cases, where the component means are very aligned with the coordinate
basis, are precluded by the incoherence condition.

Let $\Pi_A$ denote the orthogonal projector operator to the range of $A$ and  define $\coherence(A) := \max_{i \in [d]} \{ e_i^\t \Pi_A e_i \}$ to be the
largest diagonal entry of this operator, and
assume $A$ has rank $k$.
The coherence lies between $k/d$ and $1$; it is largest when the range of
$A$ is spanned by the coordinate axes, and it is $k/d$ when the range is
spanned by a subset of the Hadamard basis of cardinality $k$.
The incoherence condition requires that for some $\veps, \delta \in (0,1)$,
$\coherence(A) \leq (\veps^2/6)/\ln(3k/\delta)$.
Essentially, this condition ensures that the non-degeneracy of the
component means is not isolated in just a few of the $d$ dimensions.
Operationally, it implies the following.
\begin{prop}[\citealp{AHK12}] \label{prop:product}
Assume $A$ has rank $k$, and
\[ \coherence(A) \leq \frac{\veps^2/6}
{\ln(3k/\delta)} \]
for some $\veps,\delta \in (0,1)$.
With probability at least $1-\delta$, a random partitioning of the
dimensions $[d]$ into three groups (for each $i \in [d]$, independently
pick $t \in \{1,2,3\}$ uniformly at random and put $i$ in group $t$)
has the following property.
For each $t \in \{1,2,3\}$ and $j \in [k]$, let $\mu_{t,j}$ be the entries
of $\mu_j$ put into group $t$, and let $A_t := [ \mu_{t,1} | \mu_{t,2} |
\dotsb | \mu_{t,k} ]$.
Then for each $t \in \{1,2,3\}$, $A_t$ has full column rank, and the $k$-th
largest singular value of $A_t$ is at least $\sqrt{(1-\veps)/3}$ times that
of $A$.
\end{prop}
Therefore, three asymmetric views can be created by randomly partitioning
the observed random vector $x$ into $x_1$, $x_2$, and $x_3$, such that the
resulting component means for each view satisfy the conditions of
Theorem~\ref{thm:multiview}.

\subsection{Spherical Gaussian mixtures, revisited}

Consider again the case of spherical Gaussian mixtures described in~Section~\ref{sec:spherical-Gaussian}. The previous analysis in Theorems~\ref{thm:spherical_same}~and~\ref{thm:spherical} can be used when the observation dimension $d \geq k$, and
 the $k$ component means are linearly independent. We now show
that when the dimension is slightly larger, say greater than $3k$, a
different (and simpler) technique based on the multi-view structure
can be used to extract the relevant structure.

We again use a randomized reduction.
Specifically, we create three views by (i) applying a random rotation to
$x$, and then (ii) partitioning $x \in \R^d$ into three views $\tl{x}_1,
\tl{x}_2, \tl{x}_3 \in \R^{\tl{d}}$ for $\tl{d} := d/3$.
By the rotational invariance of the multivariate Gaussian distribution, the
distribution of $x$ after random rotation is still a mixture of spherical
Gaussians (\emph{i.e.}, a mixture of product distributions), and thus
$\tl{x}_1, \tl{x}_2, \tl{x}_3$ are conditionally independent given $h$.
What remains to be checked is that, for each view $t \in \{1,2,3\}$, the
matrix of conditional means of $\tl{x}_t$ for each view has full column
rank.
This is true with probability $1$ as long as the matrix of conditional
means $A := [ \mu_1 | \mu_2 | \dotsb | \mu_k ] \in \R^{d \times k}$  has
rank $k$ and $d \geq 3k$.
To see this, observe that a random rotation in $\R^d$ followed by a
restriction to $\tl{d}$ coordinates is simply a random projection from $\R^d$ to
$\R^{\tl{d}}$, and that a random projection of a linear subspace of dimension $k$
to $\R^{\tl{d}}$ is almost surely injective as long as $\tl{d} \geq k$.
Applying this observation to the range of $A$ implies the following.
\begin{prop}[\citealp{HK13-mog}]
Assume $A$ has rank $k$ and that $d \geq 3k$.
Let $R \in \R^{d \times d}$ be chosen uniformly at random among all
orthogonal $d \times d$ matrices, and set $\tl{x} := Rx \in \R^d$ and
$\tl{A} := RA = [ R\mu_1 | R\mu_2 | \dotsb | R\mu_k ] \in \R^{d \times k}$.
Partition $[d]$ into three groups of sizes $d_1, d_2, d_3$ with $d_t \geq
k$ for each $t \in \{1,2,3\}$.
Furthermore, for each $t$, define $\tl{x}_t \in \R^{d_t}$ (respectively,
$\tl{A}_t \in \R^{d_t \times k}$) to be the subvector of $\tl{x}$ (resp.,
submatrix of $\tl{A}$) obtained by selecting the $d_t$ entries (resp.,
rows) in the $t$-th group.
Then $\tl{x}_1, \tl{x}_2, \tl{x}_3$ are conditionally independent given
$h$; $\E[\tl{x}_t | h = j] = \tl{A}_t e_j$ for each $j \in [k]$ and $t \in
\{1,2,3\}$; and with probability $1$, the matrices $\tl{A}_1, \tl{A}_2,
\tl{A}_3$ have full column rank.
\end{prop}

It is possible to obtain a quantitative bound on the $k$-th largest
singular value of each $A_t$ in terms of the $k$-th largest singular value
of $A$ (analogous to Proposition~\ref{prop:product}).
One avenue is to show that a random rotation in fact causes $\tl{A}$ to
have low coherence, after which we can apply Proposition~\ref{prop:product}.
With this approach, it is sufficient to require $n = O(k \log k)$ (for
constant $\veps$ and $\delta$), which results in the $k$-th largest
singular value of each $A_t$ being a constant fraction of the $k$-th
largest singular value of $A$.
We conjecture that, in fact, $n \geq c \cdot k$ for some $c > 3$ suffices.

\subsection{Hidden Markov models}

Our next example is the time-homogeneous Hidden Markov models (HMM)\citep{baum1966statistical} for sequences of vector-valued
observations $x_1, x_2, \dotsc \in \R^d$.
Consider a Markov chain of discrete hidden states $y_1 \to y_2 \to y_3 \to
\dotsb$ over $k$ possible states $[k]$; given a state $y_t$ at time
$t$, the random observation $x_t \in \R^d$ at time $t$ is independent of all other observations and hidden states.
See Figure~\ref{fig:hmm}.

Let $\pi \in \Delta^{k-1}$ be the initial state distribution (\emph{i.e.},
the distribution of $y_1$), and $T \in \R^{k \times k}$ be the stochastic
transition matrix for the hidden state Markov chain such that for all times $t$,
\[ \Pr[y_{t+1} = i | y_t = j] = T_{i,j} , \quad i,j \in [k] . \]
Finally, let $O \in \R^{d \times k}$ be the matrix whose $j$-th column is
the conditional expectation of $x_t$ given $y_t = j$: for all times $t$,
\[ \E[ x_t | y_t = j] = O e_j , \quad j \in [k] . \]

\begin{prop}[\citealp{AHK12}]
Define $h := y_2$, where $y_2$ is the second hidden state in the Markov
chain.
Then
\begin{itemize}
\item $x_1, x_2, x_3$ are conditionally independent given $h$;
\item the distribution of $h$ is given by the vector $w := T\pi \in
\Delta^{k-1}$;
\item for all $j \in [k]$,
\begin{align*}
\E[x_1 | h = j] & = O \diag(\pi) T^\t \diag(w)^{-1} e_j \\
\E[x_2 | h = j] & = O e_j \\
\E[x_3 | h = j] & = OT e_j .
\end{align*}
\end{itemize}
\end{prop}

Note the matrix of conditional means of $x_t$ has full column rank, for each $t
\in \{1,2,3\}$, provided that: (i) $O$ has full column rank, (ii) $T$ is
invertible, and (iii) $\pi$ and $T\pi$ have positive entries. Using the result of this proposition, we can formulate the problem as a multi-view mixture model and apply Theorem~\ref{thm:multiview}.

\section{Nonlinear Model: Noisy-Or Networks} \label{sec:noisy-or-network}

The models we stated in the previous sections are all {\em linear} for the purpose of tensor decomposition; in particular, the observed moment tensors $T$ have an exact decomposition with the rank-1 components as the desired parameters to be learned. This behavior is fairly common if given hidden components, the conditional expectation of the observation is a linear combination of different components, e.g., in the Latent Dirichlet Allocation model, if the document has a mixture of topics, the probabilities of observing different words are also linear mixtures.

In more complicated models, the observation may not be linear. In this section, we consider the {\em noisy-or} model, which is among the first non-linear models that can be learned by tensor decomposition.

The noisy-or model is a Bayes network with binary latent variables $h\in \{0,1\}^k$, and binary observed variables $x\in \{0,1\}^d$. The hidden variables are independent Bernoulli variables with parameter $\rho$, i.e., $\Pr[h_j = 1]=\rho, j \in [k]$. The conditional distribution $\Pr[x \vert h]$ is parameterized by a non-negative weight matrix $W\in \R^{d\times k}$. Conditioned on $h$, the observations $x_1,\dots, x_d$ are independent  with distribution
\begin{align}
\label{eqn:noisyor1}\Pr\left[x_i = 0\mid h\right] = \prod_{j=1}^k \exp(-W_{ij}h_j) =\exp(-\inner{W^i, h}), \quad i \in [d],
\end{align}
where $W^{i}$ denotes the $i$-th row of $W$.
This model is often used to model the relationship between diseases and symptoms, as in the classical human-constructed tool for medical diagnosis called {\em Quick Medical  Reference} (QMR-DT) by \citet{shwe1991probabilistic}. In this case, the latent variables $h_j$'s are diseases and observed variables $x_i$'s are symptoms. We see that $1-\exp(-W_{ij}h_j)$ can be thought of as the probability that disease $h_j$ activates symptom $x_i$, and $x_i$ is activated if one of $h_j$'s activates it. This also explains the name of the model, noisy-or.  

Given~\eqref{eqn:noisyor1} and the independence of different $x_i$'s given $h$, we have
\begin{equation*}
\Pr[x \mid h] = \prod_{i=1}^d \left(1-\exp(-\langle W^i, h \rangle)\right)^{x_i} \left(\exp(-\langle W^i, h \rangle)\right)^{1-x_i}.
\end{equation*}
Contrasting with the linear models in the previous sections, we see that under this model when a patient has multiple diseases, the expectation of the symptoms $x$ is {\em not} a linear combination of different components.

\paragraph{Point-wise Mutual Information:}Since the conditional probability is a product of $d$ terms, it is natural to consider taking $\log$ in order to convert it into a summation. This motivates the use of Point-wise Mutual Information (PMI), which is a common metric for the correlations between two events. Given events $X$ and $Y$, the PMI is defined as
$$
\text{PMI}(X,Y) := \log \frac{\Pr[X,Y]}{\Pr[X]\Pr[Y]}.
$$
Intuitively, if $X,Y$ are independent, then $\text{PMI} = 0$; if they are positively correlated, then $\text{PMI} > 0$; if they are negatively correlated, then $\text{PMI} < 0$. This can be also generalized to three random variables as
\begin{equation*}
\text{PMI3}(X, Y, Z) := \log \frac{\Pr[X,Y]\Pr[Y,Z]\Pr[X,Z]}{\Pr[X,Y,Z]\Pr[X]\Pr[Y]\Pr[Z]}.
\end{equation*}

For the noisy-or networks, we use PMI and PMI3 as the (generalized) moments that we observe. More precisely, we define the following PMI matrix $M \in \R^{d \times d}$ and PMI tensor $T \in \R^{d \times d \times d}$ as
\begin{align*}
M_{i_1,i_2} &:=  \text{PMI}(1-x_{i_1}, 1-x_{i_2}), \quad i_1,i_2 \in [d].   \\
T_{i_1,i_2,i_3} &:=  \text{PMI3}(1-x_{i_1}, 1-x_{i_2}, 1-x_{i_3}), \quad i_1,i_2,i_3 \in [d].
\end{align*}
These tabulate the correlations among all pairs and triples of symptoms; more specifically, they incorporate indicator random variable for the symptom being absent.

As before, we would like to have a low rank decomposition for these observed matrix and tensor. This is almost true except for some small perturbations as follows. For convenience, we define $F, G\in \R^{d\times k}$ as
\begin{align*}
F & := 1- \exp(-W) \\
G &:= 1-\exp(-2W).
\end{align*}
Using these quantities we can approximately represent the PMI matrix and tensor in low rank forms.
 
\begin{proposition}[Proposition 2.1 in \cite{arora2016provable}]\label{prop:lowrank} Let  $F_j, G_j \in \R^d$ denote the $j$-th columns of the above matrices $F, G$, respectively.  Then we have
\begin{align*}
M & \approx \rho \left(FF^{\top} + \rho GG^{\top}\right)
= \rho \sum_{j=1}^k  F_jF_j^{\top} + \rho^2 \sum_{j=1}^k   G_j G_j^{\top}  
\\	T &\approx  \rho \sum_{j=1}^k   F_j \otimes F_j \otimes F_j + \rho^2 \sum_{j=1}^k   G_j\otimes G_j \otimes G_j.
	\end{align*}
\end{proposition}

The approximation in both equations are due to higher order terms in the Taylor expansions and are dominated by the term with $G$. Recall that $\rho$ is the probability of any disease $h_j$ being present, and therefore, for this application we expect $\rho$ to be small. Hence, the terms with $F$ are much larger than the terms with $G$  and we can say that applying tensor decomposition approximately recovers columns of $F$. Several ideas and many more details are required in analyzing the effect of the perturbation $G$ (since $G$ is not as small as required in Theorem~\ref{thm:robustpower}); we refer interested readers to \cite{arora2016provable}.

\section{Applications in Supervised Learning} \label{sec:moment-supervised}

In this section, we describe how tensor methods can be also used in supervised learning applications contrasting with the unsupervised problems that we have described so far. In particular, we focus on neural networks to elaborate on this application of tensor methods. This is a very interesting extension given the vast applications of neural networks that have significantly improved predictive performance across multiple domains such as computer vision and speech recognition with rapidly growing influence in many other areas. 
Although we focus on neural networks in this monograph, similar tensor techniques are also applied to learning mixtures of generalized linear models in~\citet{sedghi2016provable}. Most of the discussions and results in the rest of this section are borrowed from \citet{janzamin2015beating}.

In previous sections, we discussed the application of tensor methods for learning latent variable models and latent representations which was performed in an unsupervised manner. Thus, when considering supervised learning tasks such as training neural networks, the first natural and major question that we have to answer is how to adapt these tensor methods to supervised learning.
To answer this, we incorporate a generative approach in the problem and propose non-linear transformation of the input which is basically new features extracted from the input. We refer to this new transformation as {\em score function} of the input. These new extracted features enable us to formulate the problem of training neural networks as the tensor decomposition problem. More concretely, we show that the cross-moment between output and the score function of  the input has information about the weight parameters of the neural network in its rank-1 components. Before providing more details, we first elaborate more on exploiting a generative model, and in particular, the score functions.

\paragraph{Generative vs.\ discriminative models:}Generative models incorporate a joint distribution $p(x,y)$ over both the input $x$ and label $y$. On the other hand, discriminative models such as neural networks only incorporate the conditional distribution $p(y|x)$. While training neural networks for general input $x$ is NP-hard, does  knowledge about the input distribution $p(x)$ make  learning tractable?

Here, we assume knowledge of   the input density $p(x)$ which can be any continuous differentiable function.
While unsupervised learning problem of estimation of density $p(x)$ is itself a hard problem for general models, here we investigate how $p(x)$ can be exploited to make training of neural networks tractable. The knowledge of $p(x)$ is naturally available in the {\em experimental design} framework, where the person designing the experiments has the ability to choose the input distribution. Examples include   conducting polling, carrying out drug trials, collecting survey information, and so on.

We utilize the knowledge about the input density $p(x)$ (up to normalization)\footnote{We do not require the knowledge of the normalizing constant or the partition function, which is $\#P$ hard to compute~\citep{wainwright2008graphical}.} to obtain certain (non-linear) transformations of the input, given by the class of score functions.   {\em Score functions} are normalized derivatives of the input pdf; see~\eqref{eqn:diffoperator}.  If the input is a vector (the typical case), the first order score function (i.e., the first  derivative) is a vector, the second order score is a matrix, and the higher order scores are tensors.

\subsection{Moment tensor structure in neural networks}

We consider a  neural network with one hidden layer of dimension $k$. Let the output $\tl{y} \in \{0,1\}$ be the binary label, and $x \in \R^d$ be the feature (input) vector; see \citet{janzamin2015beating} for generalization to higher dimensional output (multi-label and multi-class), and also the continuous output case. 
We consider the label generating model
\begin{equation} \label{eq:nn2}
\quad \tl{f}(x) := \Ebb[\tl{y}|x]= \inner{a_2,\sigma(A_1^\top x+b_1)}+b_2,
\end{equation}
where $\sigma(\cdot)$ is a (linear/nonlinear) element-wise function named as activation function; 
see Figure~\ref{fig:DeepNet} for a schematic representation of label-function in~\eqref{eq:nn2} in the general case of vector output $\tl{y}$. 

In this section, we only focus on the {\em realizable} setting. In this setting, the goal is to learn the parameters of the neural network specified in~\eqref{eq:nn2}, i.e., to learn the weight matrices (vectors) $A_1 \in \R^{d \times k}$, $a_2 \in \R^k$ and bias vectors $b_1 \in \R^k$, $b_2 \in \R$, given labeled data samples $\{(x_i, \tl{y}_i)\}$. This only involves the estimation analysis where we have a label-function $\tl{f}(x)$ specified in~\eqref{eq:nn2} with fixed unknown parameters $A_1, b_1, a_2, b_2$, and we would like to learn these parameters and finally bound the overall function estimation error $\E_x[|\tl{f}(x) - \hf(x)|^2]$, where $\hf(x)$ is the estimation of fixed neural network $\tl{f}(x)$ given finite samples. The approximation\footnote{Here by approximation we mean how accurate the neural network can approximate any arbitrary function $f(x)$.} analysis and consequently the risk bound is out of the focus of this monograph and the interested reader is referred to read \citet{janzamin2015beating} for details.

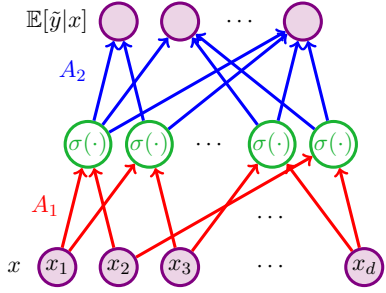
\begin{figure}[t]
\begin{center}
\resizebox{0.45\textwidth}{!}{
\begin{tikzpicture}
  [
    scale=1,
    observed/.style={circle,minimum size={width("$x_{d}$")+6pt},inner
sep=0mm,draw=violet,fill=lpurple,line width=.5mm},
    hidden/.style={circle,minimum size=0.6cm,inner sep=0mm,draw=dkg,line width=.5mm},
        func/.style={circle,minimum size=0.6cm,inner sep=0mm,draw=blue1,dashed, line width=.5mm},
        vdots/.style={min, node distance=.5mm},
  ]
  \node [hidden,name=f1] at ($(-2,0)$) {$\tcdkg{\sigma(\cdot)}$};
  \node [hidden,name=f2] at ($(-1,0)$) {$\tcdkg{\sigma(\cdot)}$};
  \node [hidden,name=fk] at ($(2,0)$) {$\tcdkg{\sigma(\cdot)}$};
 \node [hidden,name=fn] at ($(1,0)$) {$\tcdkg{\sigma(\cdot)}$};
  \node[observed,name=y1] at ($(-1.5,2)$){};
  \node [observed,name=yk] at ($(-.5,2)$){};
  \node [observed,name=y2] at ($(1.5,2)$){};
 \node [observed,name=x11] at ($(-2.5,-2)$) {$x_1$}; 
   \node [observed,name=x1] at ($(-1.5,-2)$) {$x_2$}; 
  \node [observed,name=x2] at ($(-0.5,-2)$) {$x_3$}; 
  \node [observed,name=xd2] at ($(2.5,-2)$) {$x_{d}$}; 
  \node [] at ($(-3.2,-2)$) {$x$};
  \node [] at ($(-2.5,2)$) {$\E[\tl{y}|x]$};
  \node [] at ($(-2.25,1.2)$) {\tcb{$A_2$}};
   \node [] at ($(-2.7,-1)$) {\tcr{$A_1$}};
    \node at ($(0.5,2)$) {$\dotsb$};
        \node at ($(1,-1.2)$) {$\dotsb$};
  \node at ($(0,0)$) {$\dotsb$};
   \node at ($(1,-2)$) {$\dotsb$};

  \draw [blue, line width=.5mm, ->] (f1) to (y1);
  \draw [blue, line width=.5mm, ->] (f1) to (y2);
  \draw [blue, line width=.5mm, ->] (f1) to (yk);
 \draw [blue, line width=.5mm,->] (f2) to (y2);
  \draw [blue, line width=.5mm, ->] (f2) to (y1);
  \draw [blue, line width=.5mm, ->] (fk) to (yk);
  \draw [blue, line width=.5mm, ->] (fk) to (y2);
   \draw [blue, line width=.5mm, ->] (fn) to (y2);
   \draw [blue, line width=.5mm, ->] (fn) to (yk);

    \draw [red, line width=.5mm, <-] (f1) to (x1);
  \draw [red, line width=.5mm, <-] (f1) to (x11);
   \draw [red, line width=.5mm, <-] (f2) to (x11);
  \draw [red, line width=.5mm,<-] (f2) to (x2);
  \draw [red, line width=.5mm,<-] (fk) to (xd2);
  \draw [red, line width=.5mm, <-] (fk) to (x1);
  \draw [red, line width=.5mm, <-] (fn) to (x2);
  \draw [red, line width=.5mm,<-] (fn) to (xd2);

\end{tikzpicture}
}
\end{center}
\caption[Graphical representation of a neural network with single hidden layer]{ Graphical representation of a neural network, $\Ebb[\tl{y}|x] = A_2^\top \sigma(A_1^\top x+b_1)+b_2$.}
\label{fig:DeepNet}
\end{figure}

We are now ready to explain how learning the parameters of two-layer neural network in Figure~\ref{fig:DeepNet} can be characterized as a tensor decomposition algorithm. Note that we only provide the tensor decomposition algorithm for learning the parameters of first layer and as described in~\citet{janzamin2015beating}, the bias parameter in first layer is learned using a Fourier method and the parameters of second layer are learned using linear regression. These parts are not within the focus of this monograph and we refer the reader to \citet{janzamin2015beating} for more details.  Note that most of the unknown parameters (compare the dimensions of matrix $A_1$, vectors $a_2$, $b_1$, and scalar $b_2$) are estimated in the first part, and thus, the tensor decomposition method for estimating $A_1$ is the main part of the learning algorithm.

In order to provide the tensor structure, we first define the score functions as follows.

\subsubsection{Score function} \label{sec:ScoreFunc}
The $m$-th order score function $\Sc_m(x) \in \bigotimes^m \R^d$ is defined as~\citep{janzamin2014matrix}
\begin{equation} \label{eqn:diffoperator}
\Sc_m(x) := (-1)^m \frac{\nabla_x^{(m)} p(x)}{p(x)},
\end{equation}
where $p(x)$ is the probability density function of random vector $x \in \R^d$, and $\nabla_x^{(m)}$ denotes the $m$-th order derivative operator defined as follows. For function $g(x): \R^d \rightarrow \R$ with vector input $x \in \R^d$, the $m$-th order derivative w.r.t.\ variable $x$ is denoted by $\nabla_x^{(m)} g(x) \in \bigotimes^{m} \R^d$ (which is a $m$-th order tensor) such that
\begin{equation} \label{eqn:derivativedef}
\left[ \nabla_x^{(m)} g(x) \right]_{i_1,\dotsc,i_m} := \frac{\partial g(x)}{\partial x_{i_1} \partial x_{i_2} \dotsb \partial x_{i_m}}, \quad i_1,\dotsc,i_m \in [d].
\end{equation}
When it is clear from the context, we drop the subscript $x$ and write the derivative as $\nabla^{(m)} g(x)$.

The main property of score functions as yielding differential operators that enables us to estimate the weight matrix $A_1$ via tensor decomposition is discussed in next subsection; see Equation~\eqref{eqn:socreFunc-diffOperator}.

Note that in this framework, we assume  access to a sufficiently good approximation of the input pdf $p(x)$ and the corresponding  score functions $\Sc_2(x)$, $\Sc_3(x)$. Indeed, estimating these quantities in general is a hard problem, but there exist numerous instances where this becomes tractable. Examples include spectral methods for learning latent variable models such as Gaussian mixtures, topic or admixture models, independent component analysis (ICA) and so on that we discussed in the previous sections. Moreover, there have been recent  advances in non-parametric score matching methods~\citep{sriperumbudur2013density} for density estimation in infinite dimensional exponential families with guaranteed convergence rates. These methods can be used to estimate the input pdf in an unsupervised manner. Below, we discuss more about score function estimation methods. In this work, we focus on how we can use the input generative information to make   training of  neural networks tractable. We refer the interested reader to \citet{janzamin2015beating} for more discussions on this. 

\paragraph{Estimation of score function}There are various efficient methods for estimating the score function. The framework of score matching is popular for parameter estimation  in probabilistic models~\citep{hyvarinen2005estimation, swersky2011autoencoders}, where the criterion is to fit parameters based on matching the data score function. \citet{swersky2011autoencoders} analyze the score matching for latent energy-based models.
In deep learning, the framework of auto-encoders attempts to find encoding and decoding functions which minimize the reconstruction error under added noise; the so-called Denoising Auto-Encoders (DAE). This is an unsupervised framework involving only unlabeled samples. \citet{alain2012regularized} argue that the DAE   approximately learns the first order score function of the input, as the noise variance goes to zero. ~\citet{sriperumbudur2013density} propose non-parametric score matching methods for density estimation in infinite dimensional exponential families with guaranteed convergence rates. Therefore, we can use any of these methods for estimating $\Sc_1(x)$ and use the recursive form~\citep{janzamin2014matrix}
$$\Sc_m(x) = - \Sc_{m-1}(x) \otimes \nabla_x \log p(x) - \nabla_x \Sc_{m-1}(x)$$ to estimate higher order score functions. Despite the existence of these techniques, there still exist so much room for proposing methods to efficiently estimate score functions.

\subsubsection{Tensor form of the moment} \label{sec:tensordecomp}
The score functions are new representations (extracted features) of input data $x$ that can be used for training neural networks.
The score functions have  the property of yielding  differential operators with respect to the input distribution. 
More precisely, for label-function $\tl{f}(x) := \E[\tl{y}|x]$, \citet{janzamin2014matrix} show that
\begin{equation} \label{eqn:socreFunc-diffOperator}
\E[\tl{y} \cdot \Sc_3(x)] = \E[\nabla_x^{(3)} \tl{f}(x)].
\end{equation}
Now for the neural network output in~\eqref{eq:nn2}, note that the function $\tl{f}(x)$ is a non-linear function of both input $x$ and weight matrix $A_1$. The expectation operator $\Ebb[\cdot]$ averages out the dependency on $x$, and the derivative acts as a {\em linearization operator} as follows. In the neural network output~\eqref{eq:nn2}, we observe that  the columns of weight vector $A_1$ are the linear coefficients involved with input variable $x$. When taking the derivative of this function, by the chain rule, these linear coefficients show up in the final form. With this intuition, we are now ready to provide the precise form of the moment where we show how the cross-moment between label and score function as $\E[\tl{y} \cdot \Sc_3(x)]$ leads to a tensor decomposition form for estimating weight matrix $A_1$:

\begin{lemma}[\citep{janzamin2015beating}] \label{lem:moment}
For the two-layer neural network specified in~\eqref{eq:nn2}, we have
\begin{equation} \label{eqn:cross-moment}
\Ebb \left[ \tl{y} \cdot \Sc_3(x) \right] = \sum_{j \in [k]} \lambda_j \cdot (A_1)_j \otimes (A_1)_j \otimes (A_1)_j,
\end{equation}
where $(A_1)_j \in \R^d$ denotes the $j$-th column of $A_1$, and 
\begin{equation} \label{eqn:cross-moment-coeffs}
\lambda_j = \E \left[ \sigma'''(z_j) \right] \cdot a_2(j),
\end{equation}
for vector $z := A_1^\top x+b_1$ as the input to the nonlinear operator $\sigma(\cdot)$.
\end{lemma}

This is proved  by the main property of score functions as yielding differential operators that was described earlier. This lemma shows that by decomposing the cross-moment tensor $\E[\tl{y} \cdot \Sc_3(x)]$, we can recover the columns of $A_1$. This clarifies how the score function acts as a linearization operator while the final output is nonlinear in terms of $A_1$.

\section{Other Models}
\label{sec:othermodels}
Tensor decompositions have been applied to learn many other models. Several ideas we introduced in this section originated from more complicated settings, and can be applied to more models. Here we give hints to more examples, but the list is by no means complete.

The idea of manipulating moments was well-known in the ICA (Independent Component Analysis) literature, where cumulants are used frequently instead of moments. For other distributions, it was used for the Latent Dirichlet Allocation~\citep{SpectralLDA}, and widely applied in all the models where the hidden variables are not categorical. 

The multi-view model was first used in \citet{MR06} to learn Hidden Markov Models and Phylogeny Tree Reconstruction. The original technique in \citet{MR06} was based on spectral algorithms and was not viewed as tensor decomposition, however it is very similar to the simultaneous diagonalization algorithm we introduced in Section~\ref{sec:simdiag}.

Tensor methods can be also applied to learning more complicated mixtures of Gaussians, where each component may have a different, non-spherical component~\citep{ge2015learning}. The covariance matrix creates many technical problems which is beyond the scope of this monograph. The idea of creating different views as we discussed in Section~\ref{sec:multi} can be also applied to learning community models~\citep{AnandkumarEtal:community12COLT}. 

Tensor decomposition is also particularly useful in the context of deep neural networks, most notably with the aim to speed up computation. One way to do so is to apply tensor factorization to the kernel of convolutional layers~\citep{tai2016convolutional}. In particular, by applying CP decomposition to the convolutional kernel of a pre-trained network, not only is it possible to reduce the number of parameters, but it also gives a way of re-expressing the convolution in terms of a series of smaller and more efficient convolutions. Both ALS~\citep{lebedev2015speeding} and tensor power method~\citep{astrid2017cp} have been considered. This process typically results in a performance deterioration which is restored by fine-tuning. A similar result can be obtained using Tucker decomposition~\citep{yong2016compression}. It is possible to go further and jointly parameterize multiple layers or whole networks, resulting in large parameter space savings without loss of performance~\citep{kossaifi2019t}.

We can also preserve the multi-linear structure in the activation tensor, using tensor contraction~\citep{kossaifi2017tensor}, or by removing fully connected layers and flattening layers altogether and replacing with tensor regression layers~\citep{kossaifi2018tensor}. Adding a stochastic regularization on the rank of the decomposition can also help render the models more robustly~\citep{stochastic_rank}. \emph{Tensorization} can be also leveraged by applying it to the weight matrix of fully-connected layers~\citep{novikov2015tensorizing}.\clearpage{}

\clearpage{}\chapter{Practical Implementations}\label{ch:implementation}

We have so far covered many aspects of tensors including tensor decomposition and how they are useful in learning different machine learning models in both supervised and unsupervised settings. In this section, we discuss practical implementation of tensor operations using Python programming language. We first motivate our choice and introduce some actual code to perform tensor operations and tensor decomposition.  We then briefly show how to perform more advanced tensor operations using TensorLy~\citep{tensorly}, a library for tensor learning in Python. Finally, we show how to scale up our algorithms using the PyTorch deep learning framework~\citep{pytorch} as a backend for TensorLy.

\section{Programming Language and Framework}

Throughout this section, we present the implementations in Python language.
Python is a multi-purpose and powerful programming language that is emerging as the prime choice for Machine Learning and data science. Its readability allows us to focus on the underlying concepts we are implementing without getting distracted by low-level considerations such as memory handling or obscure syntax. Its huge popularity means that good libraries exist to solve most of our computational needs. In particular, NumPy~\citep{numpy} is an established and robust library for numerical computation. It offers a high performance structure for manipulating multi-dimensional arrays. TensorLy builds on top of this and provides a simple API for fast and easy tensor manipulation. TensorLy has a system of backends that allows you to switch transparently from NumPy to PyTorch, MXNet, TensorFlow, etc. This means you can perform any of the operations seamlessly on all these  frameworks. In particular, using a deep learning framework such as PyTorch as backend, it is easy to scale operations to GPUs and multi-machines.

\subsection{Pre-requisite}
In order to run the codes presented in this section, you will need a working installation of Python 3.0, along with NumPy (for the numerical array structure), SciPy~\citep{scipy}  (for scientific python), and optionally Matplotlib~\citep{matplotlib} for visualization.

The easiest way to get all these is to install the Anaconda distribution (\url{https://anaconda.org/}) which comes with all the above bundled and pre-compiled so you do not have to do anything else!

\section{Tensors as NumPy Arrays}

You may recall from Section~\ref{sec:totensors} that tensors can be identified as multi-dimensional arrays. Therefore, we represent tensors as NumPy arrays, which are multi-dimensional arrays. 

Let's take as an example a tensor \(T \in \R^{3 \times 4 \times 2}\), defined by the following frontal slices:

\[
    T(:, :, 1) = 
       \left[
       \begin{matrix}
       0  & 2  & 4  & 6\\
       8  & 10 & 12 & 14\\
       16 & 18 & 20 & 22
       \end{matrix}
       \right]
\]

and

\[
    T(:, :, 2) =
       \left[
       \begin{matrix}
       1  & 3  & 5  & 7\\
       9  & 11 & 13 & 15\\
       17 & 19 & 21 & 23
       \end{matrix}
       \right]
\]

In NumPy we can instantiate new arrays from nested lists of values. For instance, matrices are represented as a list of rows, where each row is itself a list.
Let's define the slices of above tensor $T$ as 2-D NumPy arrays:
    
\begin{python}
# We first import numpy
import numpy as np

# First frontal slice
T1 = np.array([[  0.,   2.,   4.,   6.],
               [  8.,  10.,  12.,  14.],
               [ 16.,  18.,  20.,  22.]])

# Second frontal slice
T2 = np.array([[  1.,   3.,   5.,   7.],
               [  9.,  11.,  13.,  15.],
               [ 17.,  19.,  21.,  23.]])
\end{python}

Let's now write a function that stacks these frontal slices into a third order tensor:

\begin{python}
def tensor_from_frontal_slices(*matrices):
    """Creates a tensor from its frontal slices
    
    Parameters
    ----------
    matrices : 2D-Numpy arrays
    
    Returns
    -------
    tensor : 3D-NumPy arrays
        its frontal slices are the matrices passed as input
    """
    return np.concatenate([matrix[:, :, np.newaxis]\
                          for matrix in matrices], axis=-1)
\end{python}

We can then build the full tensor \(T\) from its frontal slices $T1$ and $T2$ created above:

\begin{python}
T = tensor_from_frontal_slices(T1, T2)
\end{python}

We can inspect the frontal slices naturally using almost the same notation as in the math.
To do so we fix the last index while iterating over other modes (using `:').

\begin{python}
>>> T[:, :, 0]
array([[  0.,   2.,   4.,   6.],
       [  8.,  10.,  12.,  14.],
       [ 16.,  18.,  20.,  22.]])
>>> T[:, :, 1]
array([[  1.,   3.,   5.,   7.],
       [  9.,  11.,  13.,  15.],
       [ 17.,  19.,  21.,  23.]])
\end{python}

Remember that in NumPy (and generally, in Python), like in C, indexing starts at zero. In the same way, you can also inspect the horizontal slices (by fixing the first index) and lateral slices (by fixing the second index).

Similarly, we can easily inspect the fibers which, as you may recall, are higher-order analogues to column and rows. We can obtain the fibers of $T$ by fixing all indices but one:

\begin{python}
# First column (mode-1 fiber)
>>> T[:, 0, 0]
array([  0.,   8.,  16.])

# First row (mode-2 fiber)
>>> T[0, :, 0]
array([ 0.,  2.,  4.,  6.])

# First tube (mode-3 fiber)
>>> T[0, 0, :]
array([ 0.,  1.])
\end{python}

Finally, you can access the size of a tensor via its \emph{shape}, which indicates the size of the tensor along each of its modes. For instance, our tensor $T$ has shape $(3, 4, 2)$:

\begin{python}
>>> T.shape
(3, 4, 2)
\end{python}

\section{Basic Tensor Operations and Decomposition}

Tensor matricization, or unfolding, as introduced in Equation~\eqref{eqn:matricization} and described in Procedure~\ref{algo:tensor_unfolding} naturally translates into Python. One important consideration when implementing algorithms that manipulate tensors is the way elements are organised in memory. You can think of the memory as one long vector of numbers. Because of the way CPU and GPU operate, it matters how these elements are layered in the memory. To store a matrix, for instance, we can either organise the elements row-after-row (also called C-ordering) or column-after-column (also called Fortran ordering). In NumPy, elements are organised by default in row-order, same for PyTorch. It so happens that the definition of the unfolding we use is adapted for such ordering, thus avoiding expensive reordering of the data.

\floatname{algorithm}{Procedure}
\begin{algorithm}[t]
\caption{Tensor unfolding}
\label{algo:tensor_unfolding}
\begin{algorithmic}[1]
\renewcommand{\algorithmicrequire}{\textbf{input}}
\renewcommand{\algorithmicensure}{\textbf{output}}
\REQUIRE Tensor $T$ of shape $(d_1, d_2, \cdots, d_n)$; unfolding mode $m$.
\ENSURE Mode-$m$ matricization (unfolding)
\STATE Move the $m^{\text{th}}$ dimension to the first position.
\STATE Reshape into a matrix $M$ of shape $(d_m, \prod_{k \neq m} d_k)$.
\RETURN $M$.
\end{algorithmic}
\end{algorithm}

As a result, matricization (or unfolding) of a tensor along a given mode simplifies to moving that mode to the front and reshaping into a matrix as also described in Procedure~\ref{algo:tensor_unfolding}.

\begin{python}
def unfold(tensor, mode):
    """Returns unfolding of a tensor -- modes starting at 0.
    
    Parameters
    ----------
    tensor : ndarray
    mode : int (default is 0), mode along which to unfold
    
    Returns
    -------
    ndarray
        unfolded_tensor of shape
    """
    return np.reshape(np.moveaxis(tensor, mode, 0),
                      (tensor.shape[mode], -1))
\end{python}

Folding the tensor is done by performing the inverse operations: we first reshape the matrix into a tensor and move back the first dimension to its original position.

\begin{python}
def fold(unfolded_tensor, mode, shape):
    """Refolds the unfolded tensor into a full tensor.
        In other words, refolds the n-mode unfolded tensor
        into the original tensor of the specified shape.
    
    Parameters
    ----------
    unfolded_tensor : ndarray
        unfolded tensor of shape ``(shape[mode], -1)``
    mode : int
        the mode of the unfolding
    shape : tuple
        shape of the original tensor before unfolding
    
    Returns
    -------
    ndarray
        folded_tensor of shape `shape`
    """
    full_shape = list(shape)
    mode_dim = full_shape.pop(mode)
    full_shape.insert(0, mode_dim)
    return np.moveaxis(np.reshape(
              unfolded_tensor, full_shape), 0, mode)
\end{python}

\subsection{CP decomposition}
Now that we know how to manipulate tensors using NumPy arrays, we are ready to implement a simple version of the CP decomposition via Alternating Least Squares, as explained in Section~\ref{sec:als}. We will start by writing the auxiliary functions we need in the main algorithm.

CP decomposition expresses its input tensor as a sum of outer products of vectors; see Equation~\eqref{eqn:tensordecomp} for the definition. Taking the unfolded expression, there is a useful equivalent formulation that uses the Khatri-Rao product which we used in Equation~\eqref{eqn:ALSoptMod}. In particular, for vectors $u,v,w$, we have
\begin{equation*}
\mat(u \otimes v \otimes w,1) = u \cdot (v \odot w)^\top. 
\end{equation*}
Note that here, $mat(., 1)$ corresponds to unfolding along mode $0$ in our code.

Let's first write a function to take the Khatri-Rao product of two matrices, as defined in equation~\eqref{eqn:KhatriRao}.
A naive, literal implementation of that equation could be as follows:
\begin{python}
def naive_khatri_rao(A, B):
    # Both matrices must have the same number of columns k
    d1, k = A.shape
    d2, k = B.shape
    
    # The khatri-rao product has size d1d2 x k
    C = np.zeros((d1*d2, k))
    for i in range(d1):
        for l in range(d2):
            for j in range(k):
                # Indexing starts at 0!
                C[l + i*d1, j] = A[i, j]*B[l, j]
    return C
\end{python}

However, loops are typically slow in Python and this naive implementation is as a result extremely slow. By contrast, we can use the built-in \texttt{einsum} function from NumPy, which uses Einstein's notation to define the operation, to write a \emph{vectorized}  version. This results in a much more efficient function:

\begin{python}
def khatri_rao(matrix1, matrix2):
    """Returns the khatri-rao product of matrix1 and matrix2
    """
    n_columns = matrix1.shape[1]
    result = np.einsum('ij,lj->ilj', matrix1, matrix2)
    return result.reshape((-1, n_columns))
\end{python}

Recall that the khatri-rao takes a column-wise Kronecker product of two matrices with the same number of columns. The \texttt{einsum} function here expresses this idea in terms of indices, where $A$ is indexed by $i$ and $j$ and $B$ is indexed by $l$ and $j$. The output is of size $ilj$ and we simply have to reshape it into a matrix of the appropriate size.

Then, given a third order tensor in its Kruskal form (i.e., a decomposed tensor, expressed as a series of factors $A, B$ and $C$ with unit norm and the associated vector of coefficients  $\lambda$ implying the norms), we need a method to return the reconstruction $T = \sum_{j\in[k]} \lambda_j \ a_j\otimes b_j\otimes c_j$. Using the above matricization property, this reconstruction can also be written in its unfolded form as $mat(T, 1) = \tl{A} \cdot \diag(\tl{\lambda}) \cdot (\tl{B} \odot \tl{C})^\top$, resulting in the following function:
\begin{python}
def kruskal_to_tensor(weights, A, B, C):
    """Converts the kruskal form into a tensor
    """
    full_shape = (A.shape[0], B.shape[0], C.shape[0])
    # Reconstruct in unfolded form
    unfolded_tensor = np.dot(A.dot(np.diag(weights)), 
                             khatri_rao(B, C).T)
    # Fold back to a tensor
    return fold(unfolded_tensor, 0, full_shape)
\end{python}

To measure convergence, we can use, for instance, the Frobenius norm of the reconstruction error. Recall that the Frobenius norm is simply the square root of the sum of the squared elements of the tensor. This can be written in NumPy as :

\begin{python}
def frobenius_norm(tensor):
    """Frobenius norm of the tensor
    """
    return np.sqrt(np.sum(tensor**2))
\end{python}

We are now ready to implement the Alternating Least Squares method for Tensor Decomposition described in Algorithm~\ref{algo:als}.

\begin{python}
def parafac(tensor, rank, l2_reg=1e-3, n_iter_max=200, tol=1e-10):
    """CANDECOMP/PARAFAC decomposition via ALS

    Parameters
    ----------
    tensor : ndarray
    rank  : int
            number of components
    l2_reg : float, default is 0.1
            regularization parameter (\alpha)
    n_iter_max : int
                 maximum number of iterations
    tol : float, optional
          tolerance: the algorithm stops when the variation in
          the reconstruction error is less than the tolerance
    verbose : int, optional
        level of verbosity
        
    Returns
    -------
    weights, A, B, C :  weights, factors of the decomposition
    """
    # Initialize the factors of the decomposition randomly
    A = np.random.random_sample((tensor.shape[0], rank))
    B = np.random.random_sample((tensor.shape[1], rank))
    C = np.random.random_sample((tensor.shape[2], rank))

    # Norm of the input tensor
    norm_tensor = frobenius_norm(tensor)
    error = None
    
    # Initalize the weights to 1
    weights = np.ones(rank)
    
    # Avoid division by zero
    eps = 1e-12
    
    # Regularization term \alpha*I
    regularization = np.eye(rank)*l2_reg
    
    for iteration in range(n_iter_max):
        # Update A
        prod = B.T.dot(B)*C.T.dot(C) + regularization
        factor = unfold(tensor, 0).dot(khatri_rao(B, C))
        A = np.linalg.solve(prod.T, factor.T).T
        # Normalization (of the columns) of A
        weights = np.linalg.norm(A, ord=2, axis=0)
        A /= (weights[None, :] + eps)

        # Update B
        prod = A.T.dot(A)*C.T.dot(C) + regularization
        factor = unfold(tensor, 1).dot(khatri_rao(A, C))
        B = np.linalg.solve(prod.T, factor.T).T
        # Normalization of B
        weights = np.linalg.norm(B, ord=2, axis=0)
        B /= (weights[None, :] + eps)
        
        # Update C
        prod = A.T.dot(A)*B.T.dot(B) + regularization
        factor = unfold(tensor, 2).dot(khatri_rao(A, B))
        C = np.linalg.solve(prod.T, factor.T).T
        # Normalization of C
        weights = np.linalg.norm(C, ord=2, axis=0)
        C /= (weights[None, :] + eps)
        
        # Compute the reconstruction error
        prev_error = error
        rec = kruskal_to_tensor(weights, A, B, C)
        error = frobenius_norm(tensor - rec) / norm_tensor

        if iteration > 1:
            if tol and abs(prev_error - error) < tol:
                print('converged in {} iterations.'.format(
                      iteration))
                break

    return weights, A, B, C
\end{python}

Using our previously introduced tensor $T$ as an example, we can verify that our algorithm indeed does what it is supposed to:

\begin{python}
# decompose T into factors using CP
weights, A, B, C = parafac(T, 3)

# reconstruct the full tensor from these
rec = kruskal_to_tensor(weights, A, B, C)

# verify that the reconstruction is correct
np.testing.assert_array_equal(np.round(rec), T)
\end{python}

Let's now go over some aspects of the algorithm we just wrote, in particular, how we integrated unit-norm constraints on the columns of the factor, as well as $\ell_2$ regularization.

\paragraph{Normalization:} Within the CP decomposition method, after updating each factor, we further normalize it by dividing each column by its norm, as also done in equation~\eqref{eqn:ALS-mode1}. For example, for the first factor matrix A, we have:

\begin{python}
    # First, update the factor as previously
    A = np.linalg.solve(prod.T, factor.T).T

    # Normalize the columns
    weights = np.linalg.norm(A, ord=2, axis=0)
    A /= (weights[None, :] + eps)
\end{python}
We do similar normalization for the other two factor matrices B and C. 
Note that we have also added a tiny value \texttt{eps} to the normalization, where \texttt{eps} is defined as $10^{-12}$, which is close to machine precision. This additional term is used to avoid any division by zero. Note that, here, we are using float64, which has a machine epsilon of about $10^{-15}$, this would have to be adapted when changing the data type (e.g. to float32).

\paragraph{Broadcasting:}In the last line of the update of A, the expression \texttt{weights[None, :]} is equivalent to  \texttt{weights[np.newaxis, :]}. In other words, we add a dimension (of $1$) to weights, and consider it as a matrix of size \texttt{(1, rank)} rather than a vector of length \texttt{rank}. This allows us to use \emph{broadcasting}: \texttt{weight} is broadcasted to the same shape as the factor without actually duplicating the memory. This results in an efficient \emph{vectorized} operation which divides each element of each column of the factor by the norm of that column.

This concept of broadcasting can also be used to simplify our \texttt{kruskal\_to\_tensor} by replacing the matrix multiplication of the first factor \texttt{A} and \texttt{diag(weights)} with a simple element-wise multiplication:
\begin{python}
def kruskal_to_tensor(weigths, A, B, C):
    """Converts the kruskal form into a tensor
    """
    full_shape = (A.shape[0], B.shape[0], C.shape[0])
    # The main difference: we incorporate the weights
    unfolded_tensor = np.dot(A*weigths[np.newaxis, :], 
                             khatri_rao(B, C).T)
    return fold(unfolded_tensor, 0, full_shape)
\end{python}

\paragraph{Regularization:}In section~\ref{sec:als}, we also introduced an $\ell_2$ regularized version of the ALS.
The difference with the unregularized version is an additional term in the pseudo-inverse in the ALS updates; see Equation~\eqref{eqn:ALS-model-l2}. Considering a regularization parameter $\alpha = \texttt{l2\_reg}$, the update for factor \texttt{A} changes by the addition of a weighted identity matrix $\alpha I $ to the product $(\tl{B} \odot \tl{C})^\top (\tl{B} \odot \tl{C})$ and similarly for \texttt{B} and \texttt{C}. In the code, \texttt{np.eye(rank)} is the identity matrix of size $\texttt{rank} \times \texttt{rank}$.

\section{Example: Image Compression via Tensor Decomposition}

We now use our function to compress an image. We use as an example an image of a raccoon that comes shipped in with the SciPy library.

\begin{python}
from scipy.misc import face

# Load the face
image = face()

# Convert it to a tensor of floats
image = np.array(image, dtype=np.float64)

# Check the size of the image
print(image.shape)
# (768, 1024, 3)
\end{python}

Our image is a third order tensor of shape (height, width, 3), the last mode corresponding to the RGB channels (Red, Green, Blue), the way colors are encoded on your computer. You can see the original image in Figure~\ref{fig:raccoon}, in this case with a height of $768$ and a width of $1024$.

To visualize the tensor, we need a helper function to convert tensors of floats (typically stored into 64 bits) into an image, which consists of values stored into 8 bits. Here, a simple conversion suffices since the image already has a dynamic range between $0$ and $255$ as it was originally stored in $8$ bits.
If the image had a high dynamic range (higher than $255$) then a more complex transformation (\emph{tone mapping}) such as histogram equalization would be needed.

\begin{python}
def to_image(tensor):
    """convert a tensor of float values into an image
    """
    tensor -= tensor.min()
    tensor /= tensor.max()
    tensor *= 255
    return tensor.astype(np.uint8)
\end{python}

This type of conversion, called tone mapping, can be much more complex than this simple conversion. Since we have a dynamic range between $0$ and $255$, it is appropriate here, but in general, when converting an image from 32 bits to just 8, we might want to use more complex techniques such as histogram normalization.

\begin{figure}
\bc
\includegraphics[width=.4\linewidth]{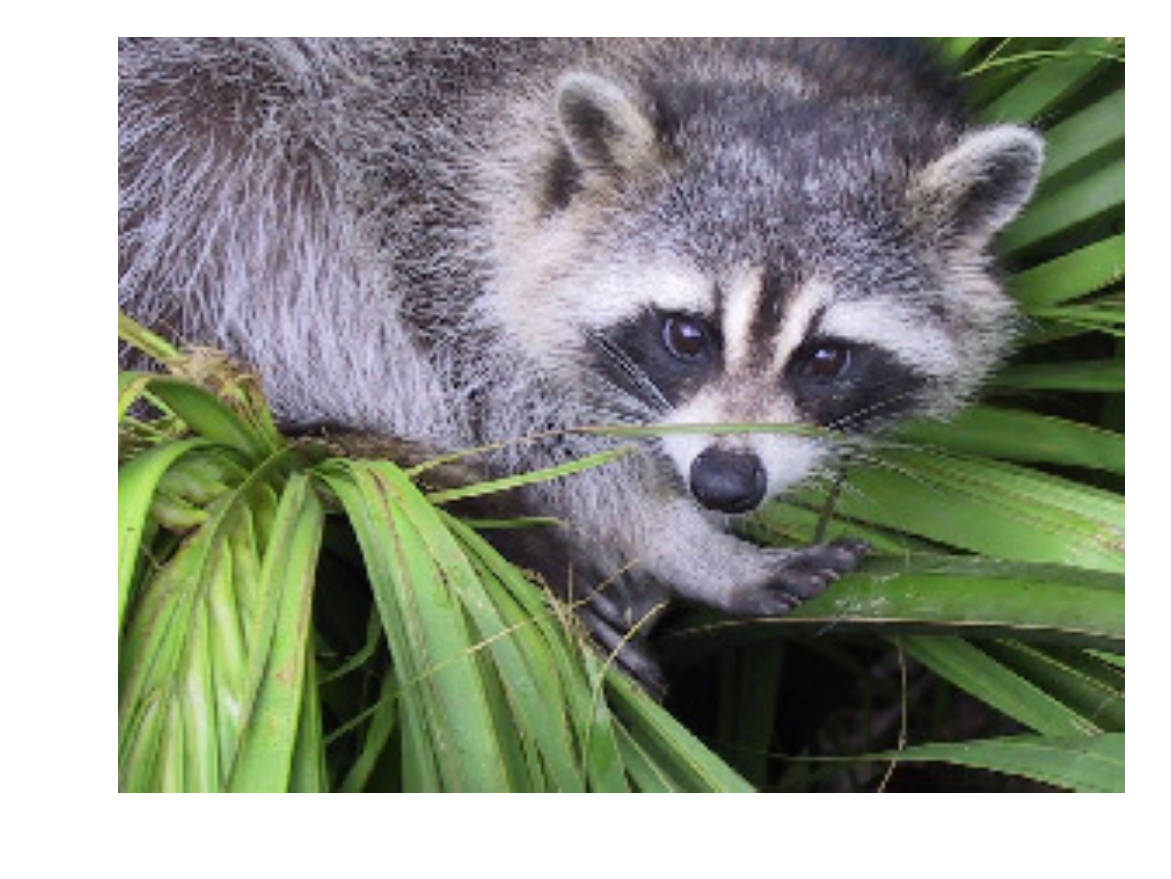}
\hspace{0.5in}
\includegraphics[width=.4\linewidth]{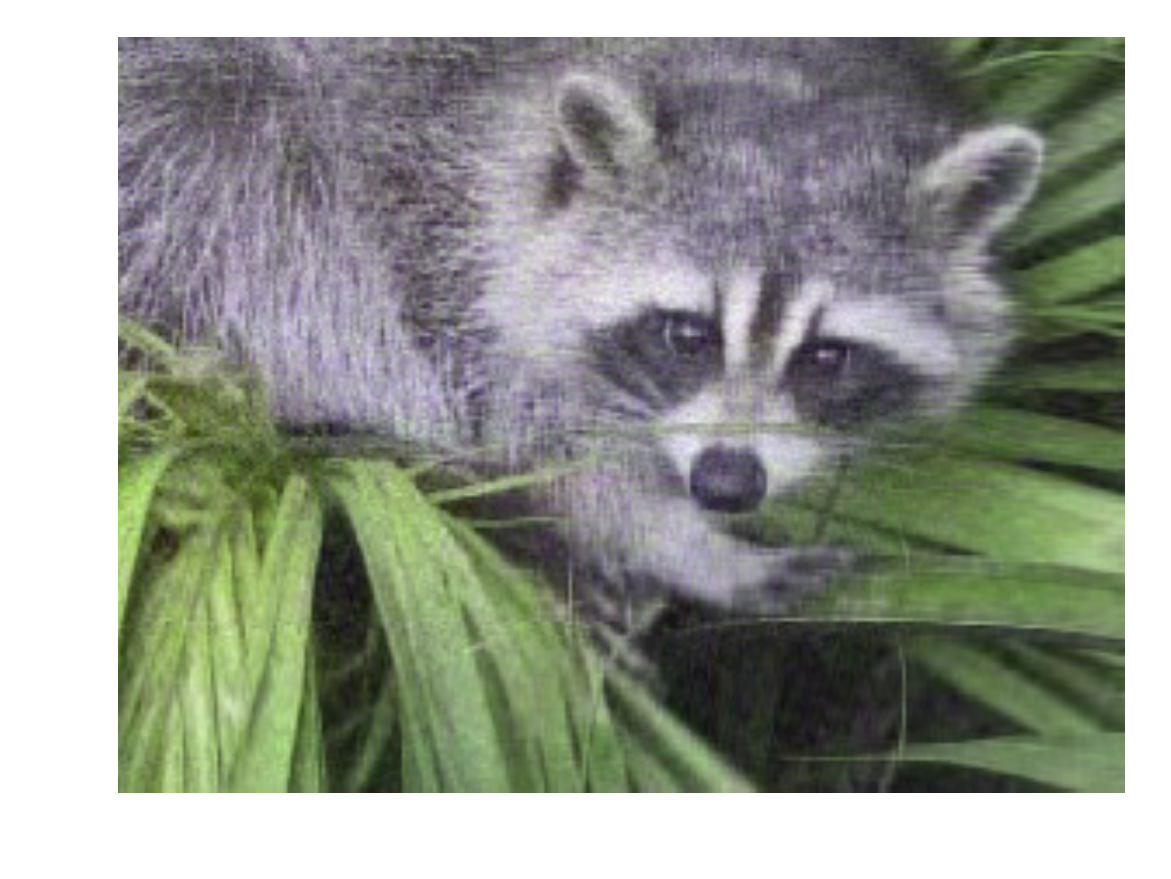}
\ec
\caption[]{Our beautiful guinea pig, which happens to be a raccoon. On the left, the original image, and on the right, the reconstructed image from the factors of the decomposition, with a rank $50$.}
\label{fig:raccoon}
\end{figure}

Now that the image is loaded in memory, we can apply our decomposition method to it, and build a reconstructed image from the compressed version, i.e., the factors of the decomposition,

\begin{python}
# Apply CP decomposition
weights, A, B, C = factors = parafac(image, rank=50, tol=10e-6)

# Reconstruct the full image
rec = kruskal_to_tensor(weights, A, B, C)
\end{python}

If you want to visualise the result, you can do so easily with Matplotlib:

\begin{python}
rec_image = to_image(rec)

#Import matplotlib to plot the image
import matplotlib.pyplot as plt
plt.imshow(rec_image)
plt.show()
\end{python}
The original image, shows in Figure~\ref{fig:raccoon}, has $\text{height} \times \text{width} \times \#\text{channels}$ $ = 768\times 1024 \times 3 = 2,359,296$ elements. The decomposition, on the other hand, expresses the image as a series of factors A, B and C containing respectively \(\text{height} \times \text{rank} = 768 \times 50\), 
\(\text{width} \times \text{rank} = 1024 \times 50\)
and \(\text{\# channels} \times \text{rank} = 3 \times 50\). 
In addition, we have to count the elements of \texttt{weights}, which is a vector of length equal to the rank. In total, the decomposition only has a total of less than $90,000$ parameters, or approximately $26\times$ less than the original image. Yet, as you can see in Figure~\ref{fig:raccoon}, the reconstructed image looks visually similar to the uncompressed image.

Note that the CP decomposition is not the best fit here, since the same rank is used for all modes, including the RGB channels. This is a case where a Tucker decomposition would be more adapted as we can select the \emph{Tucker rank} (or \emph{multi-linear rank}) to more closely match that of the input tensor.

\section{Going Further with TensorLy}
We have so far shown how to implement some basic tensor manipulation functions as well as a CP decomposition algorithm based on Alternating Least Squares method.
However, in practice, we want well-tested and robust algorithms that work at scale. This already exists in the TensorLy library, which implements the methods presented in this section, and several more including Tucker decomposition, Robust Tensor PCA, low-rank tensor regression, etc.

The easiest way is to install TensorLy with pip (by simply typing \emph{pip install tensorly} in the console). You can also install it directly from source at \url{https://github.com/tensorly/tensorly}.

When you have it installed, the usage is similar to what we have introduced above:

\begin{python}
import tensorly as tl
import numpy as np

# Create a random tensor:
T = tl.tensor(np.random.random((10, 10, 10)))

# unfold the tensor:
unfolding = tl.unfold(T, mode=0)

# fold it back into a tensor
tl.fold(unfolding, mode=0, shape=tl.shape(T))
\end{python}

Decompositions are already implemented and can be readily applied to an input tensor:
\begin{python}
from tensorly.decomposition import parafac, tucker

# CP decomposition 
weights, factors = parafac(T, rank=3, normalize_factors=True)

# Tucker decomposition returns a core tensor and factor matrices
core, factors = tucker(T, ranks=[3, 2, 4])
\end{python}

You can also easily perform tensor regression using TensorLy, with a similar API that scikit-learn~\citep{scikit-learn} offers.
Refer to the website for a detailed tutorial\footnote{\url{https://tensorly.github.io/dev/}} and API guide.

\section{Scaling up with PyTorch}
All the examples we have presented so far used small tensors that fit nicely in the memory of most commodity laptops and could be run quickly on their CPUs. However, as the size of the data and the complexity of the algorithms grow, we need highly-optimized functions that run on both GPU and CPU and on several machines in parallel. Running in multi-machines setup introduces the challenge of distributed inference and training. These can be incredibly complex to implement correctly. Fortunately, libraries exist that take care of it for you and let you focus on the logic of your model. One notable such framework is PyTorch~\citep{pytorch}.

By default, TensorLy uses NumPy as its backend. However, you can easily switch to PyTorch, a deep learning framework optimized for running large scale methods. Once you have installed PyTorch, you can easily use it as a backend for TensorLy and have all the operations run transparently on multiple machines and GPU. While CPUs performs operations on tensors in a mostly sequential way, GPUs accelerate operations by running them efficiently in parallel: modern CPUs typically contain up to 16 cores, while a GPU has thousands of them.

\begin{python}
import tensorly as tl
import torch

# Use PyTorch as the backend
tl.set_backend('pytorch')

# Create a random tensor:
T = tl.tensor(np.random.random((10, 10, 10)))

type(T) # torch.Tensor!

# You can also specify where the tensor lives:
T = tl.tensor(np.random.random((10, 10, 10)), device='cuda:0')
\end{python}

Now, not only do all the algorithms in TensorLy run on GPU and CPU, you can also interface it easily with PyTorch and Deep Learning algorithms:

\begin{python}
# Let's create a random number generator (rng)
random_state = 1234
rng = tl.random.check_random_state(random_state)

# You can put your tensor on cpu or gpu
device = 'cpu' # Or 'cuda:0'

# Create a random tensor
shape = [5, 5, 5]
tensor = tl.tensor(rng.random_sample(shape), device=device)
\end{python}

We have created a random tensor, which we will try to decompose in the Tucker form. This time, however, we will optimize the factors using gradient descent.

And this is where the magic happens: we can attach gradients to the tensors,
using \texttt{requires\_grad} parameter.
\begin{python}
#Initialize a random Tucker decomposition of that tensor

# We choose a rank for the decomposition
rank = [5, 5, 5]

# We initialize a random Tucker core
core = tl.tensor(rng.random_sample(rank), requires_grad=True,
                 device=device)

# We create a list of random factors
factors = []
for i in range(tl.ndim(tensor)):
    factor = tl.tensor(rng.random_sample(
                        (tensor.shape[i], rank[i])),
                       requires_grad=True, device=device)
    factors.append(factor)

#Let's use the simplest possible learning method: SGD
def SGD(params, lr):
    for param in params:
        # Gradient update
        param.data -= lr * param.grad.data
        # Reset the gradients
        param.grad.data.zero_()
\end{python}

Now we can iterate through the training loop using gradient backpropagation:

\begin{python}
n_iter = 7000
lr = 0.01
penalty = 0.1

for i in range(1, n_iter + 1):
    # Reconstruct the tensor from the decomposed form
    rec = tl.tucker_to_tensor((core, factors))

    # l2 loss 
    loss = tl.norm(rec - tensor, 2)

    # l2 penalty on the factors of the decomposition
    for f in factors:
        loss = loss + penalty * tl.norm(f, 2)

    loss.backward()
    SGD([core] + factors, lr)

    if i 
        rec_error = tl.norm(rec - tensor, 2)/tl.norm(tensor, 2)
        print("Epoch {},. Rec. error: {}".format(i, rec_error))

    if i 
        # Learning rate decay every 3000 iterations
        lr /= 10
\end{python}

You will see the loss gradually go down as the approximation improves. You can verify that the relative reconstruction error is indeed small (we compute the error within a \texttt{no\_grad} context as we do not want to compute gradients here):
\begin{python}
with torch.no_grad():
    # reconstruct the full tensor from these
    rec = tl.tucker_to_tensor((core, factors))
    relative_error = tl.norm(rec - tensor)/tl.norm(tensor)
    
print(relative_error)
\end{python}

\paragraph{}
To conclude, we have demonstrated in this section how to go from theory to a working implementation of tensor methods. These are powerful tools that can be efficiently leveraged using TensorLy. 
Using PyTorch and TensorLy together, you can easily combine tensor methods and Deep Learning, and run your model at scale across several machines and GPUs on millions of data samples. Next, we will discuss further practical considerations of tensor decomposition such as running time, memory usage, and sample complexity.
\clearpage{}

\clearpage{}\chapter{Efficiency of Tensor Decomposition}

In this section, we discuss the running time, memory usage and sample complexity for algorithms based on tensor decomposition.

Tensors are objects of very high dimensions; even a 3rd order $d\times d\times d$ tensor with $d = 10,000$ is already huge and hard to fit into memory of a single machine. A common misconception about tensor decomposition algorithms is that they need to use at least $\Theta(d^3)$ memory, running time and number of samples, because the intuition is one needs at least one unit of resource for each entry in the tensor. This is in fact far from truth and the requirements on these three resources can be much smaller than $O(d^3)$. Furthermore, many tensor algorithms can be naturally parallelized and some of them can be run in an online fashion which greatly reduces the amount of memory required.

\section{Running Time and Memory Usage}

Storing a tensor explicitly as a high dimensional array  and directly performing the computations on the explicit tensor can be very expensive. However, when applied to learning latent variable models and more generally when the tensor has an intrinsic lower dimensional structure, the tensor decomposition algorithms can often be made efficient.

\paragraph{Number of Components:}Latent variable models represent observed variables using hidden variables, e.g., Gaussian mixture model with hidden Gaussian components, topic models with hidden topics and many other models that we described in Section~\ref{sec:applications}. The good news is in most of the cases, the number of hidden components $k$ is often much smaller than the dimension $d$ of observed variables. For example, in topic modeling, the dimension $d$ is equal to the number of words in vocabulary, which is at least in the order of thousands, while the number of topics can be $k = 100$ in many applications. In these cases, after applying the Whitening Procedure  proposed in Procedure~\ref{algo:whitening} we only need to work with a $k\times k\times k$ tensor which is easy to store in memory and allow for efficient computations.

\subsection{Online Tensor Decomposition}

Even when the number of components is large, it is still possible to run many tensor decomposition algorithms without explicitly constructing the tensor. This is because in most of the algorithms we only need to consider the effect of the tensor applied to vectors/matrices and not the whole tensor itself.

\paragraph{Tensor Power Method:}It is very straightforward to convert each iteration of tensor power method to an online algorithm.
In many cases, given samples $x^{(1)},\dotsc,x^{(n)}$, the empirical tensor that we estimate can be represented as $\frac{1}{n}\sum_{i=1}^n S(x^{(i)})$ where $S(\cdot)$ is a function that maps a sample to a tensor.
As an example, consider the multi-view model as explained in Section~\ref{sec:multi}. Each sample $x$ consists of three views $(x_1,x_2,x_3)$. Let $S(x) = x_1\otimes x_2\otimes x_3$, then we desire to estimate the mean tensor $\E[S(x)]$. Given $n$ samples $\{(x^{(i)}_1,x^{(i)}_2, x^{(i)}_3), i\in [n]\}$, then the estimated empirical tensor is
\begin{equation}
\hat{T} = \frac{1}{n}\sum_{i=1}^n S(x^{(i)}) =\frac{1}{n} \sum_{i=1}^n x^{(i)}_1\otimes x^{(i)}_2\otimes x^{(i)}_3. \label{eq:empiricaltensor}
\end{equation}
In tensor power method, the main iteration in~\eqref{eq:map2} involves applying the tensor $\hat{T}$ to vectors $u,v$, which can be easily done as
\begin{equation}
\hat{T}(u, v, I) = \frac{1}{n} \sum_{i=1}^n \inner{x^{(i)}_1,u}\inner{x^{(i)}_2,v}x^{(i)}_3. \label{eq:onlinepower}
\end{equation}
Clearly, using this formula we only need to compute two inner-products for each sample, and the algorithm never needs to store more than a constant number of vectors. 

\begin{claim}[Online Tensor Power Iteration] In many settings, one iteration of tensor power method can be done in time $O(nd)$, where $n$ is the number of samples and $d$ is the dimension. If number of samples is large enough, the algorithm is guaranteed to find an accurately estimated component in $O(nd\log d)$ time with high probability.
\end{claim}

\paragraph{Alternating Least Squares:}ALS method relies on repeatedly solving least square problems; see Algorithm~\ref{algo:als} for the details. To simplify the discussion, we focus on one step of the algorithm, where we are given matrices $A$, $B$, eigenvalues $\lambda$ and want to find $C$ such that $\sum_{j=1}^k \lambda_j a_j\otimes b_j\otimes c_j$ is as close to the empirical tensor $\hat{T}$ as possible; this is what Step 5 in Algorithm~\ref{algo:als} does. All other steps are symmetric and can be computed similarly.

First, we observe that the problem can be decoupled into $d$ sub-problems \--- one for finding each row of $C$.
Consider the variant of Equation~\eqref{eqn:ALSoptMod} for updating matrix $C$ (when $A$, $B$ and $\lambda$ are fixed), and pick the $i$-th row of $\mat(T,3) \in \R^{d \times d^2}$ and matricize it to a $d \times d$ matrix. This leads to the following set of sub-problems to solve for different rows of matrix $C$ denoted by $C^{(i)}$,
$$
\min_{C^{(i)}} \ \Bigl\|\hat{T}(I, I, e_i) - \sum_{j=1}^k \lambda_j C_{i,j} a_jb_j^\top \Bigr\|_F, \quad i \in [d].
$$
These $d$ sub-problems can be solved in parallel which makes it faster to run ALS.

Furthermore, we can use efficient gradient-based methods in the context of online learning even without exploiting parallelization as above. Recently there has been a lot of research on using online gradient-based algorithms to solve least square problems~\citep{shalev2013stochastic,johnson2013accelerating}, and they can all be applied here. A common assumption in these works is that the objective function can be decomposed into the sum of $n$ terms, where the gradient for each term can be computed efficiently. More precisely, the optimization should be of the form
\begin{equation}
\min_{C} \sum_{i=1}^n f \left(C, x^{(i)}\right). \label{eq:problem}
\end{equation}
The guarantee for these online algorithms can be stated in the following informal statement.
\begin{claim}\label{clm:fastleastsquare}
Suppose the objective function in \eqref{eq:problem} is well-conditioned and $n$ is large enough, and the time for computing the gradient for a single $f$ is $T$. Then, there exist algorithms that can find the optimal solution with accuracy $\epsilon$ in time $O(Tn\log n/\epsilon)$. 
\end{claim}

In other words, when the problem is well-conditioned, the algorithms only need a few passes on the data set to find an accurate solution. Having these results for gradient-based methods, we now convert the objective function of ALS to a form similar to the one in~\eqref{eq:problem}. Again suppose we are in the setting that the empirical tensor can be computed as average of $S(x^{(i)})$'s; see Equation~\eqref{eq:empiricaltensor}. Recall the original objective function for ALS is
$$
\min_{C} \Bigl\| \ \sum_{j=1}^k \lambda_j a_j\otimes b_j\otimes c_j - \E_{i\in[n]} S(x^{(i)}) \Bigr\|_F^2,
$$
where $\E_{i\in[n]}(\cdot) := \frac{1}{n} \sum_{i \in [n]} \cdot^{(i)}$. For any random variable $X$, we know 
$$(a - \E[X])^2 = \E[(a-X)^2] - \E[(X-\E[X])^2].$$
Therefore, we can rewrite the objective function as
$$
\min_{C} \ \frac{1}{n} \sum_{i=1}^n \Bigl\| \sum_{j=1}^k \lambda_j a_j\otimes b_j\otimes c_j - S(x^{(i)}) \Bigr\|_F^2 - \frac{1}{n} \sum_{i=1}^n \|S(x^{(i)}) - \hat{T}\|_F^2,
$$
where $\hat{T} := \frac{1}{n} \sum_{i \in [n]} S(x^{(i)})$. The second term does not depend on $C$, so it can be ignored in the optimization problem. Let 
$$f(C, x^{(i)}) := \Bigl\| \sum_{j=1}^k \lambda_j a_j\otimes b_j\otimes c_j - S(x^{(i)}) \Bigr\|_F^2,$$
and thus, we have rewritten the objective function as $\frac{1}{n} \sum_{i=1}^n f(C,x^{(i)}),$
which is exactly the form required in~\eqref{eq:problem}. The gradient of $f$ functions w.r.t.\ to the columns of matrix $C$ denoted by $C_t$ can also be computed as
$$
\frac{\partial}{\partial C_t} f(C,x^{(i)}) = 2\lambda_t\sum_{j=1}^k \lambda_j  \inner{a_j,a_t}\inner{b_j,b_t} C_{j} - 2\lambda_t S(x^{(i)})(a_t,b_t,I).
$$
Computing this stochastic gradient for all the entries of matrix $C$, i.e., all $C_{i,j}$'s only take $\Theta(k^2d)$ time. Then, combined with Claim~\ref{clm:fastleastsquare} allows the least squares problem to be solved efficiently. However, from an arbitrary initialization, we do not have any theoretical bounds on the condition number of these least-squares problems, or the number of iterations it takes ALS to converge. Theoretical analysis of ALS algorithm is still an open problem.

\section{Sample Complexity}

One major drawback of tensor decomposition algorithms is that they often require a fairly large number of samples. A large number of samples may be hard to get in practice, and can also slow down the algorithms \--- as we just saw, many of the tensor decomposition algorithms can be implemented so that they only need to go through the data set small number of times.

A misleading intuition argues that in order to estimate every entry of an $d\times d\times d$ tensor to an accuracy of $\epsilon$, one would need $d^3/\epsilon^2$ samples, which is often too large to be practical. However, this argument is based on the incorrect assumptions that 1) each sample is highly noisy and only provide a small amount of information; 2) the tensor decomposition algorithms require every entry of the tensor to be estimated accurately. The real number of samples required is distinct for different applications, and is far from well-understood.

\subsection{Tensor Concentration Bounds} In tensor decompositions, often we do not need to estimate every entry of the tensor. Instead, we would like to approximate the tensor in a certain norm, e.g., spectral norm, Frobenius norm and other norms based on Sum-of-Squares relaxations are often used. For a specific norm, tensor concentration bounds give estimates on how many samples we need in order to estimate the tensor within some error $\epsilon$. 

When the norm is the Frobenius norm, or the spectral norm of some unfolded version of the tensor (matricized version), the problem can be reduced to vector concentration bounds or matrix concentration bounds. There has been a lot of research on matrix concentrations, many popular bounds can be found in \cite{tropp2012user}.

\paragraph{Tensors with Independent Entries:}For the spectral norm of the tensor, one of the first concentration bounds is by \citet{latala2005some}, and later generalized in \citet{nguyen2010tensor}. They consider the case when there is a random tensor $T$ whose entries are independent random variables with zero mean. For simplicity, we state the following corollary to give a flavor on what they provide.

\begin{corollary}[Corollary 3 of \cite{nguyen2010tensor}] \label{cor:iidtensor} Suppose order-$p$ tensor $T \in \R^{d_1\times d_2\times \cdots \times d_p}$ has i.i.d.\ standard Gaussian entries. Then for every $p$, there exists a constant $C_p > 0$ such that with high probability,
$$
\|T\|^2 \le C_p \max\{d_1,d_2,\dotsc,d_p\}.
$$
\end{corollary}

This shows the spectral norm of a Gaussian tensor only depends on its largest dimension. More specifically, for a $d\times d\times d\times d$ tensor, its spectral norm is still with high probability $O(\sqrt{d})$, which is much smaller than its Frobenius norm $\Theta(d^2)$ or the spectral norm of an unfolded matricization $\Theta(d)$. The technique used in these papers is called the ``entropy-concentration''. The key idea is to argue about linear forms $T(v_1,\dotsc,v_p)$ separately for vectors $v_i$'s that are sparse (low entropy) and dense.

\paragraph{Tensors from Latent Variable Models:}
When the tensor is constructed from a latent variable model (see Section~\ref{sec:applications} for many examples), 
the coordinates of the tensor are often not independent. A case-by-case analysis is required. 
\citet{OvercompleteLVMs2014} analyzed the number of samples required for multi-view model and independent component analysis. The ideas used are again similar to the entropy concentration approach, except a vector is considered ``sparse'' if it has large correlation only with a few components. We provide the guarantee in a simple multi-view model; refer to~\citet{OvercompleteLVMs2014} for more detailed results.

\begin{corollary}
[Spectral Norm Bound for Multi-view Model by~\citet{OvercompleteLVMs2014}]
Consider a simple multi-view model where the latent variable has $k$ possibilities. Each sample is generated by first picking a hidden variable $h \in[k]$, and then observing $x_1 = a_h + \zeta_1$, $x_2 = b_h+\zeta_2$, $x_3 = c_h+\zeta_3$. Here for $h\in[k]$, $a_h$'s, $b_h$'s, $c_h$'s are $d$-dimensional conditional means and assumed to be random unit vectors, and $\zeta_1$, $\zeta_2$, $\zeta_3$ are independent random Gaussian noise vectors whose variance is 1 in each coordinate. Given $n$ samples $\{(x_1^{(i)}, x_2^{(i)}, x_3^{(i)}): i \in [n]\}$, let 
\begin{align*}
\hat{T} &:= \frac{1}{n}\sum_{i=1}^n x^{(i)}_1\otimes x^{(i)}_2\otimes x^{(i)}_3, \\
T &:= \frac{1}{n} \sum_{i=1}^n h_i \ a_{h_i}\otimes b_{h_i}\otimes c_{h_i},
\end{align*}
where $h_i$ denotes the true hidden value for the $i$-th sample. Then with high probability,
$$
\|\hat{T} - T\| \le O(\sqrt{d/n}\cdot \mbox{poly}\log n).
$$
\end{corollary}

Note that in the above model, the noise is extremely high where the total norm of the noise is $O(\sqrt{d})$ compared to the norm of the signal $\|a_h\|=\|b_h\|=\|c_h\|= 1$. Even in this high-noise regime, it only takes $d \cdot \mbox{poly}\log d$ samples to estimate the tensor with constant accuracy in spectral norm. The result is tight up to $\mbox{poly}\log$ factors.

\subsection{Case Study: Tensor PCA and Tensor Completion}

Given a tensor $\hat{T} = T+E$ where $E$ is a perturbation tensor, concentration bounds give us tools to bound the norm of error $E$. However, different tensor decomposition algorithms may have different requirement on $E$. Finding the ``most robust'' tensor decomposition algorithm is still an open problem. In this section, we will describe recent progress in some specific problems.

\paragraph{Tensor PCA:}The model of tensor PCA is very simple. There is an unknown signal $v\in \R^d$ with $\|v\| = 1$. Now suppose we are given tensor
$$
\hat{T} = \tau \cdot v\otimes v\otimes v + E,
$$
where $E\in \R^{d\times d\times d}$ is a noise tensor whose entries are independent standard Gaussians, and $\tau \in \R$ is a scalar. The goal is to find a vector that is within a small constant distance to $v$ when $\hat{T}$ is given. The parameter $\tau$ determines the signal-to-noise ratio, and the problem is easier when $\tau$ is larger. This problem was originally proposed by \cite{richard2014statistical} as a simple statistical model for tensor PCA.

If the algorithm can take exponential time, then the best solution is to find the unit vector $u$ that maximizes $\hat{T}(u,u,u)$. By Corollary~\ref{cor:iidtensor} we know the spectral norm of $E$ is bounded by $O(\sqrt{d})$, and therefore, as long as $\tau \ge C\sqrt{d}$ for some universal constant $C$, the optimal direction $u$ has to be close to $v$.

However, when the algorithm is required to run in polynomial time, the problem becomes harder. The best known result is from \citet{hopkins2015tensor} as follows.

\begin{theorem}[\cite{hopkins2015tensor}] If $\tau = Cd^{3/4}$ for some universal constant $C$, then there is an efficient algorithm that finds a vector $u$ such that with high probability $\|u-v\|\le 0.1$. Moreover, no Sum-of-Squares algorithm of degree at most 4 can do better.
\end{theorem}

The term $d^{3/4}$ in the above theorem is between the information theoretic limit $\sqrt{d}$ and the trivial solution that treats the tensor as a $d\times d^2$ matrix which gives $\Theta(d)$ bound for $\tau$. The problem can also be solved more efficiently using a homotopy optimization approach~\citep{anandkumar2016homotopy}. However, it seems there are some fundamental difficulties in going below $d^{3/4}$.

\paragraph{Tensor Completion:}A very closely related problem is called tensor completion. In this problem, we observe a random subset of entries of a low-rank tensor $T$, and the goal is to recover the original full low-rank tensor. \citet{barak2016noisy} provide a tight bound on the number of samples required to recover $T$.

\begin{theorem}[\cite{barak2016noisy}, informal]
Suppose tensor $T\in \R^{d\times d\times d}$ has rank $k$, given $n = d^{1.5} k \poly\log d$ random observations of the entries of the tensor, there exists an algorithm that recovers $T$ up to a lower order error term.
\end{theorem}

For small $k$, the term $d^{1.5} k \poly\log d$ in the above guarantee is again between information theoretic limit $\Theta(dk\log d)$ and the trivial solution that considers the tensor as a $d\times d^2$ matrix which gives  $\Theta(d^2k\log d)$. \cite{barak2016noisy} showed the $d^{1.5}$ dependency which is likely to be tight because improving this bound will also give a better algorithm for refuting random 3-XOR clauses (which is a conjectured hard problem~\citep{feige2002relations,grigoriev2001linear,schoenebeck2008linear}).
The algorithm is again based on Sum-of-Squares. Recently there were also several improvements in the recovery guarantees, see e.g., \citet{potechin2017exact} and references therein.\clearpage{}

\clearpage{}\chapter{Overcomplete Tensor Decomposition} \label{sec:overcomplete-decomp}

Unlike matrices, the rank of a tensor can be higher than its dimension. We call such tensors {\em overcomplete}. Overcomplete tensors can still have a unique decomposition; recall Theorem~\ref{thm:kruskal} for 3rd order tensors and see~\citet{Sidiropoulos2000:CPuniqueness} for higher order tensors. This is useful in the application of learning latent variable models: it is possible to learn a model with more components than the number of dimensions, e.g., a mixture of 100 Gaussians in 50 dimensions. 

However, finding a CP decomposition for an overcomplete tensor is much harder than the undercomplete case (when the rank is at most the dimension). In this section we will describe a few techniques for decomposing overcomplete tensors.

\section{Higher-order Tensors via Tensorization}

For higher order tensors, the most straightforward approach to handle overcomplete decomposition is to convert them to lower order tensors but in higher dimension. We call this approach as {\emph tensorization} and describe it in this section. For simplicity, we restrict our attention to symmetric tensors, but what we discuss here also applies to asymmetric tensors. Consider a 6th order tensor
$$T=\sum_{j=1}^k \lambda_j a_j^{\otimes 6},$$ with $d$-dimensional rank-1 components $a_j \in \R^d, j\in [k]$, and real weights $\lambda_j \in \R, j\in [k]$, where $k \gg d$. We can reshape this tensor as a 3rd order tensor $\hat{T} \in R^{d^2\times d^2 \times d^2}$ as follows. Let $$b_j := a_j \odot a_j \in \R^{d^2},$$ where $\odot$ denotes the Khatri-Rao product defined in~\eqref{eqn:KhatriRao}; note that with vector inputs, this works the same as Kronecker product. Then we have
$$
\hat{T} = \sum_{j=1}^k \lambda_j b_j^{\otimes 3}.
$$
We call this process of reshaping the tensor to a different order as tensorization. Now for the 3rd order tensor $\hat{T}$, if the rank-1 components $b_j$'s are linearly independent, we can use the tensor decomposition algorithms in Section~\ref{ch:tensor-decomp} to recover its rank-1 components $b_j$'s, and since $b_j = a_j \odot a_j$, the original rank-1 components $a_j$ is computed as the top singular vector of the matricized version of $b_j$. This whole approach is provided in Algorithm~\ref{algo:overcomplete}.

\floatname{algorithm}{Algorithm}
\begin{algorithm}[t]
\caption{Decomposing Overcomplete Tensors via Tensorization}
\label{algo:overcomplete}
\begin{algorithmic}[1]
\renewcommand{\algorithmicrequire}{\textbf{input}}
\renewcommand{\algorithmicensure}{\textbf{output}}
\REQUIRE tensor $T = \sum_{j=1}^k a_j^{\otimes 6}$
\ENSURE rank-1 components $a_j, j \in [k]$
\STATE Reshape the tensor $T$ to $\hat{T} \in \R^{d^2\times d^2\times d^2}$, each mode of $\hat{T}$ is indexed by $[d]\times [d]$, and $\hat{T}_{(i_1,i_2),(i_3,i_4),(i_5,i_6)} = T_{i_1,i_2,i_3,i_4,i_5,i_6}$.
\STATE Use Tensor Power Method (Algorithm~\ref{alg:robustpower}) to decompose $\hat{T} = \sum_{j=1}^k b_j^{\otimes 3}$ and recover $b_j$'s.
\STATE For each $b_j \in \R^{d^2}, j\in[k]$, reshape it as a $d \times d$ matrix, and let $(\lambda_j,v_j)$ be its top singular value-vector pair.
\RETURN $\sqrt{\lambda_j}v_j, j\in[k]$.
\end{algorithmic}
\end{algorithm}

If $k \le {d+1 \choose 2}$ and the vectors $a_j$'s are in general position, then the vectors $b_j$'s are going to be linearly independent. Recent work by~\citet{bhaskara2014smoothed} shows that if $a_j$'s are perturbed by a random Gaussian noise, the smallest singular value of matricized $b_j$'s are lower bounded (where the lower bound depends polynomially on the magnitude of the noise and the dimension). As a result this algorithm is robust to small amount of noise.

In Algorithm~\ref{algo:overcomplete}, we used tensor power iteration as the core tensor decomposition algorithm. It is worth mentioning that we can also use other tensor decomposition algorithms. In particular, if we use simultaneous diagonalization algorithm as proposed in Algorithm~\ref{algo:simdiag} instead of tensor power method, then it suffices to have a 5th order input tensor $T = \sum_{j=1}^k \lambda_j a_j^{\otimes 5}$. Again we can reshape the tensor as
$$
\hat{T} = \sum_{j=1}^k \lambda_j b_j\otimes b_j\otimes a_j.
$$
Even though the third mode still only has $d$ dimensions which is smaller than rank $k$ in the overcomplete regime, simultaneous diagonalization only requires the rank to be less or equal to the dimension of the first two modes, and thus, the algorithm can work.
 
The same idea can be also generalized to even higher order tensors. In general, if vectors $a_j^{\otimes r}$ are linearly independent, then we can apply simultaneous diagonalization algorithm to a $(2r+1)$-th order tensor and compute the unique tensor decomposition.
 
This algorithm can be applied to learning several latent variable models, as long as we have access to higher order tensors. In some applications such as pure topic models, this is fairly straightforward  as we only need to form the moment using the correlations of $p$ words instead of 3. In other applications this may require careful manipulations of the moments. In both cases, working with higher order tensors can potentially increase the sample complexity and running time of the algorithm.
 
\section{FOOBI Algorithm}

In practice, working with high-order tensors is often too expensive in terms of both sample complexity and running time. Therefore, it is useful to design algorithms that can handle overcomplete tensors when the order of the tensor is low, e.g., 3rd or 4th order tensors. For 4th order tensors, \citet{de2007fourth} proposed an algorithm called FOOBI (Fourth-Order-Only Blind Identification) that can work up to rank $k = Cd^2$ for some fixed constant $C > 0$. For simplicity, we again describe the algorithm for symmetric tensors $$T = \sum_{j=1}^k \lambda_j a_j^{\otimes 4},$$ and we will keep the notation $b_j := a_j \odot a_j$ as we had in the previous section. We will also assume $\lambda_j>0$ and the components are real-valued. All these requirements can be removed for this algorithm and interested readers are encouraged to check the original paper. We provide the FOOBI method in Algorithm~\ref{algo:foobi}. Intuitively, the algorithm has three main parts
\begin{enumerate}
\item Finding the span of vectors $\{b_j: j \in [k]\}$.
\item Finding the $b_j$'s.
\item Computing the $a_j$'s.
\end{enumerate}
Step 1 is done using a SVD operation and step 3 is achieved the same as what we discussed in the previous section for higher order tensors. The magic happens in step 2 where the algorithm uses a quadratic operator to detect rank-1 matrices. In the rest of this section, we describe these steps in more details.

\begin{algorithm}[h!]
\caption{FOOBI for Decomposing Overcomplete Tensors~\citep{de2007fourth}}
\label{algo:foobi}
\begin{algorithmic}[1]
\renewcommand{\algorithmicrequire}{\textbf{input}}
\renewcommand{\algorithmicensure}{\textbf{output}}
\REQUIRE tensor $T = \sum_{j \in [k]} \lambda_j a_j^{\otimes 4}$
\ENSURE rank-1 components $\{(\lambda_j, a_j)\}$'s
\STATE Reshape the tensor to a matrix $$M = \sum_{j=1}^k \lambda_j (a_j\odot a_j)(a_j\odot a_j)^\top \in \R^{d^2 \times d^2}.$$
\STATE Compute the SVD of $M$ as $M = UDU^\top$.
\STATE Let $L:\R^{d\times d}\to \R^{d^4}$ be a quadratic operator such that
$$
L(A)(i_1,i_2,j_1,j_2) = \det \left(\begin{array}{cc}A_{i_1,j_1} & A_{i_1,j_2} \\ A_{i_2,j_1} & A_{i_2,j_2} \end{array}\right).
$$
Let $\tilde{L}:\R^{d^2\times d^2}\to \R^{d^4}$ be the unique linear operator that satisfy $\tilde{L}(A\otimes A) = L(A)$.
\STATE Construct matrix $Z = \tilde{L} [(UD^{1/2})\otimes (UD^{1/2})]$. 
\STATE Let $y_1,y_2,\dotsc,y_k \in \R^{k^2}$ be the $k$ least right singular vectors of $Z$. 
\STATE Pick random vectors $u,v\in \mbox{span}(y_1,\dotsc,y_k)$, reshape them as  $k\times k$ matrices $U,V$.
\STATE Use Simultaneous Diagonalization (see Algorithm~\ref{algo:simdiag} in Section~\ref{sec:simdiag}) to express $U = QD_UQ^\top$ and $V = QD_VQ^\top$, where $D_U,D_V$ are $k\times k$ diagonal matrices and $Q\in \R^{k\times k}$ is shared between $U,V$.
\STATE Let $x_1,\dotsc,x_k\in \R^k$ be the columns of $Q$, $b_j = UD^{1/2}x_j$ for all $j\in[k]$.
\STATE For $j\in[k]$, reshape $b_j \in \R^{d^2}$ to $d\times d$ matrix and let $(\delta_j, v_j)$ be its top singular value-vector pair.
\RETURN $\{(\delta_j^2, v_j): j\in[k] \}$.
\end{algorithmic}
\end{algorithm}

\subsection{Finding Span of $\{b_j\}$'s}

In the first step, we try to find the span of the vectors $b_j$'s. This is very simple as we can matricize the tensor as
\begin{equation} \label{eqn:foobi-matrix}
M := T_{\{1,2\},\{3,4\}} = \sum_{j=1}^k \lambda_j b_j b_j^\top \in \R^{d^2 \times d^2},
\end{equation}
where $b_j := a_j \odot a_j \in \R^{d^2}$.
Therefore, we just need to compute the column span (or row span) of $M$, and it would corresponds to the span of vectors $\{b_j\}$'s. In order to make the algorithm more robust to noise, we use singular value decomposition to find the top singular values, and drop all the singular values that are very close to $0$.

\subsection{Finding $\{b_j\}$'s}

In the second step of the algorithm, we can view the vectors $b_j$'s as reshaped $d\times d$ matrices. For vector $u \in R^{d^2}$, the matricized version $\mbox{mat}(u) \in \R^{d \times d}$ is defined as
$$
\mbox{mat}(u)_{i,j} = u(d(j-1)+i), \quad i,j \in [d],
$$
 which is formed by stacking the entries of $u$ in the columns of the matrix.
Given the definition of $b_j$'s, we have $\mbox{mat}(b_j) = a_j a_j^\top$ which are rank-1 matrices. In addition, from the previous step, we know the linear subspace spanned by these matrices. Using the key observation that $\mbox{mat}(b_j)$ are all rank-1 matrices, we hope to recover $b_j$'s as follows.

Suppose the SVD of $M$ in~\eqref{eqn:foobi-matrix} is denoted by $UDU^\top$, where $U \in \R^{d^2 \times k}$ be an orthonormal matrix that represents the span of the vectors $\{b_j\}$'s.
Since we know 
$$M = (UD^{1/2})(UD^{1/2})^\top = \sum_{j=1}^k (\sqrt{\lambda_j}b_j)(\sqrt{\lambda_j}b_j)^\top,$$
there exists an orthogonal matrix $R$ such that the columns of $UD^{1/2}R$ are equal to $\{\sqrt{\lambda_j}b_j\}$'s. In order to find the vectors $b_j$'s, we need to find the columns of this orthogonal matrix $R$ denoted by $x_j$ such that $\mbox{mat}(UD^{1/2}x_j)$'s are rank-1 matrices; recall the above discussion that the matricized versions of $b_j$'s are rank-1 matrices.

Finding these $x_j$ directions is not an easy task. \citet{de2007fourth} show that it is possible to do this using a very interesting rank-1 detector.

\subsection{Rank-1 Detector}

For a matrix $A \in \R^{d\times d}$, we know $A$ is rank at most 1 if and only if determinants of all $2\times 2$ submatrices of $A$ are equal to $0$. In particular, for a symmetric matrix $A$, we can define a mapping $L(A)$ that maps $A$ to a $\mathcal{D} \approx d^4/8$ dimensional space, where each entry in $L(A)$ corresponds to the value of the determinant of a unique $2\times 2$ submatrix of $A$. The exact number of dimensions is $\mathcal{D} = {{d\choose 2}+1 \choose 2}$, because that is the number of 4-tuples $(i_1,i_2,j_1,j_2)$ where $i_1 < i_2$, $j_1<j_2$ and $(i_1,i_2)\le (j_1,j_2)$.

\begin{definition}[rank-1 detector] Function $L$ maps $d\times d$ symmetric matrices to $\mathcal{D} = {{d\choose 2}+1 \choose 2}$ dimensional space indexed by  $(i_1,i_2,j_1,j_2)$ where $i_1 < i_2$, $j_1<j_2$ and $(i_1,i_2)\le (j_1,j_2)$, where
$$
L(A)(i_1,i_2,j_1,j_2) = \det \left(\begin{array}{cc}A_{i_1,j_1} & A_{i_1,j_2} \\ A_{i_2,j_1} & A_{i_2,j_2} \end{array}\right).
$$
\end{definition}

It is easy to prove that this rank-1 detector indeed works. 

\begin{claim}
Symmetric matrix $A\in \R^{d\times d}$ is of rank at most 1, if and only if $L(A) = 0$. 
\end{claim}

The mapping $L(A)$ is quadratic in the entries of $A$. Therefore, if we apply $L$ to the matrix $Vx$, then $L(Vx)$ is also quadratic in the variable $x$. Na\"ively, $L(Vx) = 0$ would give a set of quadratic equations, and solving a system of quadratic equations is again hard in general. Luckily, we have a very large number of equations \--- $\mathcal{D} \approx d^4/8$. This allows us to use a {\em linearization} approach: instead of treating $L(Vx) = 0$ as a system of quadratic equations over $x$, we will {\em lift} the variables to $X = x \odot x$ and view $X$ as a $d+1\choose 2$ dimensional vector. Now in our problem, $L(UD^{1/2}x)$ is equal to a linear operator $\tilde{L}$ applied to $[(UD^{1/2})\otimes (UD^{1/2})]X$, i.e.,
$$L(UD^{1/2}x) = \tilde{L}([(UD^{1/2})\otimes (UD^{1/2})]X).$$ 

\begin{definition}[Linearized detector] Linearized rank-1 detector $\tilde{L}$ maps a ${d^2\times d^2}$ matrix such that 
$$
\tilde{L}(A\otimes A) = L(A).
$$
More precisely, for two matrices $A,B\in \R^{d\times d}$ we have
$$
\tilde{L}(A\otimes B)(i_1,i_2,j_1,j_2) = A_{i_1,j_1}B_{i_2,j_2} - A_{i_1,j_2} B_{i_2,j_1},
$$
which is the determinant of the $2 \times 2$ submatrix of $A$ when $A=B$.
\end{definition}

Since $\tl{L}$ is a linear operator,  we can represent it as a matrix; let $Z = \tilde{L}[(UD^{1/2})\otimes (UD^{1/2})]$, and try to solve the system of linear equations $ZX = 0$. Of course, in general doing this ignores the structure in $X$ (that it has the form of $X = x\odot x$), and will not work for general quadratic equations. In this specific case, \citet{de2007fourth} were able to show that the only solutions of this equation are linear combinations of the desired solution.

\begin{theorem}[\citep{de2007fourth}]
\label{thm:lifttensor}
Let $a_j$'s be in general positions and $k \le Cd^2$ for some universal constant $C>0$. Let $x_1,x_2,\dotsc,x_k$ be $k$ vectors such that $Ux_j = b_j$, and $X_j = x_j\odot x_j$ for $j \in [k]$. Then the solution of $\tilde{L}[(UD^{1/2})\otimes (UD^{1/2})]X = 0$ (the null space of $Z$) is exactly equal to the span of $X_j$'s.
\end{theorem}

\subsection{Finding the Rank-1 Components}

In this final step, we are given a subspace which is equal to the span of $\{x_j\odot x_j\}$'s, and we are trying to find $\{x_j\}$'s. At a first glance, this might look exactly the same as the problem we were facing in the previous step: we were given the span of $\{a_j\odot a_j\}$ and trying to find $\{a_j\}$'s. Luckily these two problems are actually very different \--- in both cases we are looking for $k$ vectors, but previously we were given a span of $d\times d$ matrices (and $k\gg d$) and now we have a span of $k\times k$ matrices. The vectors $a_j$'s cannot be linearly independent, while the vectors $x_j$'s are usually linearly independent. Now, to find $\{x_j\}$'s, the key observation is that every matrix in the span of $\{x_j \odot x_j\}$'s can be {\em simultaneously diagonalized}, and the vectors $x_j$'s are the only way to do that. Therefore the last step of the algorithm is very similar to the simultaneous diagonalization algorithm for undercomplete tensor decomposition. 
In Algorithm~\ref{algo:foobi}, for simplicity we just applied simultaneous diagonalization on two random matrices $U,V$ in the subspace found in the last step (as we discussed in Section~\ref{sec:simdiag} this can be easily done by eigen-decomposition of $UV^{-1}$).
To ensure numerical stability, the original FOOBI algorithm requires a simultaneous diagonalization of all the $y_j$'s. As a result, we find $k$ vectors $x_1,\dotsc,x_k$ such that $UD^{1/2}x_j = \sqrt{\lambda_j}b_j$. The rest of the algorithm is simply recovering $a_j$'s from $b_j$'s.

\section{Third Order Tensors}

Algorithms like FOOBI can work with tensors with order at least 4. That leaves only third order tensors. We still don't have any algorithms for overcomplete third order tensors when the components are only guaranteed to be in general position.

Third order tensor is very special and might be fundamentally more difficult to decompose. As an example, it is very easy to construct an explicit 4-th order tensor that has rank at least $d^2$, because the rank of the tensor is at least as large as its matricizations. However, for a $d\times d\times d$ third order tensor, all matricizations can have rank at most $d$; note that even the most balanced matricizations have dimensions $d^2 \times d$ or $d \times d^2$. It is still an open problem to construct an explicit third order tensor whose rank is super-linear in dimension, and in fact doing so will lead to circuit lower bounds that were open for decades~\citep{raz2013tensor}.

Because of these difficulties, researchers have focused on the simpler setting where the components $\{a_j\}$'s are chosen from a random distribution. Even in this simple case, the only provable algorithm relies on complicated algorithms called Sum-of-Squares Hierarchies. We refer the readers to the survey by \citet{barak2014sum}. In the rest of this section, we give some intuitions on how to handle overcomplete third order tensors without going into the details.

\subsection{Lifting the Tensor}

A key technique in handling third order tensor is to {\em lift} the tensor into a higher order tensor. This can either be done explicitly, or implicitly using Sum-of-Squares framework. Here we show a simple transformation that lifts a third order tensor to a 4th order tensor. Again for simplicity, we only work with symmetric tensors in this section.

\begin{definition}[Lifted tensor]
Given a tensor $T\in \R^{d\times d\times d}$, we can construct a lifted tensor $M(T) \in \R^{d \times d \times d \times d}$ as $$M(T)_{i_1,i_2,i_3,i_4} := \sum_{i=1}^d T_{i,i_1,i_2} T_{i,i_3,i_4}.$$
\end{definition}

Note that for a rank-1 tensor $T = a\otimes a\otimes a$, the lifted tensor $M(T) = \|a\|^2 a^{\otimes 4}$ is also rank-1. It is easier to interpret the lifted tensor using the multilinear form, in particular, we have
$$
M(T)(x,x,x,x) = \|T(:,x,x)\|^2.
$$
As a result, if $T$ has decomposition $T = \sum_{j=1}^k a_j\otimes a_j\otimes a_j$, we can represent $M(T)$ as a low rank tensor plus noise. This is formulated as follows.

\begin{theorem}[\cite{ge2015decomposing}]
Suppose $$T =\sum_{j \in [k]} a_j\otimes a_j\otimes a_j,$$ then the lifted  tensor can be represented as $$M(T) = \sum_{j \in [k]} \|a_j\|^2 a_j^{\otimes 4} + M',$$ where $$M' = \sum_{i\ne j\in[k]} \inner{a_i,a_j} a_i\otimes a_i\otimes a_j\otimes a_j.$$
Furthermore, suppose $a_j$'s are chosen according to Gaussian distribution with expected square norm 1. Then the norm $\|M'_{\{1,3\},\{2,4\}}\|$ is bounded by $o(1)$ when $k \le d^{3/2}/\poly\log(d)$, where $M'_{\{1,3\},\{2,4\}} \in \R^{d^2 \times d^2}$ denotes the matricization of tensor $M'$ such that the 1st and 3rd modes are stacked along the rows, and the 2nd and 4th modes are stacked along th columns of the matrix.
\end{theorem}

Intuitively, this theorem shows that  after the lifting operation, we get a 4th order rank-$k$ tensor with noise $M'$. The norm of $M'$ is small compared to the true components in the low rank decomposition. Therefore, it is possible to find $a_j$'s as long as we can decompose 4th order tensors under such kind and amount of noise. \citet{ge2015decomposing} gave a quasi-polynomial time algorithm to do this. Later, \citet{ma2016polynomial} showed it is also possible to do this within polynomial time.

\subsection{Robust 4th Order Tensor Decomposition} 

In order to solve the 4th order tensor decomposition problem, we might want to use the FOOBI algorithm described earlier. However, the noise term $M'$ here has spectral norm $o(1)$, and the FOOBI algorithm is not known to be robust to such perturbations.

Using Sum-of-Squares techniques, \citet{ma2016polynomial} gave an algorithm that can decompose a 4th order tensor even under significant noise.

\begin{theorem}[\cite{ma2016polynomial}]
Let $T\in \R^{d \times d \times d \times d}$ be a symmetric 4th order tensor and $a_1,\dotsc,a_k \in \R^d$ be a set of vectors. Define perturbation tensor  $E := T - \sum_{j \in [k]} a_j^{\otimes 4}$, and define $A$ as the matrix with columns $a_j^{\otimes 2}$. If the (matricized) perturbation norm $\|E_{\{1,2\},\{3,4\}}\| \le \delta \cdot \sigma_k(AA^\top)$ for some $\delta>0$, then there is an algorithm that outputs a set of vectors $\hat{a}_j$, and there is a permutation $\pi:[k]\to [k]$ such that for every $j \in [k]$, we have
$$
\min\{\|a_j-\hat{a}_{\pi(j)}\|,\|a_j+\hat{a}_{\pi(j)}\|\} \le O(\delta\|A\|/\sigma_k(A))\|a_j\|.
$$
\end{theorem}

Note that we cannot directly combine this theorem with Theorem~\ref{thm:lifttensor} to get a complete algorithm for decomposing overcomplete 3rd order tensors. There are a few technical issues: 1. The norms of $\|a_j\|$'s are not exactly 1, but they are very close to 1 by concentration; 2. we need to reshape the tensor so that the matricization $T_{\{1,2\},\{3,4\}}$  has small spectral norm; 3. when $a_j$'s are random, $\|A\|/\sigma_k(A)$ is usually $\sqrt{k/d}$ which is bigger than 1. The first two problems are easy to handle, while the third problem requires more work.

\citet{ma2016polynomial} also give a direct analysis for overcomplete 3rd order tensors using Sum-of-Squares, and that analysis do not rely on the explicit lifting.

\section{Open Problems}
Despite the algorithms we provided, decomposing an overcomplete tensor is still a very difficult problem. The algorithms often require access to high-order tensors, which is often expensive in both sample complexity and running time. Sum-of-Squares algorithms can tolerate more noise and therefore, work with fewer samples, but the running time is prohibitive. Finding a provable overcomplete tensor decomposition algorithm that is efficient in practice is still a major problem.

Although the current provable algorithms are quite complicated, in practice algorithms like Alternating Least Squares or Power Method (see Section~\ref{ch:tensor-decomp}) work surprisingly well even when the tensor is overcomplete. For a random 3rd order tensor with dimension $100$ and rank $1000$, Alternating Least Squares with random initialization almost always converges to the right answer within 10 iterations. This is very surprising and we do not yet know how to prove it works. When the components are not randomly generated, people have observed Alternating Least Squares can be sometimes slow~\citep{comon2002tensor}. How to handle and analyze these kind of tensors is also widely open.\clearpage{}

\begin{acknowledgements}
The authors are grateful to anonymous reviewers for valuable comments that have significantly improved the manuscript.
\end{acknowledgements}

\backmatter  

\printbibliography

\end{document}